\documentclass[article]{elsarticle}
\usepackage[utf8]{inputenc}
\usepackage{csquotes}
\usepackage{soul}
\setuldepth{article}

\usepackage{hyperref}

\usepackage{amssymb}
\usepackage{amsmath}
\usepackage{amsthm} 
\usepackage{mathptmx}
\usepackage{centernot}
\usepackage{multirow}
\usepackage{lineno}

\usepackage{booktabs}

\usepackage{float}
\usepackage{subfig}

\makeatletter
\renewcommand{\p@subfigure}{}
\makeatother

\usepackage[inline, shortlabels]{enumitem}
 
\usepackage[shortcuts, theorems, arg]{macros}

\usepackage{dcolumn}

\newcolumntype{d}[1]{D{.}{.}{#1}}

\begin{document}

\begin{frontmatter}              
\title{Contribution Functions for Quantitative Bipolar Argumentation Graphs: A Principle-based Analysis} %

\author{Timotheus Kampik}
\ead{tkampik@cs.umu.se}
\address{Umeå University, Sweden}
\address{SAP Signavio, Germany}

\author{Nico Potyka}
\ead{PotykaN@cardiff.ac.uk}
\address{Cardiff University, UK}

\author{Xiang Yin}
\ead{x.yin20@imperial.ac.uk}
\address{Imperial College London, UK}

\author{Kristijonas \v{C}yras}
\ead{kristijonas.cyras@ericsson.com}
\address{Ericsson, USA}

\author{Francesca Toni}
\ead{f.toni@imperial.ac.uk}
\address{Imperial College London, UK}

\begin{abstract}
We present a principle-based analysis of \emph{contribution functions} for quantitative bipolar argumentation graphs that quantify the contribution of one argument to another.
The introduced principles formalise the intuitions underlying different contribution functions as well as expectations one would have regarding the behaviour of contribution functions in general.
As none of the covered contribution functions satisfies all principles, our analysis can serve as a tool that enables the selection of the most suitable function based on the requirements of a given use case.
\end{abstract}

\end{frontmatter}
\section{Introduction}
\label{sec:intro}
Formal argumentation has emerged as a vibrant research area within the field of artificial intelligence.
In particular, the dialectical nature of formal argumentation is considered a promising facilitator of joint human-machine reasoning and decision-making, as well as a potential bridge between subsymbolic and symbolic AI~\cite{10.1145/3284432.3284470}.
In formal argumentation, arguments and their relations are represented as directed graphs, in which nodes are arguments and edges are argument relationships (typically: attack or support).
From these \emph{argumentation graphs}, inferences about the acceptability statuses or strengths of arguments are drawn.
One formal argumentation approach that is gaining increased research attention is Quantitative Bipolar Argumentation (QBA).
In QBA, (typically numerical) weights---so-called \emph{initial strengths}---are assigned to arguments, and arguments are connected by a support and an attack relation. Hence, arguments directly connected to a node through the node's incoming edges can be referred to as \emph{attackers} and \emph{supporters} (depending on the relation).
Given a Quantitative Bipolar Argumentation Graph (QBAG), an \emph{argumentation semantics} then infers the arguments' \emph{final strengths}; intuitively, an argument's attackers tend to decrease its final strength, whereas supporters tend to increase it.
Nascent applications of formal argumentation in general and QBA in particular are often related to \emph{explainability}~\cite{Cyras.et.al:2021-IJCAI,vassiliades_bassiliades_patkos_2021}, e.g., in the context of explainable recommender systems~\cite{10.5555/3304889.3304929}, review aggregations~\cite{cocarascu2019extracting} or machine learning models
like random forests~\cite{Potyka0T23}
or neural networks~\cite{AyoobiPT23}.
In order to utilise QBA as a facilitator of explainability, it is crucial to develop a rigorous understanding of the influence of one argument on the final strength of another one.
This follows the above-mentioned intuition of argument influence that (direct or indirect) attackers or supporters exercise and reflects the ideas of feature attributions in machine learning (see ~\cite{ribeiro2016should,lundberg2017unified,bloniarz2016supervised,plumb2018model,potyka2022towards} as examples), as well as of contributions of agents towards a joint objective in cooperative game theory.
Initial steps towards this objective have been undertaken by nascent/preliminary research~\cite{Cyras:Kampik:Weng:2022,DBLP:conf/ecai/0007PT23} that introduced so-called \emph{contribution functions} quantifying the influence of one argument in a QBAG on another one. 

Yet, a comprehensive theoretical study of contribution functions is missing.
Our paper aims to close this gap in the literature by introducing a set of principles that a contribution function should satisfy, and by conducting a principle-based analysis for \emph{acyclic} QBAGs.
We claim that this limitation is well-motivated: in many applications of QBAGs there is a clear hierarchical or temporal structure to the arguments~\cite{kotonya2019gradual,cocarascu2019extracting,lidecision,chi2021optimized}, which helps avoid cycles or allows for removing them on the meta-level.
The introduced principles reflect the ideas that underlie the different contribution functions, as well as some intuitive, common-sense assumptions (that are, however, not always satisfied).

Let us introduce a simple example that provides an intuitive understanding of the paper's core contribution.
\begin{example}\label{ex:intro-modular}
Consider the QBAG $\graph$ in Figure~\ref{fig:graph-intro} (the interpretation of the figure is explained in the caption).
The initial strengths of $\graph$'s arguments are determined by DFQuAD semantics~\cite{rago2016discontinuity}, according to the following intuition:
\begin{itemize}
    \item Final strengths are propagated through the graph following a topological ordering of the nodes.
    \item Given an argument $\arga$ and its direct attackers $A$ and supporters $S$, the final strengths of $A$ and $S$ are aggregated as follows:
    $$f(A, S) := \Pi_{\argb \in A} (1 - (\sigma(\argb))- \Pi_{\argc \in S} (1 - (\sigma(\argc)),$$ where $\sigma(\argb)$ and $\sigma(\argc)$ denote the final strengths of $\argb$ and $\argc$, respectively.
    \item Given the aggregated score $f(A, S)$, $\arga$ is then updated as follows, given its initial strength $\tau(\arga)$:
    $$g(\tau(a), f(A, S)) :=  \tau(a) - \tau(a) \times max\{0, -f(A, S)\} + (1 - \tau(a)) \times max\{0, f(A, S)\}.$$
    
\end{itemize}
We are interested in explaining how different arguments (``contributors'') contribute to the final strength of $\arga$.
As we are primarily interested in $\arga$ and its final strength, we refer to $\arga$ as the \emph{topic argument}.
Obviously, the contribution of supporter $\argb$ should be positive (or possibly zero) and the contributions of attackers $\argc$ and $\argd$ should be negative (or zero): in these cases, the direction (with respect to $0$) of the contribution is clear.
In contrast, it is not at all clear how $\arge$ contributes to $\arga$.
Considering this more intricate case, we may specify our expectation regarding the behaviour of the function that determines the contribution (in this case: of $\arge$ to the final strength of $\arga$) as a \emph{contribution function principle}.
For example, one may stipulate that:
\begin{enumerate}
    \item The contribution should be negative if (and only if) removing $\arge$ from $\graph$ increases the final strength of $\arga$ and positive if (and only if) the removal decreases it; we call this desideratum \emph{counterfactuality}.
    \item If the contribution is negative/positive, then marginally increasing the initial strength of $\arge$ should decrease/increase the strength of $\arga$; we call this desideratum \emph{faithfulness}.
\end{enumerate}
The QBAG in Figure~\ref{fig:plot-intro} illustrates that
these two ideas can be in conflict. To explain this,
we plotted the final strength of $\arga$ (y-axis)
against the initial strength of $\arge$ (x-axis)
in the function graph on the right. 
Since $\arge$'s initial strength is $0.5$ in the original QBAG, $\arga$'s
final strength is about $0.37$.
When removing $\arge$, the final strength of $\argb, \argc$ and $\argd$
will just be their initial strength, which is $0$. Since arguments
with strength $0$ cannot affect other arguments (given DFQuAD semantics), %
$\arga$'s final
strength will also be its initial strength, which is $0.5$.
Hence, since $\arga$'s strength increases from $0.37$ to $0.5$
when removing $\arge$, $\arge$ 
should have a negative contribution
according to \emph{counterfactuality}. Now, if $\arge$'s 
contribution is negative, then
\emph{faithfulness} demands that marginally increasing
the strength of $\arge$ in the original example
must decrease the strength of $\arga$. However, as we can see in
the strength plot in Figure~\ref{fig:plot-intro}, 
increasing the initial strength of $\arge$ will actually increase
the final strength of $\arga$.
Hence, it is impossible to satisfy both \emph{counterfactuality}
and \emph{faithfulness} simultaneously in this example.

\begin{figure}[ht]
\centering
\subfloat[$\graph$.]{
\label{fig:graph-intro}
\begin{tikzpicture}[scale=0.8]
    \node[unanode]    (a)    at(2,0)  {\argnode{\arga}{0.5}{0.375}};
    \node[unanode]    (b)    at(0,2)  {\argnode{\argb}{0.0}{0.5}};
    \node[unanode]  (c)    at(2,2)  {\argnode{\argc}{0.0}{0.5}};
    \node[unanode]    (d)    at(4,2)  {\argnode{\argd}{0.0}{0.5}};
    \node[unanode]    (e)    at(2,4)  {\argnode{\arge}{0.5}{0.5}};
    
    \path [->, line width=0.2mm]  (b) edge node[left] {+} (a);
    \path [->, line width=0.2mm]  (c) edge node[left] {-} (a);
    \path [->, line width=0.2mm]  (d) edge node[left] {-} (a);
    \path [->, line width=0.2mm]  (e) edge node[left] {+} (b);
    \path [->, line width=0.2mm]  (e) edge node[left] {+} (c);
    \path [->, line width=0.2mm]  (e) edge node[left] {+} (d);
\end{tikzpicture}
}
\subfloat[Final strength of $\arga$, given initial strength of $\arge$.]{
\label{fig:plot-intro}
\includegraphics[width=0.6\columnwidth]{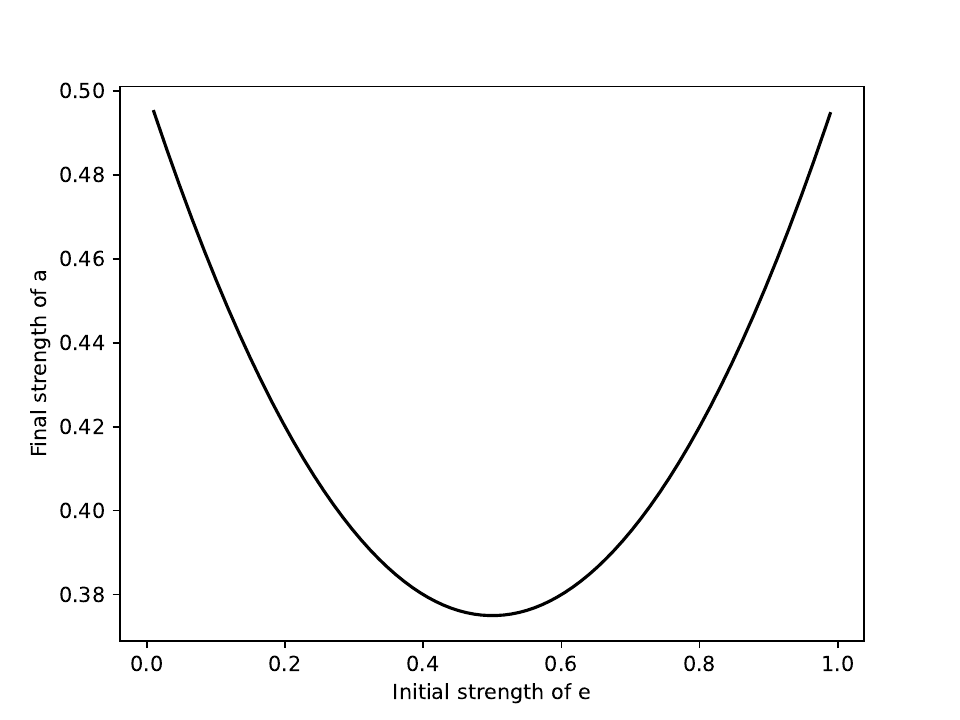}}
\caption{QBAG $\graph$ (Example~\ref{ex:intro-modular}). A node labelled {\scriptsize $\argnode{\argx}{i}{f}$} represents argument $\argx$ with initial strength $\is(\argx) = i$ and final strength $\fs(\argx) = \mathbf{f}$. Edges labelled $+$ and $-$, respectively, represent supports and attacks.}
\label{fig:intro}
\end{figure}  
\end{example}
The previous example highlights that a principle-based analysis is required to shed light on contribution functions in a rigorous manner.
After introducing some preliminaries in Section~\ref{sec:Preliminaries}, we present four contribution functions for QBA (Section~\ref{sec:Contributions})\footnote{We first introduced these contribution functions in a workshop paper~\cite{Cyras:Kampik:Weng:2022}; here, we fix some technicalities for the Shapley value-based contribution function. The first two contribution functions have been introduced first to abstract gradual argumentation in a roughly analogous manner in~\cite{Delobelle:Villata:2019}.}:
\begin{description}
    \item[Removal.] The contribution of one argument to another is determined by the effect that the removal of the former argument has on the final strength of the latter.
    \item[Intrinsic removal.] The contribution of one argument to another is determined by the effect that the removal of the former argument, \emph{from a QBAG from which all attack and support relationships that target the contributing argument have already been removed}, has on the final strength of the latter.
    \item[Shapley values.] The contribution of one argument to another is determined by the contribution of the former argument to the final strength of the latter when interpreting the problem as a coalitional game and applying Shapely values~\cite{shapley1951notes} to it.
    \item[Gradient.] The contribution of one argument to another is determined by the effect that a marginal change to the initial strength of the former argument has on the final strength of the latter.
\end{description}
We then introduce a set of principles (Section~\ref{sec:principles}), following intuitions underlying some of the contribution functions, as well as common sense-based desiderata: 
\begin{description}
    \item[Directionality.] 
    Many quantitative argumentation
    semantics satisfy a \emph{directionality} principle \cite{amgoud2018evaluation}: argument $\arga$ can influence argument $\argb$ only if there is
    a directed path from $\arga$ to $\argb$. Consequently, under these
    semantics, if there is no path from $\arga$ to $\argb$, then $\arga$'s contribution to $\argb$ should be $0$.
    \item[(Quantitative) contribution existence.] If our topic argument's initial strength does not equal its final strength then there exists another argument whose contribution to the topic argument is not $0$. In the quantitative case, the sum of the contributions of all other arguments (to the topic argument) should account for the delta between the topic's final and initial strength (where negative contributions give rise to a decrease in the strength of the topic argument).
    \item[(Quantitative) local faithfulness.] If an argument's contribution is negative/positive, then marginally increasing the initial strength of the contributor should increase/decrease the strength of the topic argument. In the quantitative case, the contributor's contribution to the topic argument should be proportional to the change in strength.
    \item[(Quantitative) counterfactuality.] The contribution should be negative if (and only if) removing the contributor from the QBAG increases the final strength of the topic argument and positive if (and only if) the removal decreases it. In the quantitative case, the delta between the topic argument's final strength in the initial QBAG and the topic's argument final strength with the contributor removed should be equal to the contribution of the contributor to the topic argument.
\end{description}
Subsequently, we analyse principle satisfaction of the four contribution functions with respect to five argumentation semantics (Section~\ref{sec:analysis}). The results are summarised in Table~\ref{table:principle-overview}.
\begin{table}
\centering
\caption{Satisfaction and violation of contribution principles (defined in Section~\ref{sec:analysis}), given a semantics $\fs$ (defined in Section \ref{sec:Preliminaries}) and a contribution function $\ctrbempty$ (defined in Section \ref{sec:Contributions}).}
\label{table:principle-overview}
\begin{tabular}{ |c|c|c|c|c| c| }
\hline
\textbf{$\ctrbempty$ / $\sigma$} & \textbf{QE}  & \textbf{DFQuAD} & \textbf{SD-DFQuAD} & \textbf{EB}  & \textbf{EBT} \\
\hline
\multicolumn{6}{|c|}{\textbf{Contribution Existence}} \\
\hline
$\ctrbrempty$ & \cmark & \xmark & \xmark & \cmark & \xmark \\ \hline
$\ctrbriempty$ & \cmark & \xmark & \xmark & \cmark & \xmark \\ \hline
$\ctrbsempty$ & \cmark & \cmark & \cmark & \cmark & \cmark \\ \hline
$\ctrbgempty$ & \cmark & \xmark & \xmark & \cmark & \xmark \\ \hline
\hline
\multicolumn{6}{|c|}{\textbf{Quantitative Contribution Existence}} \\
\hline
$\ctrbrempty$ & \xmark & \xmark & \xmark & \xmark & \xmark \\ \hline
$\ctrbriempty$ & \xmark & \xmark & \xmark & \xmark & \xmark \\ \hline
$\ctrbsempty$ & \cmark & \cmark & \cmark & \cmark & \cmark \\ \hline
$\ctrbgempty$ & \xmark & \xmark & \xmark & \xmark & \xmark \\ \hline
\hline
\multicolumn{6}{|c|}{\textbf{Directionality}} \\
\hline
$\ctrbrempty$ & \cmark & \cmark & \cmark & \cmark & \cmark \\ \hline
$\ctrbriempty$ & \cmark & \cmark & \cmark & \cmark & \cmark \\ \hline
$\ctrbsempty$ & \cmark & \cmark & \cmark & \cmark & \cmark \\ \hline
$\ctrbgempty$ & \cmark & \cmark & \cmark & \cmark & \cmark \\ \hline
\hline
\multicolumn{6}{|c|}{\textbf{(Quantiative) Local Faithfulness}} \\
\hline
$\ctrbrempty$ & \xmark & \xmark & \xmark & \xmark & \xmark \\ \hline
$\ctrbriempty$ & \xmark & \xmark & \xmark & \xmark & \xmark \\ \hline
$\ctrbsempty$ & \xmark & \xmark & \xmark & \xmark & \xmark \\ \hline
$\ctrbgempty$ & \cmark & \cmark & \cmark & \cmark & \cmark \\ \hline
\hline
\multicolumn{6}{|c|}{\textbf{(Quantitative) Counterfactuality}} \\
\hline
$\ctrbrempty$ & \cmark & \cmark & \cmark & \cmark & \cmark \\ \hline    
$\ctrbriempty$ & \xmark & \xmark & \xmark & \xmark & \xmark \\ \hline
$\ctrbsempty$ & \xmark & \xmark & \xmark & \xmark & \xmark \\ \hline
$\ctrbgempty$ & \xmark & \xmark & \xmark & \xmark & \xmark \\ \hline
\end{tabular}
\end{table}
These results are restricted to acyclic QBAGs (the focus of our study, as pointed out earlier). Note that
this special case is worth
studying because in many nascent applications of QBAGs there is a clear hierarchical or temporal structure to the arguments
that prevents the existence of cycles~\cite{kotonya2019gradual,cocarascu2019extracting,lidecision,chi2021optimized,Potyka0T23,AyoobiPT23}.
Note also that 
many ideas that we consider here can be extended 
to general QBAGs with
cycles, but the existence of cycles causes some additional 
technical difficulties (given that %
strength values may not converge in
particular situations) whose discussion would obscure our main points.
Hence, we leave the exploration of contribution functions for general QBAGs to future work.

We also provide some minor results and conjectures that we produced as a side-effect of our work (Section~\ref{sec:rest}).
To highlight practical perspectives, we position our results in the context of an application example (Section~\ref{sec:case-study})
Finally, we conclude the paper by discussing our results in the light of related research (Section~\ref{sec:discussion}).

\section{Preliminaries}
\label{sec:Preliminaries}

This section introduces the formal preliminaries of our work.
Let $\mathbb{I}$ be a
real interval of \emph{strength values}.
Here, we will focus on the unit interval $\mathbb{I} = [0, 1]$.
A \emph{quantitative bipolar argumentation graph} contains a set of arguments related by binary \emph{attack} and \emph{support} relations, and assigns an \emph{initial strength} in $\mathbb{I}$ to the arguments.
The (initial) strength can be thought of as initial credence in, or importance of, arguments. 
Typically, the greater the strength in  $\interval$, 
the more credible or important the argument is.
\begin{definition}[Quantitative Bipolar Argumentation Graph (QBAG)~\cite{Potyka:2019,Baroni:Rago:Toni:2019}]
A \emph{Quantitative Bipolar Argumentation Graph (QBAG)} is a quadruple
$\QBAG$ consisting of  a set of arguments 
$\Args$, an \emph{attack} relation 
$\Att \subseteq \Args \times \Args$, a \emph{support} relation 
$\Supp \subseteq \Args \times \Args$
such that $\Att \cap \Supp = \emptyset$,
and a total function 
$\is : \Args \to \interval$ 
that assigns an \emph{initial strength} $\is(\argx)$ to every $\argx \in \Args$.
\end{definition}

Given a QBAG $\QBAG$ and $\arga, \argb \in \Args$, a path from $\arga$ to $\argb$ is a path from $\arga$ to $\argb$ in the directed graph $(\Args, \Att \cup \Supp)$.
We say that a QBAG $\QBAG$ is acyclic iff there exists no $\arga \in \Args$ for which there exists a path from $\arga$ to $\arga$.
Henceforth, we assume as given a fixed acyclic QBAG $\graph = \QBAG$
with a finite number of arguments.

Given $\arga, \argx \in \Args$, the set $\Att_{\graph}(\argx) := \{ \argz \in \Args \mid (\argz, \argx) \in \Att \}$ is the set of attackers of $\argx$ and
each $\argz \in \Att_{\graph}(\argx)$ is a (direct) \emph{attacker} of $\argx$ (\emph{$\argz$ attacks $\argx$}); 
the set $\Supp_{\graph}(\argx) := \{ \argy \in \Args \mid (\argy, \argx) \in  \Supp \}$ is the set of supporters of $\argx$ and 
each $\argy \in \Supp_{\graph}(\argx)$ is a (direct) \emph{supporter} of $\argx$ (\emph{$\argy$ supports $\argx$}).
We may drop the subscript ${\graph}$ when the context is clear.
$\argx$ is an indirect attacker of $\arga$ iff $\argx$ is a supporter of some direct or indirect attacker of $\arga$ or $\argx$ attacks some direct or indirect supporter of $\arga$. $\argx$ is an indirect supporter of $\arga$ iff $\argx$ attacks some direct or indirect attacker of $\arga$ or $\argx$ supports some direct or indirect supporter of $\arga$.
Given $\graph = \QBAG$ and $A \subseteq \Args$, we define the QBAG \emph{restriction} of $\graph$ to $A$, denoted by $\graph\downarrow_{A}$, as follows: $\graph\downarrow_{A} := \left(A, \tau \cap (A \times \interval), \Att \cap (A \times A), \Supp \cap (A \times A) \right)$.

Reasoning in QBAGs amounts to updating the initial strengths of arguments to their final strengths, taking into account the strengths of attackers and supporters. 
Specifically, given a QBAG, a strength function assigns final strengths to arguments in the QBAG. 
Different ways of defining a strength function are called gradual semantics.
\begin{definition}[Gradual Semantics and Strength Functions~\cite{Baroni:Rago:Toni:2019,Potyka:2019}]\label{semantics-strength-functions}
A \emph{gradual semantics} $\fs$ defines for $\graph = \QBAG$ a (possibly partial) \emph{strength function} $\fs_{\graph} : \Args \to \interval \cup \{ \bot \}$ that assigns the \emph{final strength} $\fs_{\graph}(\argx)$ to each $\argx \in \Args$, 
where $\bot$ is a reserved symbol meaning `undefined'. 
\end{definition}
If $\graph$ is clear from the context, 
we will drop the subscript $\phantom{}_{\graph}$ and denote the 
strength function by $\fs$.

Many QBAG semantics from the literature including \cite{baroni2015automatic,rago2016discontinuity,amgoud2018evaluation,potyka2018continuous,Potyka21}
can be seen as \emph{modular semantics}
\cite{mossakowski2018modular}. Modular semantics
define the final strength values of arguments by an iterative
procedure that works as follows:
\begin{description}
    \item[Initialization:] Let $(\argx_1, \dots, \argx_n)$ be an arbitrary but fixed ordering of the arguments in $\Args$. Define a vector
    $s^{(0)} \in \mathbb{R}^n$ of initial strength values of the arguments
    by letting $s^{(0)}_i = \is(\argx_i)$.
    \item[Update:] Given the vector $s^{(i)}$ from the previous
    iteration, apply an update function $u_{\graph}: \mathbb{R}^n \rightarrow \mathbb{R}^n$ to obtain 
    $s^{(i+1)} = u_{\graph}(s^{(i)})$.
\end{description}
$u_{\graph}(s^{(i)})$ will adapt the strength of every
argument $\argx_i$ in two steps.
\begin{description}
    \item[Aggregation.] First an aggregation function $\alpha$ combines the current strength values of $\argx_i$'s supporters and attackers with respect to  $s^{(i)}$ to a single value that reflects how $\argx_i$ should be adapted. For representing attack, support, and lack of both attack and support, we may utilise a vector $v \in \{-1, 0, 1\}^n$ in which each attacker of $\argx_i$ is represented by $-1$, each supporter by $1$, and each argument neither attacking nor supporting $\argx_i$ by $0$.
    \item[Influence.] Then an influence function $\iota$ will determine a new strength value for $\argx_i$ based on the aggregate and $\argx_i$'s initial strength $w = \is(\argx_i)$.
\end{description}
Table~\ref{table:semantics} gives a list of common influence and aggregation functions.
Table~\ref{table:semanticsExamples} shows some examples of gradual semantics.
With abuse of notation, we may denote the evaluation of the aggregation function of an argument $\arga$ given a QBAG $\graph = \QBAG$ (with $\arga \in \Args$) by $\alpha_{\graph}(\arga)$, i.e.,  $\alpha_{\graph}(\arga) := \alpha(\arga_0, \dots, \arga_n)$, where $\arga_0, \dots, \arga_n$ are the supporters and attackers of $\arga$ in $\graph$.
Analogously, we may denote the evaluation of $\arga$'s influence function $\iota$ given $\graph$ by $\iota_\graph(\arga)$.

\begin{table}[ht]
\footnotesize{
\begin{tabular}{lll}
\hline
\multicolumn{3}{c}{\textbf{Aggregation Functions}} \\ \hline
Sum & $\alpha^{\Sigma}_{v}: [0,1]^n \rightarrow \mathbb{R}$ & $\alpha^{\Sigma}_{v}(s) = \sum_{i = 1}^n v_i \times s_i $  \\
Product & $\alpha^{\Pi}_{v}: [0,1]^n \rightarrow [-1, 1]$ & $\alpha^{\Pi}_{v}(s) = \prod_{i:v_i=-1} (1 - s_i) - \prod_{i:v_i=1} (1 - s_i)$  \\
Top & $\alpha^{max}_{v}: [0,1]^n \rightarrow [-1, 1]$ & $a_v^{max}(s) = M_v(s) - M_{-v}(s),$ \\
 &  & where $M_v(s) = max\{0,v_1 \times s_1, \dots, v_n \times s_n\}$  \\
\hline
\multicolumn{3}{c}{\textbf{Influence Functions}} \\ \hline
Linear($k$) & $\iota^{l}_{w}: [-k, k] \rightarrow [0, 1]$ & $\iota^{l}_{w}(s) = w - \frac{w}{k} \times max\{0,-s\} + \frac{1-w}{k} \times max\{0, s\}$ \\
Euler-based & $\iota^{e}_{w}: \mathbb{R} \rightarrow [w^2, 1]$ & $\iota^{e}_{w}(s) = 1 - \frac{1-w^2}{1 + w \times e^s}$ \\
p-Max($k$) & $\iota^{p}_{w}: \mathbb{R} \rightarrow [0, 1]$ & $\iota^{p}_{w}(s) = w - w \times h(- \frac{s}{k}) + (1-w) \times h(\frac{s}{k}),$  \\
for $p \in \mathbb{N}$ &  & where $h(x) = \frac{max\{0,x\}^p}{1 + max\{0,x\}^p}$  \\
\end{tabular}}
\caption{Common aggregation $\alpha$ and influence $\iota$ functions~\cite[pp.~1724 Table 1; with a fixed typo for p-Max($k$)]{Potyka:2019}. Intuitively, $s \in [0,1]^n$ is a strength vector 
(associating each argument with its current strength),
$v \in \{-1,0,1\}^n$ is a relationship vector indicating
which arguments attack ($-1$), support ($1$) or are in no 
relationship to ($0$) the argument of interest,
and $w$ is an initial strength.
}
\label{table:semantics}
\end{table}

\begin{table}[ht]
\centering
\begin{tabular}{lll}
\hline
\textbf{Semantics}           & \textbf{Aggregation} & \textbf{Influence}  \\ \hline
QuadraticEnergyModel (QE)        & Sum         & 2-Max(1)  \\
DFQuADModel (DFQuAD)       & Product     & Linear(1) \\
SquaredDFQuADModel (SD-DFQuAD) & Product     & 1-Max(1)  \\
EulerBasedModel (EB)   & Sum         & EulerBased  \\
EulerBasedTopModel (EBT) & Top         & EulerBased 
\end{tabular}
\caption{Examples of gradual semantics.}
\label{table:semanticsExamples}
\end{table}

Formally, the final strength $\fs_{\graph}(\argx_i)$
of the i-th argument is the $i$-th component $s_i$ of the limit vector
$s = \lim_{i \rightarrow \infty} s^{(i)}$. As demonstrated
in \cite{mossakowski2018modular}, the limit may not exist in
cyclic graphs. However, as we focus on acyclic graphs here,
the limit is guaranteed to exist and can, in fact, be computed
in linear time by a simple forward propagation procedure \cite{Potyka:2019}.
Roughly speaking, it works by first computing
a topological ordering of the arguments. The
strength values are then computed following this order.
The final strength of the arguments
without parents (i.e.\ either attackers or supporters) is just their initial strength. When the next argument in the topological order
is selected,
all strength values of its parents are
already fixed, so its final strength
can be computed by a single application of
the update function.

As we can easily see, the semantics in Table~\ref{table:semanticsExamples} are modular semantics.
For example, using the Sum aggregation function and the 2-Max(1) influence function yields the
Quadratic Energy (QE) semantics.
For a step-by-step walk-through of final strength computation using a modular semantics, we refer to Example~\ref{ex:intro-modular}.

To facilitate the analysis and comparison of gradual semantics, formal principles have been defined.
In the context of this paper, the \emph{directionality} and \emph{stability} principles~\cite{amgoud2018evaluation} are of particular relevance.
Directionality stipulates that removing arguments that cannot reach an argument of interest from the QBAG must not affect the latter argument's final strength.
When defining the principles we assume an arbitrary (acyclic and finite) QBAG $\graph = \QBAG$ and $\arga \in \Args$.
Directionality states, intuitively, that an argument $\arga$ can only
be influenced by another argument $\argx$ if there is a directed path from $\argx$ to $\arga$.
\begin{principle}[Directionality (Gradual Semantics)]\label{sprinciple:directionality}
    A gradual semantics $\sigma$ satisfies the \emph{directionality principle} iff
 for all 
$\graph = (\Args, \is, \att, \supp)$ and
$\graph' = (\Args, \is, \att', \supp')$
such that $\att \cup \supp = \att' \cup \supp' \cup (\argx, \argy)$ for some $\argx \in \Args$, $\argy \in \Args \setminus \{\arga\}$ such that there is no directed path from $\argy$ to 
$\arga$, we have $\fs_{\graph}(\arga) = \fs_{\graph'}(\arga)$.
\end{principle}
Stability stipulates that an argument's final strength equals its initial strength in the absence of supporters and attackers.
\begin{principle}[Stability]\label{sprinciple:stability}
    A QBAG semantics $\sigma$ satisfies the \emph{stability principle} iff $\Att_\graph(\arga) = \Supp_\graph(\arga) = \emptyset$ implies that $\fs(\arga) = \is(\arga)$.
\end{principle}
All semantics in Table~\ref{table:semanticsExamples} satisfy directionality and stability \cite{amgoud2018evaluation,mossakowski2018modular,Potyka:2018}.

\section{Argument Contributions}
\label{sec:Contributions}

To determine how much an argument $\argx \in \Args$ \emph{contributes} to the final strength $\fs_\graph(\arga)$ of a \emph{topic} argument $\arga \in \Args$ in a given QBAG $\graph$ (given a semantics $\sigma$), we make use of contribution functions $\ctrbempty_{\fs,\graph,\arga}: \Args \rightarrow \mathbb{R} \cup \{\bot\}$. We drop the subscripts $\fs$ and $\graph$ where the context is clear.
Intuitively, $\ctrbempty_\arga(\argx)$ measures the influence
of $\argx$ on the topic argument $\arga$. The exact meaning of
$\ctrbempty_\arga(\argx)$ depends on the contribution function.
We will investigate the following contribution functions from~\cite{Cyras:Kampik:Weng:2022}.

\begin{description}[leftmargin=*]
    \item[Removal-based contribution:] Intuitively, we ask: how  does the removal of the contributor affect the topic argument's final strength? For this, we determine the difference between the final strength of the topic argument in $\graph$ and its final strength in $\graph$ with the contributor removed, i.e.\ we compute $\fs_\graph(\arga) - \fs_{\graph'}(\arga)$, where $\graph'$ is obtained from $\graph$ by removing $\argx$:
    \begin{align}\label{eq:delta}
       \ctrbr{\argx}{\arga} = \fs_\graph(\arga) - \fs_{\graph\downarrow_{\Args \setminus \{\argx\}}}(\arga) 
    \end{align}
    \item[Removal-based contribution without indirection:] Continuing with the intuition above, we may require that the contribution of one argument to another one should be based on the former's \emph{intrinsic} (initial) strength and not on boosts received by (direct or indirect) attackers and supporters~\cite{Delobelle:Villata:2019}. To reflect this criticism of the removal-based contribution function, we may first remove all attacks and supports that have the contributor as a target before determining the difference (assuming that our semantics is \emph{directional}, i.e.\ that strength updates are propagated only in the direction represented by the support and attack relations); i.e., we compute $\fs_{\graph^-}(\arga) - \fs_{\graph'}(\arga)$ where $\graph^-$ is obtained from $\graph$ by removing direct relations to $\argx$ and $\graph'$ is obtained from $\graph$ by removing $\argx$ entirely:
    \begin{align}\label{eq:delta2}
    \ctrbri{\argx}{\arga} = \fs_{(\Args, \tau, \Att \setminus \{(\argy, \argx) | (\argy, \argx) \in \Att \}), \Supp \setminus \{(\argy, \argx) | (\argy, \argx) \in \Supp \})}(\arga) - \fs_{\graph\downarrow_{\Args \setminus \{\argx\}}}(\arga)
    \end{align}
    
    \item[Shapley-based contribution:]
    Yet more elaborately, we can see the difference between a topic argument's final and initial strengths as a value achieved by a coalition of all other arguments in our QBAG, which we can then consider \emph{players} in the game theoretical sense.
    Then, we can use Shapley values \cite{shapley1953value} as a well-established approach to quantify the argument contributions\footnote{Note that we fixed some technicalities in the Shapley value-based contribution function initially presented in~\cite{Cyras:Kampik:Weng:2022}.}.
    Intuitively, we determine a weighted average of the topic argument's final strength difference that is achieved by adding the contributor to all possible subgraphs of our QBAG that already contain the topic argument:
    \begin{align}\label{eq:shap}
    &{} \ctrbs{\argx}{\arga}=  \nonumber \\
    &{} \quad \sum_{X \subseteq \Args \setminus \{ \argx, \arga \}} \frac{|X|! \cdot (|\Args \setminus \{\arga\}|-|X|-1)!}{|\Args \setminus \{\arga\}|!}\left(\operatorname{\fs}_{\graph\downarrow_{\Args \setminus X}}(\arga) - \operatorname{\fs}_{\graph\downarrow_{\Args \setminus (X \cup \{ \argx \})}}(\arga)\right)
    \end{align}
    
    \item[Gradient-based contribution:] 
    The last contribution function is based on the \emph{gradient} 
    from real analysis~\cite{Davies.et.al:2021}: 
    intuitively, the gradient is a vector 
    (whose components are partial derivatives)
    that measures how a small change in an input parameter of a function
    affects the output. In our case, the function of interest 
    maps the initial strength values of arguments to the final
    strength of the topic argument. Formally, we consider a 
    function $f_{\arga}: \mathbb{R}^n \rightarrow \mathbb{R}$ 
    that maps a vector $s^{(0)} \in \mathbb{R}^n$ of initial strength values to the final strength
    $f_{\arga}(s^{(0)})$ of $\arga$.
    
    Since we assume that $\graph$ is acyclic, we can construct  
    an explicit representation of $f_{\arga}$ 
    recursively. 
    Intuitively, the representation of $f_{\arga}$ corresponds
    to a long composition of influence and aggregation functions of the
    semantics applied to the initial strength values in $s^{(0)}$.
    Let $i$ be the index of the topic argument $\arga$ 
    in the topological ordering 
    $\{  \argx_1, \ldots, \argx_n \}$ associated with $\graph$.
    If $\arga$ does not have any parents, then its final strength 
    is just $s^{(0)}_i$ by \emph{stability} (Principle~\ref{sprinciple:stability}) of the semantics. In this case, $f_{\arga}$ is just the 
    constant function that always returns $s^{(0)}_i$.
    If, on the other hand, $\arga$ has the set of parents $S = \{\argx_{i_1}, \dots, \argx_{i_{k_\arga}}\}$, we create placeholders $T_{i_1}, \dots, T_{i_{k_\arga}}$ for the explicit representation of the strength
    value computation of the parents and let 
    \begin{align}
        f_{\arga} = \iota_{w_\arga}(\alpha_{v_\arga}(T_{i_1}, \dots, T_{i_{k_\arga}})),
    \end{align}
    where $w_\arga = \is(\arga)$ is $\arga$'s initial strength
    and $v_\arga$ is the parent vector that contains $-1/0/1$
    at the $i$-th component if $\argx_i$ attacks/is unrelated to/supports
    $\arga$.
    We now have to replace $T_{i_1}, \dots, T_{i_{k_\arga}}$ with
    explicit function representations of the strength of
    $\argx_{i_1}, \dots, \argx_{i_{k_\arga}}$
    as well. This can be done exactly as for $\arga$. In the process, we may create new 
    placeholders, but since $\graph$ is acyclic, the recursion
    will stop after a finite number of steps. A naive implementation
    of the recursive scheme may compute the same template multiple times.
    This can be avoided by storing templates
    (there can be at most $|\Args|$) or using a dynamic programming approach. 
    
    Let us note that, in the process of constructing the explicit
    function representation of $f_{\arga}$, we can only add
    strength values $s^{(0)}_i$ of arguments $\argx_i$ such that there is a
    directed path from $\argx_i$ to $\arga$ in $\graph$. This is in line with the
    directionality property (Principle~\ref{sprinciple:directionality}) of gradual semantics. We emphasise this observation as it will
    be useful in the following.
    \begin{observation}
    $f_{\arga}(s^{(0)})$ depends on $s^{(0)}_i$ if and only if there
    is a directed path from $\argx_i$ to $\arga$.
    \end{observation}

    The gradient-based contribution function is now defined as
    \begin{equation}
    \label{eq:gradient}
    \ctrbg{\argx}{\arga} = \frac{\partial f_{\arga}}{\partial \is(\argx)} \left(\is(\argx_1), \ldots, \is(\argx_n)\right).
    \end{equation}
    With abuse of notation, we may also denote the partial derivative $\frac{\partial f_{\arga}}{\partial \is(\argx)}$ as $\frac{\partial \fs(\arga)}{\partial \is(\argx)}$.
\end{description}

Irrespective of the specific method for determining an argument's contribution to the final strength of some topic argument in a given QBAG $\graph$, we use a generic \emph{contribution function} $\ctrbempty:~\Args \times \Args \to \mathbb{R}$ defined by $\ctrbempty(\argx, \arga) = \ctrb{\argx}{\arga}$. 
For reference, we use $\ctrbrempty, \ctrbriempty, \ctrbsempty, \ctrbgempty$ to denote the contribution functions given, respectively, in Equations~\ref{eq:delta}--\ref{eq:shap} and \ref{eq:gradient}. 

Let us illustrate the intuitions behind the contribution functions with an example.
\begin{example}
\label{ex:ctrb}
Consider the QBAF $G$ in Figure~\ref{fig:ctrb-g}.
We use DFQuAD semantics and are interested in quantifying the contribution of $\argb$ to $\arga$.
The removal-based contribution $\ctrbr{\argb}{\arga} = -0.375$ is obtained by subtracting the final strength of $\arga$ in $\graph'$ from the final strength of $\arga$ in $\graph$.
However, one may argue that $\argb$'s contribution to $\arga$ is not entirely intrinsic to $\argb$, as $\argb$ benefits from its supporter $\argc$, following a contribution quantification approach first introduced to abstract gradual argumentation~\cite{Delobelle:Villata:2019}.
To determine the intrinsic contribution, we may first remove all attack and support relationships to $\argb$ to then determine the difference of $\arga$'s final strengths with and without $\argb$, as seen in Figures~\ref{fig:graph-intrinsic-removal} and~~\ref{fig:graph-removal}, respectively: $\ctrbri{\argb}{\arga} = -0.25$. Still, one may argue that this comes at the cost of the simpler intuition behind $\ctrbrempty$: $\ctrbriempty$is not entirely aligned with the topological reality of $\graph$ and hence neither with the intuition of the counterfactuality principle as sketched in the introduction.
Yet more nuanced, $\ctrbsempty$ determines a weighted average contribution of $\argb$ to $\arga$ as provided by $\ctrbrempty$ in each restriction of $\graph$ to some subset of its nodes that contains $\arga$ and $\argb$: $\ctrbs{\argb}{\arga} = -0.3125$.
Interestingly, we have $\ctrbs{\argb}{\arga} + \ctrbs{\argc}{\arga} = \tau(\arga) - \sigma(\arga)$, i.e.\ the sum of the contributions of all arguments to $\arga$ (note that the contribution of $\arga$ to itself is not defined) exactly accounts for the difference between $\arga$'s initial and final strength.
This is aligned with the intuition of (quantitative) contribution existence and can formally be traced back to the efficiency principle of Shapley values (see Subsection~\ref{subsec:ctrb-existence}).

Note that the three contribution functions we have covered so far all use some form of argument removal (and, indeed, always remove at least the contributing argument) to determine the contribution.
In contrast, the gradient contribution function measures the effect of marginally changing the contributing argument's initial strength on the final strength of the topic argument.
With the strength function expanded as $\fs(\arga) = \is(\arga) - \is(\arga)\fs(\argb) = 
\is(\arga)\left(1 - \left( \is(\argb) + \left(1-\is(\argb)\right)\fs(\argc)\right)\right) = 
\is(\arga) - \is(\arga)\is(\argb) + \is(\arga)\is(\argb)\is(\argc) - \is(\arga)\is(\argc)$, 
we for instance find $\frac{\partial}{\partial \is(\argb)} \fs(\arga) = -\is(\arga) + \is(\arga)\is(\argc)$,
which at the vector $(0.5, 0.5, 0.5)$ of initial strengths evaluates to 
$\ctrbg{\argb}{\arga} = -0.25$.
This also means that we can quantify the strength of an argument to itself: 
$\frac{\partial}{\partial \is(\arga)} \fs(\arga) = 1-\is(\argb) + \is(\argb)\is(\argc) - \is(\argc)$ evaluates at the initial strengths to 
$\ctrbg{\arga}{\arga} = 0.25$.
We may claim that, in contrast to the other contribution functions, the gradient contribution function answers a fundamentally different question, which is not about the argument's existence, but about the local effect of changing its initial strength.
Table~\ref{table:example} shows all contributions that can be computed given $\graph$.
This overview already gives us an intuition of some of the principles, such as the quantitative contribution existence principle of $\ctrbsempty$; also, note that contribution functions $\ctrbrempty, \ctrbriempty, \ctrbsempty$ given, respectively, in Equations~\ref{eq:delta}, \ref{eq:delta2}, \ref{eq:shap} are partial: they do not define an argument's contribution to itself $\ctrb{\arga}{\arga}$, since $\fs_{\graph'}(\arga)$ is undefined when $\arga$ is not in (the set of arguments of) $\graph'$. 
In contrast, the gradient contribution function $\ctrbgempty$ (\autoref{eq:gradient}) is total\footnote{For the sake of conciseness, we abstain from formalising this trivial observation.}. 
 \begin{figure}[ht]
\centering
\subfloat[$\graph$.]{
\label{fig:ctrb-g}
\begin{tikzpicture}[scale=0.8]
    \node[invnode]    (inv1)    at(-2,0)  {};
    \node[invnode]    (inv2)    at(2,0)  {};
    \node[unanode]    (a)    at(0,0)  {\argnode{\arga}{0.5}{0.125}};
    \node[unanode]    (b)    at(0,2)  {\argnode{\argb}{0.5}{0.75}};
    \node[unanode]  (c)    at(0,4)  {\argnode{\argc}{0.5}{0.5}};
    
    \path [->, line width=0.2mm]  (b) edge node[left] {-} (a);
    \path [->, line width=0.2mm]  (c) edge node[left] {+} (b);
\end{tikzpicture}
}
\hspace{10pt}
\subfloat[$\graph'$ ($\graph$, with $\argb$ removed).]{
\label{fig:graph-removal}
\begin{tikzpicture}[scale=0.8]
    \node[invnode]    (inv1)    at(-2,0)  {};
    \node[invnode]    (inv2)    at(2,0)  {};
    \node[unanode]    (a)    at(0,0)  {\argnode{\arga}{0.5}{0.5}};
    \node[unanode]  (c)    at(0,4)  {\argnode{\argc}{0.5}{0.5}};
    
\end{tikzpicture}
}
\hspace{10pt}
\subfloat[$\graph''$ ($\graph$, with incoming relationships to $\argb$ removed).]{
\label{fig:graph-intrinsic-removal}
\begin{tikzpicture}[scale=0.8]
    \node[invnode]    (inv1)    at(-2,0)  {};
    \node[invnode]    (inv2)    at(2,0)  {};
    \node[unanode]    (a)    at(0,0)  {\argnode{\arga}{0.5}{0.25}};
    \node[unanode]    (b)    at(0,2)  {\argnode{\argb}{0.5}{0.5}};
    \node[unanode]  (c)    at(0,4)  {\argnode{\argc}{0.5}{0.5}};
    
    \path [->, line width=0.2mm]  (b) edge node[left] {-} (a);
\end{tikzpicture}
}
\caption{A QBAG $\graph$ and its updates, with $\argb$ ($\graph'$) and incoming relationships to $\argb$ removed ($\graph''$).}
\label{fig:ctrb}
\end{figure}
\begin{table}
\caption{Argument contributions, given DFQuAD semantics, in $\graph$ (Figure~\ref{fig:ctrb-g}), given different contribution functions (contributors are row headers, topic arguments column headers; $\bot$ denotes undefined contributions.}
\label{table:example}
\begin{center}
\renewcommand*{\arraystretch}{1.2}
\begin{tabular}{ |c|c|c|c| } 
 \hline
 $\ctrbrempty$ & $\arga$ & $\argb$ & $\argc$ \\
 \hline
 $\arga$ & $\bot$ & $0$ & $0$  \\ 
 $\argb$ & $-0.375$\phantom{e} & $\bot$ & $0$ \\ 
  $\argc$ & $-0.125$\phantom{e} & $0.25$ & $\bot$ \\ 
 \hline 
\end{tabular} \hspace{10pt}
\begin{tabular}{ |c|c|c|c| } 
 \hline
 $\ctrbriempty$ & $\arga$ & $\argb$ & $\argc$ \\ 
 \hline
 $\arga$ & $\bot$ & $0$ & $0$  \\ 
 $\argb$ & $-0.25$\phantom{e} & $\bot$ & $0$ \\ 
  $\argc$ & $-0.125$ & $0.25$ & $\bot$ \\ 
 \hline
\end{tabular} \\ \vspace{10pt}
\begin{tabular}{ |c|c|c|c| } 
 \hline
 $\ctrbsempty$ & $\arga$ & $\argb$ & $\argc$ \\ 
 \hline
 $\arga$ & $\bot$ & $0$ & $0$  \\ 
 $\argb$ & $-0.3125$ & $\bot$ & $0$ \\ 
  $\argc$ & $-0.0625$ & $0.25$ & $\bot$ \\ 
 \hline
\end{tabular} \hspace{10pt}
\begin{tabular}{ |c|c|c|c| } 
 \hline
 $\ctrbgempty$\phantom{'} & $\arga$ & $\argb$ & $\argc$ \\ 
 \hline
 $\arga$ & \phantom{e}$0.25$ & $0$ & $0$  \\ 
 $\argb$ & \phantom{e}$\;-0.25$\phantom{ee} & $0.5$ & $0$ \\ 
  $\argc$ & $-0.25$ & $0.5$ & $1$ \\ 
 \hline
\end{tabular}
\end{center}
\end{table}
\end{example}

\section{Principles}
\label{sec:principles}
We expect some intuitive behaviour from a contribution function and can formalise these expectations as contribution function principles, similar to principles that have been defined for gradual semantics~\cite{Baroni:Rago:Toni:2019}.
In the case of contribution functions, principle satisfaction obviously depends on the properties of the semantics used. 
A few contribution function principles were proposed in \cite{Cyras:Kampik:Weng:2022}. We recap those next and propose new ones in the following subsections.
In the remainder of this section, as well as of the next one, we restrict attention to acyclic QBAGs and fix otherwise arbitrary gradual semantics $\fs$, QBAG $\graph = \QBAG$ and topic argument $\arga \in \Args$, denoting the contributor by $\argx$ (with $\argx \in \Args$). 
A contribution function $\ctrbempty$ will be said to satisfy a principle w.r.t.\ $\fs$ if (and only if) a given condition holds for arbitrary (acyclic) $\graph$ and $\arga$, if not stated otherwise.

\subsection{Contribution Existence}
\label{subsec:ctrb-existence}
One would expect that whenever an argument's final strength differs from its initial strength, there is another argument whose contribution explains this difference.
Accordingly, the contribution existence principle stipulates that some contribution from another argument exists whenever the final strength of the topic argument changes. 
\begin{principle}[Contribution Existence~\cite{Cyras:Kampik:Weng:2022}\footnote{Note that we slightly adjusted the principle: the non-zero contribution needs to be explained by an argument other than the topic argument.}]
\label{principle:existence}
$\ctrbempty$ satisfies the \emph{contribution existence principle} w.r.t.\ a gradual semantics $\fs$
iff  $\sigma(\arga) \neq \is(\arga)$ implies that $\exists \argx \in \Args \setminus \{\arga\}$ with $\ctrb{\argx}{\arga} \neq 0$.
\end{principle}
Here and henceforth, where the context is clear we may drop ``w.r.t.\ a gradual semantics $\fs$'' for any principle. We can further strengthen the principle and stipulate that the sum of all contributions must equal the difference between the topic argument's initial strength and its final strength.
\begin{principle}[Quantitative Contribution Existence]
\label{principle:qexistence}
$\ctrbempty$ satisfies the \emph{quantitative contribution existence principle} w.r.t.\ a gradual semantics $\fs$
iff it holds that $\sum_{\argx \in \Args \setminus \{\arga\}} \ctrb{\argx}{\arga} = \sigma(\arga) - \tau(\arga)$.
\end{principle}
Intuitively, quantitative contribution existence implies contribution existence.
\begin{proposition}
\label{contribution-existence:implication}
If $\ctrbempty$ satisfies quantitative contribution existence w.r.t.\ a gradual semantics $\fs$ then $\ctrbempty$ satisfies contribution existence w.r.t.\ $\fs$.
\end{proposition}
\begin{proof}
Assume $\ctrbempty$ satisfies quantitative contribution existence w.r.t.\ $\fs$.
Observe that by definition of quantitative contribution existence (Principle~\ref{principle:qexistence}), if $\sigma(\arga) \neq \is(\arga)$ then $\sum_{\argx \in \Args \setminus \{\arga\}} \ctrb{\argx}{\arga} \neq 0$ and hence it must hold that $\exists \argx \in \Args \setminus \{\arga\}$ with $\ctrb{\argx}{\arga} \neq 0$. This means contribution existence (Principle~\ref{principle:existence}) must be satisfied, which proves the proposition.
\end{proof}

\subsection{Directionality}

Assuming that the gradual semantics satisfies \emph{directionality} (Principle~\ref{sprinciple:directionality} for gradual semantics), contribution functions should respect the direction of influence as represented by the topology of the QBAG.
Accordingly, the directionality principle for contribution functions stipulates that an argument's contribution to a topic argument can only be non-zero if there is a directed path from the former to the latter.

\begin{principle}[Directionality (Contribution Function)~\cite{Cyras:Kampik:Weng:2022}]
\label{principle:directionality}
$\ctrbempty$ satisfies the \emph{directionality principle} w.r.t.\ a gradual semantics $\fs$
iff whenever there is no directed path in $\graph$ from $\argx \in \Args$ to $\arga$, then $\ctrb{\argx}{\arga} = 0$.
\end{principle}
Intuitively, we can expect that our contribution functions satisfy directionality with respect to modular semantics that traverse a QBAG from leave nodes (without incoming supports or attacks).
However, directionality is not necessarily satisfied if we cannot make any assumptions about semantics behaviour.
In particular, if a gradual semantics 
makes an argument's final strength dependent on the arguments 
it attacks or supports (somewhat reflecting the idea of \emph{range-based} semantics in abstract argumentation~\cite{verheij1996two}),
then it does not respect
\emph{directionality}. 
For such a semantics, all of the contribution functions introduced above would violate the directionality principle
for contribution functions. 
For instance, consider the strength function $\fs(\argx) = \is(\argx) - \sum_{\{\argy | \argx \in \Att(\argy) \}} \tau(\argy)$, 
the QBAG $\graph = (\{\arga, \argb\}, \{(\arga, 1), (\argb, 1)\}, \{(\argb, \arga)\}, \{\})$ with $\argb$ attacking $\arga$, 
and $\ctrbrempty$ defined using \autoref{eq:delta} (removal-based contribution). 
There is no directed path from $\arga$ to $\argb$, but removing $\arga$ changes the final strength of $\argb$ from $0$
to $1$, whence the contribution of $\arga$ to $\argb$ is $0 - 1 \neq 0$ (see Figure~\ref{fig:directionality-violation}).
\begin{figure}[ht]
\centering
\subfloat[$\graph$.]{
\label{fig:ctrbg}
\begin{tikzpicture}[scale=0.8]
    \node[invnode]    (inv1)    at(-2,0)  {};
    \node[invnode]    (inv2)    at(2,0)  {};
    \node[unanode]    (a)    at(0,0)  {\argnode{\arga}{1}{1}};
    \node[unanode]    (b)    at(0,2)  {\argnode{\argb}{1}{0}};
    
    \path [->, line width=0.2mm]  (b) edge node[left] {-} (a);
\end{tikzpicture}
}
\hspace{10pt}
\subfloat[$\graph'$ ($\graph$, with $\arga$ removed).]{
\label{fig:graphremoval}
\begin{tikzpicture}[scale=0.8]
         \node[invnode]    (inv1)    at(-2,0)  {};
        \node[invnode]    (inv2)    at(2,0)  {};
        \node[invnode]    (inv3)    at(0,0)  {};
        \node[unanode]    (b)    at(0,2.5)  {\argnode{\argb}{1}{1}};
    
\end{tikzpicture}
}
\caption{While our contribution functions can be intuitively expected to satisfy directionality with respect to modular semantics, this may not be the case with respect to other semantics: consider $\ctrbrempty$ and $\fs(\argx) = \is(\argx) - \sum_{\{\argy | \argx \in \Att(\argy) \}} \tau(\argy)$.}
\label{fig:directionality-violation}
\end{figure}

\subsection{Faithfulness}
\label{sec:principle-faithfulness}
One may want to demand that a contribution function
\emph{faithfully} represents the effect of
one argument on another one.
The effect can be measured in different ways.
One natural way in QBAGs is to measure it based on the effect of the initial strength of an argument.
If increasing the initial strength of an argument increases (decreases) the final strength of the topic argument, then the contribution score should be positive (negative).
We can also imagine that there is no effect at all.
To formalise this intuition, we first define an initial strength modification of a QBAG.
\begin{definition}[QBAG Initial Strength Modification]
\label{defn:modification}
Given a QBAG $\graph = \QBAG$ and an argument $x \in \Args$, we define
the \emph{initial strength modification of $\argx$ in $\graph$} as the QBAG $\graph\downarrow_{\is(\argx) \leftarrow \epsilon} := \left(\Args, \tau', \Att, \Supp \right)$,
where $\epsilon \in \mathbb{I}$,
$\is'(\argx) = \epsilon$ and
$\is'(\argy) = \is(\argy)$ for all other
$\argy \in \Args \setminus \{\argx\}$.
\end{definition}
One natural first faithfulness property
could then be stated as follows: if an argument's contribution is positive/zero/negative, then increasing the argument's
base score should increase/not affect/decrease the strength
of the topic argument.
\begin{principle}[Strong Faithfulness]
\label{principle:strongCorrect}
$\ctrbempty$ satisfies the \emph{strong faithfulness principle} w.r.t.\ a gradual semantics $\fs$ iff for every QBAG $\graph = \QBAG$, for all $\arga, \argx \in \Args$,
$\epsilon \in \interval$ and $\graph_\epsilon = \graph\downarrow_{\is(\argx) \leftarrow \epsilon}$ the following statements hold:
\begin{itemize}
    \item If $\ctrb{\argx}{\arga} < 0$, then
    $\fs_\graph(\arga) <
    \fs_{\graph_\epsilon}(\arga)$
    whenever $\epsilon <\is(\argx)$
    and
    $\fs_\graph(\arga) > 
    \fs_{\graph_\epsilon}(\arga)$
    whenever $\epsilon >\is(\argx)$.
    \item If $\ctrb{\argx}{\arga} = 0$, then
    $\fs_\graph(\arga) = 
    \fs_{\graph_\epsilon}(\arga)$
    for all $\epsilon \in \interval$.
    \item If $\ctrb{\argx}{\arga} > 0$, then
    $\fs_\graph(\arga) > 
    \fs_{\graph_\epsilon}(\arga)$
    whenever $\epsilon <\is(\argx)$
    and
    $\fs_\graph(\arga) < 
    \fs_{\graph_\epsilon}(\arga)$
    whenever $\epsilon >\is(\argx)$.
\end{itemize}
\end{principle}
Intuitively, a negative (positive) score guarantees that increasing the initial strength of the evaluated argument will decrease (increase) the final strength of the topic argument.
However, this property can only be satisfied if we can guarantee a  monotonic effect of arguments.
This is not necessarily the case as we illustrate in Figure \ref{fig:intro}.
In the QBAG on the left (Figure~\ref{fig:graph-intro}), argument $\arge$ has a non-monotonic influence on the topic argument $\arga$. Figure~\ref{fig:plot-intro} plots the final strength of $\arga$ (y-axis) as a function of the initial strength of $\arge$ under DFQuAD semantics. $\arge$'s initial strength
has a negative influence up to $0.5$. Then the influence 
becomes positive.
The plot illustrates how the initial strength
of $\arge$ influences the final strength of $\arga$. As we increase
$\is(\arge)$ from $0$ to $0.5$, $\arga$ becomes weaker. However,
at this point, the effect reverses, and increasing $\is(\arge)$
further will make $\arga$ stronger.
In particular, if we let $\is(\arge) = 0.2$ in the QBAG on the left in Figure~\ref{fig:intro}, then the effect is negative
for $\epsilon \in [0,0.8)$, neutral for $\epsilon = 0.8$ and
positive for $\epsilon \in (0.8,1]$.
We provide analogous examples for QE semantics (Figure~\ref{fig:faith-qe}), as well as for EB semantics (Figure~\ref{fig:faith-eb}); for SD-DFQuAD and EBT semantics, we provide counterexamples for the case $\ctrb{\argx}{\arga} = 0$ ($\fs_\graph(\arga) \neq \fs_{\graph_\epsilon}(\arga)$ for at least some $\epsilon \in \interval$) in Figures~\ref{fig:faith-sd} and~\ref{fig:faith-ebt}, respectively.
Hence, we assume that strong faithfulness cannot be reasonably satisfied.
The property may be desirable for explainining ``monotonic'' QBAGs (see Section~\ref{sec:rest}), but is too strong for general acyclic QBAGs.

If an argument has a non-monotonic effect on the topic argument,
we may still be able to capture its effect faithfully in a small environment of the initial strength, where it behaves monotonically. However, even the local effect cannot always be
categorised as positive, negative, or neutral. 
For example,  when we let $\is(\arge)=0.5$ in 
Figure \ref{fig:graph-intro}, the local effect switches from
positive to negative at this point and cannot really be 
classified as any of the three. 
In our definition of local faithfulness, we therefore only 
demand that positive and negative effects are faithfully captured.
$0$ may mean that there is no local effect or that the local
effect changes at this point. These considerations motivate
the following definition.
\begin{principle}[Local Faithfulness]
\label{principle:localCorrect}
$\ctrbempty$ satisfies the \emph{local faithfulness principle} w.r.t.\ a gradual semantics $\fs$
iff there exists a $\delta > 0$ such that for all
$\epsilon \in [\is(\argx) - \delta, \is(\argx) + \delta] \cap \interval$ and $\graph_\epsilon = \graph\downarrow_{\is(\argx) \leftarrow \epsilon}$ the following statements hold:
\begin{itemize}
    \item If $\ctrb{\argx}{\arga} < 0$, then
    $\fs_\graph(\arga) <
    \fs_{\graph_\epsilon}(\arga)$
    whenever $\epsilon <\is(\argx)$
    and
    $\fs_\graph(\arga) > 
    \fs_{\graph_\epsilon}(\arga)$
    whenever $\epsilon >\is(\argx)$.
    \item If $\ctrb{\argx}{\arga} > 0$, then
    $\fs_\graph(\arga) > 
    \fs_{\graph_\epsilon}(\arga)$
    whenever $\epsilon <\is(\argx)$
    and
    $\fs_\graph(\arga) < 
    \fs_{\graph_\epsilon}(\arga)$
    whenever $\epsilon >\is(\argx)$.
\end{itemize}
\end{principle}
As opposed to strong faithfulness, local faithfulness does not
consider all strength modifications, but only those in a small
$\delta$-environment of the initial strength of the contributing argument. Furthermore,
it only demands that positive (negative) scores guarantee a
positive (negative) effect of the argument. As discussed
before, the score $0$ may mean no effect or a changing effect.

We can strengthen the local faithfulness principle by requiring that the change in the topic argument must be approximately
equal to the contribution of an argument times its change, 
that is, we want that $\fs_{\graph\downarrow_{\is(\argx) \leftarrow \epsilon}}(a) \approx
 \fs_{\graph}(a) + \epsilon \cdot \ctrbempty_a(x)$ if $\epsilon$
 is small (locality).
 To formalize this, we consider the error term 
 $e(\epsilon) = \fs_{\graph\downarrow_{\is(\argx) \leftarrow \epsilon}}(a) - (\fs_{\graph}(a) + \epsilon \cdot \ctrbempty_a(x))$ and demand that it goes to $0$ as $\epsilon$
goes to $0$. The condition is satisfied trivially for all
continuous semantics\footnote{This is because continuity implies
that $\lim_{\epsilon \rightarrow 0} \ \fs_{\graph\downarrow_{\is(\argx) \leftarrow \epsilon}}(a) =
\fs_{\graph}(a)$ and $\lim_{\epsilon \rightarrow 0} \ \epsilon \cdot \ctrbempty_a(x) = 0$.}. Therefore, we add a condition on the 
convergence speed, namely $\lim_{\epsilon \rightarrow 0} \frac{e(\epsilon)}{\epsilon} = 0$. This means that 
the error term goes asymptotically significantly faster to 
$0$ than $\epsilon$ does.
\begin{principle}[Quantitative Local Faithfulness]
\label{principle:qlocalCorrect}
$\ctrbempty$ satisfies the \emph{quantitative local faithfulness principle} w.r.t.\ a gradual semantics $\fs$
iff we have
\begin{equation*}
\label{eq:quant_loc_faithfulness}
 \fs_{\graph\downarrow_{\is(\argx) \leftarrow \epsilon}}(a) =
 \fs_{\graph}(a) + \epsilon \cdot \ctrbempty_a(x) - e(\epsilon),
\end{equation*}
where $e(\epsilon)$ is an error term with the property
$\lim_{\epsilon \rightarrow 0} \frac{e(\epsilon)}{\epsilon} = 0$.
\end{principle}
Let us note that both strong and quantitative local faithfulness
imply local faithfulness.
\begin{proposition}
\label{prop_faithfulness_relationships}
    Given a contribution function $\ctrbempty$ and a gradual semantics $\fs$, if $\ctrbempty$ satisfies strong faithfulness w.r.t.\ $\fs$ then $\ctrbempty$ satisfies local faithfulness w.r.t.\ $\fs$ and if $\ctrbempty$ satisfies quantitative local faithfulness w.r.t.\ $\fs$ then $\ctrbempty$ satisfies local faithfulness w.r.t.\ $\fs$.
\end{proposition}
\begin{proof}
Strong faithfulness obviously implies local faithfulness 
because the first and second condition of local faithfulness are relaxations of the first and third condition of strong faithfulness.

To see that quantitative local faithfulness implies 
local faithfulness, first note that the local change is
\begin{align*}
\fs_{\graph\downarrow_{\is(\argx) \leftarrow \epsilon}}(a) -
 \fs_{\graph}(a) 
 = \epsilon \cdot \ctrbempty_a(x) - e(\epsilon).
\end{align*}
We can assume $\ctrbempty_a(x) \neq 0$ because local
faithfulness excludes this case.
Since $\lim_{\epsilon \rightarrow 0} \frac{e(\epsilon)}{\epsilon} = 0$, we can find a $\delta > 0$ such that
$\epsilon < \delta$ implies $|\frac{e(\epsilon)}{\epsilon}| < |\ctrbempty_a(x)|$.
Hence, for all such $\epsilon$, we have
\begin{align*}
\frac{\fs_{\graph\downarrow_{\is(\argx) \leftarrow \epsilon}}(a) -
 \fs_{\graph}(a)}{\epsilon} 
 = \ctrbempty_a(x) - \frac{e(\epsilon)}{\epsilon}.
\end{align*}
If $\ctrbempty_a(x) < 0$, then the difference must be negative 
(because $|\frac{e(\epsilon)}{\epsilon}| < |\ctrbempty_a(x)|$)
and if $\ctrbempty_a(x) > 0$, then the difference must be positive 
(again, because $|\frac{e(\epsilon)}{\epsilon}| < |\ctrbempty_a(x)|$). Since $\epsilon > 0$, we can conclude
that 
$\fs_{\graph\downarrow_{\is(\argx) \leftarrow \epsilon}}(a) -
 \fs_{\graph}(a) < 0$ if $\ctrbempty_a(x) < 0$, that is, $\fs_{\graph\downarrow_{\is(\argx) \leftarrow \epsilon}}(a) 
  < \fs_{\graph}(a)$ as desired. Symmetrically,
  $\ctrbempty_a(x) > 0$ implies $\fs_{\graph\downarrow_{\is(\argx) \leftarrow \epsilon}}(a) 
  > \fs_{\graph}(a)$, which completes the proof. 
\end{proof}
Strong and quantitative local faithfulness are incomparable.
The former is stronger than the latter because it is global
rather than local and the latter is stronger than the former
because it makes quantitative rather than just qualitative 
assumptions (assumptions about the magnitude of contribution
values rather than only assumptions about the sign).
One could define a strong quantitative notion but since 
even the qualitative version is too strong for our purposes
(as discussed above), we refrain from doing so.

\subsection{Counterfactuality}
The counterfactuality principle formalises the intuition that an argument should have a positive/negative contribution value 
only if its removal from $\graph$ 
leads to a decrease/increase in the topic argument's strength.
\begin{principle}[Counterfactuality]
\label{principle:counterfactual}
$\ctrbempty$ satisfies the \emph{counterfactuality principle} w.r.t.\ a gradual semantics $\fs$
iff for any $\argx \in \Args$ the following statements hold:
\begin{itemize}
    \item If $\ctrb{\argx}{\arga} < 0$, then
    $\fs_\graph(\arga) < \fs_{\graph \downarrow_{\Args \setminus \{\argx\}}}(\arga)$.
    \item If $\ctrb{\argx}{\arga} = 0$,
    then 
    $\fs_\graph(\arga) = \fs_{\graph \downarrow_{\Args \setminus \{\argx\}}}(\arga)$.
    \item If $\ctrb{\argx}{\arga} > 0$,
    then 
    $\fs_\graph(\arga) > \fs_{\graph \downarrow_{\Args \setminus \{\argx\}}}(\arga)$.
\end{itemize}
\end{principle}
We can further strengthen this principle by introducing the \emph{quantitative counterfactuality principle}, requiring that the argument's contribution to the topic argument equals the effect of the former's removal to the latter's final strength.  
\begin{principle}[Quantitative Counterfactuality]
\label{principle:q-counterfactual}
    $\ctrbempty$ satisfies the \emph{quantitative counterfactuality principle} w.r.t.\ a gradual semantics $\fs$ iff for any $\argx \in \Args$ it holds that $\ctrb{\argx}{\arga} = \fs_\graph(\arga) - \fs_{\graph \downarrow_{\Args \setminus \{\argx\}}}(\arga)$.
\end{principle}
As one may expect intuitively, quantitative counterfactuality implies counterfactuality.
\begin{proposition}
\label{prop:qcounterfactual}
    If $\ctrbempty$ satisfies quantitative counterfactuality w.r.t.\ a gradual semantics $\fs$ then $\ctrbempty$ satisfies counterfactuality w.r.t.\ $\fs$.
\end{proposition}
\begin{proof}
    Considering the definition of the counterfactuality (Principle~\ref{principle:counterfactual}), we have three cases:
    \begin{description}
        \item[$\ctrb{\argx}{\arga} < 0$.] Then, $\fs_\graph(\arga) < \fs_{\graph \downarrow_{\Args \setminus \{\argx\}}}(\arga)$ must hold.
        This is the case, given that $\ctrbempty$ satisfies quantitative counterfactuality w.r.t.\ $\fs$ because this principle stipulates that we have $\ctrb{\argx}{\arga} = \fs_\graph(\arga) - \fs_{\graph \downarrow_{\Args \setminus \{\argx\}}}(\arga) < 0$.
        \item[$\ctrb{\argx}{\arga} = 0$.] Then, $\fs_\graph(\arga) = \fs_{\graph \downarrow_{\Args \setminus \{\argx\}}}(\arga)$ must hold.
        This is the case, given that $\ctrbempty$ satisfies quantitative counterfactuality w.r.t.\ $\fs$ because this principle stipulates that we have $\ctrb{\argx}{\arga} = \fs_\graph(\arga) - \fs_{\graph \downarrow_{\Args \setminus \{\argx\}}}(\arga) = 0$.
        \item[$\ctrb{\argx}{\arga} > 0$.] Then, $\fs_\graph(\arga) > \fs_{\graph \downarrow_{\Args \setminus \{\argx\}}}(\arga)$ must hold.
        This is the case, given that $\ctrbempty$ satisfies quantitative counterfactuality w.r.t.\ $\fs$ because this principle stipulates that we have $\ctrb{\argx}{\arga} = \fs_\graph(\arga) - \fs_{\graph \downarrow_{\Args \setminus \{\argx\}}}(\arga) > 0$.
    \end{description}
    
\end{proof}

\section{Principle-based Analysis}
\label{sec:analysis}
We now provide a theoretical analysis of the principles introduced in the previous section, considering the semantics listed in Table~\ref{table:semanticsExamples}. The results are summarised in Table~\ref{table:principle-overview}.

To make our counterexamples more robust against numerical
inaccuracies, e.g., when approximating gradients or executing floating point operations, we computed every example with two different
implementations: once using \emph{QBAF-Py}, an extended version of a C/Python (C with Python bindings) library first introduced in~\cite{KAMPIK2023109066},
which is available at \url{https://github.com/TimKam/Quantitative-Bipolar-Argumentation} and
once using \emph{Uncertainpy}, a Python implementation based
on ideas from \cite{Potyka2018_tut}, which is available at \url{https://github.com/nicopotyka/Uncertainpy}.

\subsection{Contribution Existence}
We first observe that $\ctrbrempty$ satisfies contribution existence with respect to QE and EB semantics.
\begin{proposition}
\label{prop:removal-1-contribution-existence-positive}
$\ctrbrempty$ satisfies the contribution existence principle w.r.t.\ QE and EB semantics $\fs$.
\end{proposition}
\begin{proof}
Consider a modular semantics that uses the aggregation function $\alpha = \alpha^{\Sigma}_{v}$ and any of the influence functions $\iota^{e}_{w}$ or $\iota^{p}_{w}$ (Table~\ref{table:semantics}). We observe that QE and EB semantics are such semantics.
From the definition of $\alpha^{\Sigma}_{v}$ it follows that if $\fs_{\graph}({\arga}) \neq \tau_{\graph}({\arga})$ then there must exist an argument $\argx \in \Args$ s.t.\ $\sigma_\graph(\argx) \neq 0$ and $\argx$ is a direct attacker or supporter of $\arga$;
because our QBAG is acyclic, there must exist such $\argx$ that is \emph{not} also an indirect attacker or supporter of $\arga$. Hence, it holds that $\alpha_{\graph}(\arga) \neq \alpha_{\graph \downarrow \Args \setminus \{\argx\}}(\arga)$. Consequently, by definition of $\iota \in \{\iota^{e}_{w}, \iota^{p}_{w}\}$ it must hold that $\iota_{\graph}(\arga) \neq \iota_{\graph \downarrow \Args \setminus \{\argx\}}(\arga)$, from which it follows by definition of $\ctrbrempty$ that $\ctrbr{\argx}{\arga} \neq 0$, i.e., the contribution existence principle is satisfied.
\end{proof}
Similarly, $\ctrbriempty$ satisfies contribution existence with respect to QE and EB semantics.
\begin{proposition}
\label{prop:removal-2-contribution-existence-positive}
$\ctrbriempty$ satisfies the contribution existence principle w.r.t.\ QE and EB semantics $\fs$.
\end{proposition}
\begin{proof}
Consider a modular semantics that uses the aggregation function $\alpha = \alpha^{\Sigma}_{v}$ and any of the influence functions $\iota^{e}_{w}$, or $\iota^{p}_{w}$ (Table~\ref{table:semantics}). We observe that QE and EB  semantics are such semantics.
From the definition of $\alpha^{\Sigma}_{v}$ it follows that if $\fs_{\graph}({\arga}) \neq \tau_{\graph}({\arga})$, there must exist an argument $\argx$ s.t. $\is_\graph(\argx) \neq 0$ (for this observation, we also rely on the definition of $\iota \in \{\iota^{e}_{w}, \iota^{p}_{w}\}$: if $\fs_\graph(\arga) \neq 0$ then $\is_\graph(\arga) \neq 0$ must hold, because $\is_\graph(\arga) = 0$ implies that $\fs_\graph(\arga) = 0$), $\argx$ is a direct attacker or supporter of $\arga$;
because our QBAG is acyclic, there must exist such $\argx$ that is \emph{not} also an indirect attacker or supporter of $\arga$. Hence it holds that
$\alpha_{(\Args, \tau, \Att \setminus \{(\argy, \argx) | (\argy, \argx) \in \Att \}), \Supp \setminus \{(\argy, \argx) | (\argy, \argx) \in \Supp \})}(\arga) - \\ \alpha_{\graph\downarrow_{\Args \setminus \{\argx\}}}(\arga) \neq 0$.
Consequently, for $\iota \in \{\iota^{e}_{w}, \iota^{p}_{w}\}$ it holds that $\iota_{(\Args, \tau, \Att \setminus \{(\argy, \argx)  | (\argy, \argx) \in \Att \}, \Supp \setminus \{(\argy, \argx) | (\argy, \argx) \in \Supp \})}(\arga) \neq \iota_{\graph\downarrow_{\Args \setminus \{\argx\}}}(\arga)$, from which it follows by definition of $\ctrbriempty$ that $\ctrbri{\argx}{\arga} \neq 0$, i.e., the contribution existence principle is satisfied. 
\end{proof}

In contrast, both $\ctrbrempty$ and $\ctrbriempty$ can violate contribution existence---and hence also quantitative contribution existence---with respect to DFQuAD, SD-DFQuAD, and EBT semantics.
\begin{proposition}
\label{prop:removal-contribution-existence-negative}
$\ctrbrempty$ and $\ctrbriempty$ violate the contribution existence and quantitative contribution existence principles w.r.t.\ DFQuAD, SD-DFQuAD, and EBT semantics.
\end{proposition}
\begin{proof}
Consider the QBAG depicted in Figure~\ref{fig:removal-contribution-existence-negative}, which we denote by $\graph = \QBAG$.
Given DFQuAD, SD-DFQuAD, and EBT semantics $\sigma$, we observe that $\fs_{\graph}({\arga}) < 0.5 = \is(\arga)$ ($0.5$ is the initial strength of $\arga$). Also, $\fs_{\graph \downarrow_{\{\arga, \argb\}}}({\arga}) = \fs_{\graph \downarrow_{\{\arga, \argc\}}}({\arga}) = \fs_{\graph}({\arga})$. Consequently, by definition of $\ctrbrempty$ it must hold that $\ctrbr{\argb}{\arga} = \ctrbr{\argc}{\arga} = 0$, which proves the violation of contribution existence for $\ctrbrempty$.
Similarly, it holds for $\graph' = (\Args, \tau, \Att \setminus \{(\argy, \argb) | (\argy, \argb) \in \Att \}), \Supp \setminus \{(\argy, \argb) | (\argy, \argb) \in \Supp \})$ and $\graph'' = (\Args, \tau, \Att \setminus \{(\argy, \argc) | (\argy, \argc) \in \Att \}), \Supp \setminus \{(\argy, \argc) | (\argy, \argc) \in \Supp \})$ that $\fs_{\graph'}({\arga}) = \fs_{\graph''}({\arga}) = \fs_{\graph}({\arga})$.
Hence, by definition of $\ctrbriempty$ it must hold that $\ctrbri{\argb}{\arga} = \ctrbri{\argc}{\arga} = 0$, which proves the violation of contribution existence for $\ctrbriempty$.
\end{proof}
The counterexample in Figure~\ref{fig:removal-contribution-existence-negative} applies to other semantics that use
 $\alpha^{max}_{v}$ or $\alpha^{\Pi}_{v}$ as the aggregation and
 $\iota^{l}_{w}$,  $\iota^{e}_{w}$, or $\iota^{p}_{w}$
 as the influence function because a single attacker (supporter)
 with strength $1$ under these semantics has the maximum effect,
 so that removing another attacker (supporter) cannot
 affect the topic argument.
\begin{figure}[ht]
\centering
\begin{tikzpicture}[scale=0.8]
     \node[unanode]    (a)    at(2,0)  {\argnode{\arga}{0.5}{<0.5}};
     \node[unanode]  (b)    at(0,1)  {\argnode{\argb}{1}{1}};
     \node[unanode]    (c)    at(4,1)  {\argnode{\argc}{1}{1}};
     \path [->, line width=0.2mm]  (b) edge node[left] {-} (a);
     \path [->, line width=0.2mm]  (c) edge node[left] {-} (a);
\end{tikzpicture}
\caption{$\ctrbrempty$ and $\ctrbriempty$ violate the contribution existence and quantitative contribution existence principles w.r.t. DFQuAD, SD-DFQuAD, and EBT semantics.}
\label{fig:removal-contribution-existence-negative}
\end{figure}

The Shapley value-based contribution function satisfies even the stronger quantitative contribution existence principle w.r.t.\ all of the surveyed argumentation semantics. Intuitively, it does not suffer from
the previously observed problem because it does not only
remove the argument in the existing situation, but looks
at the effect of the argument in all possible situations.
\begin{proposition}
\label{prop:shapley-qcontribution-existence}  
$\ctrbsempty$ satisfies the quantitative contribution existence principle w.r.t.\ QE, DFQuAD, SD-DFQuAD, EB, and EBT semantics $\fs$.
\end{proposition}
\begin{proof}
    We provide the proof by characterising $\ctrbsempty$ relative to the contributor's Shapley value of a coalition game where the set of players is $\Args \setminus \{ \arga \}$ (where $\arga$ is our topic argument) and the utility function $v: 2^{\Args \setminus \{ \arga \}} \rightarrow \mathbb{R}$ is defined as follows (given $S \subseteq \Args \setminus \{ \arga \}$):
    $$v(S) :=  \sigma(\arga) - \sigma_{G \downarrow_{\Args \setminus S}}(\arga)$$
    Note that by subtracting $\sigma_{G \downarrow_{\Args \setminus S}}(\arga)$ from $\sigma(\arga)$, we ensure that $v(\emptyset) = 0$. 
    Then, we obtain the following Shapley value (given player/contributor $\argx \in \Args \setminus \{\arga\}$):
    {\small
    \begin{align*}
        &{}\phi_{\argx}(v) =\\
        &{} \sum_{X \subseteq \Args \setminus \{ \argx, \arga \}} \frac{|X|! \cdot (|\Args \setminus \{\arga\}|-|X|-1)!}{|\Args \setminus \{\arga\}|!}\left( (v(X \cup \{\argx\}) - v(X) ) \right) = \\
        &{} \sum_{X \subseteq \Args \setminus \{ \argx, \arga \}} \frac{|X|! \cdot (|\Args \setminus \{\arga\}|-|X|-1)!}{|\Args \setminus \{\arga\}|!}\left(\sigma(\arga) - \operatorname{\fs}_{\graph\downarrow_{\Args \setminus (X \cup \{ \argx \})}}(\arga) - (\sigma(\arga) - \operatorname{\fs}_{\graph\downarrow_{\Args \setminus X}}(\arga))\right) = \\
        &{} \sum_{X \subseteq \Args \setminus \{ \argx, \arga \}} \frac{|X|! \cdot (|\Args \setminus \{\arga\}|-|X|-1)!}{|\Args \setminus \{\arga\}|!}\left(\operatorname{\fs}_{\graph\downarrow_{\Args \setminus X}}(\arga) -  \operatorname{\fs}_{\graph\downarrow_{\Args \setminus (X \cup \{ \argx \})}}(\arga)\right) = \\
       &{} \ctrbs{\argx}{\arga}.
    \end{align*}}
    Now, from the efficiency principle of Shapley values, it follows that $\sum_{\argx \in \Args \setminus \{\arga\}} \phi_{\argx}(v) = v(\Args \setminus \{\arga\})$~\cite{shapley1951notes}.
    Hence, because $ v(\Args \setminus \{\arga\}) = \sigma(\arga) - \tau(\arga)$ and $\sum_{\argx \in \Args \setminus \{\arga\}} \phi_{\argx}(v) = \sum_{\argx \in \Args \setminus \{\arga\}} \ctrbs{\argx}{\arga}$, it must hold that $  \sum_{\argx \in \Args \setminus \{\arga\}} \ctrbs{\argx}{\arga} = \sigma(\arga) - \tau(\arga)$ and consequently, quantitative contribution existence must be satisfied.
\end{proof}
%
%
\begin{corollary}
\label{prop:shapley-contribution-existence}  
$\ctrbsempty$ satisfies the contribution existence principle w.r.t.\ QE, DFQuAD, SD-DFQuAD, EB, and EBT semantics $\fs$.
\end{corollary}
\begin{proof}
Because $\ctrbsempty$ satisfies quantitative contribution existence w.r.t.\ QE, DFQuAD, SD-DFQuAD, EB, and EBT semantics (Proposition~\ref{prop:shapley-qcontribution-existence}), the proof follows from Proposition~\ref{contribution-existence:implication} (if a semantics satisfies quantitative contribution existence then it satisfies contribution existence).
\end{proof}
The gradient-based contribution function $\ctrbgempty$ violates contribution existence (and hence quantitative contribution existence) with respect to some of the surveyed argumentation semantics, i.e., with respect to DFQuAD, SD-DFQuAD, and EBT semantics.
\begin{proposition}
\label{prop:contribution-existence-violated-g}  
$\ctrbgempty$ violates the contribution existence and quantitative contribution existence principles w.r.t.\  DFQuAD, SD-DFQuAD, and EBT semantics $\fs$.
\end{proposition}
\begin{proof}
    Consider again the QBAG depicted in Figure~\ref{fig:removal-contribution-existence-negative}, which we denote by $\graph = \QBAG$.
    Our topic argument is $\arga$.
    Given DFQuAD, SD-DFQuAD, and EBT semantics $\sigma$, we observe that $\sigma(\arga) - \tau(\arga) \neq 0$ but $\ctrbg{\argb}{\arga} = \ctrbg{\argc}{\arga} = 0$ and hence contribution existence and quantitative contribution existence must be violated.
\end{proof}
In contrast, we can show that $\ctrbgempty$ satisfies (non-quantitative) contribution existence with respect to QE and EB semantics.
The proof is roughly based on the same intuition 
as the one for $\ctrbrempty$ and $\ctrbriempty$ with respect to QE and EB semantics: when our topic argument's final strength does not equal its initial strength, then there must exist a direct supporter or attacker of this argument whose marginal change in initial strength has an effect on the final strength of the topic argument.
\begin{proposition}
    $\ctrbgempty$ satisfies contribution existence w.r.t.\ QE and EB semantics.
\end{proposition}
\begin{proof}
Consider a modular semantics that uses the aggregation function $\alpha = \alpha^{\Sigma}_{v}$ and any of the influence functions $\iota^{e}_{w}$ or $\iota^{p}_{w}$ (Table~\ref{table:semantics}). We observe that QE and EB semantics are such semantics.
QE and EB satisfy the stability property \cite{potyka2018continuous,amgoud2018evaluation} (Principle~\ref{sprinciple:stability}), which 
states that the final strength of an argument without parents
is its base score. Since the final strength of the topic
argument $\arga$ is unequal to its base score by assumption, 
it must have predecessors. 
Since $\graph$ is acyclic, we can assume that there exists
a topological ordering of the arguments. Consider an arbitrary
topological ordering and let $\argx$ be the predecessor of
$\arga$ with highest index in the ordering. Then $\argx$
must be a parent of $\arga$ and there can be no other paths
from $\arga$ to $\argx$ (for otherwise, $\argx$ could not 
have the highest index among $\arga$'s predecessors).
By definition of the aggregation and influence functions,
if $(\argx, \arga) \in \Att$, then (marginally) increasing 
$\is(\argx)$ will decrease $\fs_{\graph}(\arga)$
and (marginally) decreasing 
$\is(\argx)$ will increase $\fs_{\graph}(\arga)$.
Hence, the partial derivative w.r.t.\ $\argx$ must be
strictly negative (non-zero).
If $(\arga, \argx) \in \Supp$, it follows symmetrically
that the partial derivative w.r.t.\ $\argx$ must be
strictly positive (non-zero).
\end{proof}
However, only $\ctrbsempty$ satisfies \emph{quantitative} contribution existence; the other contribution functions
cannot guarantee that the contributions of a topic argument's contributor ``add up'' to difference between the topic argument's final and initial strength.  We provide counterexamples for $\ctrbrempty$, $\ctrbriempty$, and $\ctrbgempty$ and QE and EB semantics in the proof of the 
following proposition.
\begin{proposition}
\label{prop:qcontribution-existence-violated}  
$\ctrbrempty$, $\ctrbriempty$, and $\ctrbgempty$ violate the quantitative contribution existence principle w.r.t.\ QE and EB semantics $\fs$.
\end{proposition}
\begin{proof}
Consider yet again the QBAG depicted in Figure~\ref{fig:removal-contribution-existence-negative}, which we denote by $\graph = \QBAG$.
Our topic argument is $\arga$.
For QE semantics, we have $\sigma(\arga) - \tau(\arga) = -0.4$, but $\ctrbr{\argb}{\arga} + \ctrbr{\argc}{\arga} = \ctrbri{\argb}{\arga} + \ctrbri{\argc}{\arga} = -0.3$ and $\ctrbg{\argb}{\arga} + \ctrbg{\argc}{\arga} \approx -0.16$.
For EB semantics, we have $\sigma(\arga) - \tau(\arga) \approx -0.2025$, but $\ctrbr{\argb}{\arga} + \ctrbr{\argc}{\arga} = \ctrbri{\argb}{\arga} + \ctrbri{\argc}{\arga} \approx -0.138$ and $\ctrbg{\argb}{\arga} + \ctrbg{\argc}{\arga} \approx -0.089$.
In these cases, we can easily see that $\sum_{\argx \in \Args \setminus \{\arga\}} \ctrb{\argx}{\arga} = \sigma(\arga) - \tau(\arga)$ does not hold and hence, quantitative contribution existence must be violated.
\end{proof}

Finally, let us introduce a simple example that highlights the advantage of the Shapley value-based contribution functions, following the intuition of quantitative contribution existence.
\begin{example}
    Consider a set of simple QBAGs, where each QBAG consists of a topic argument $\arga$ and a number of supporters $\argb_1, \dots, \argb_n$, where $n = 1$ for the smallest QBAG and $n = 20$ for the largest (Figure~\ref{fig:graph-contribution}). The arguments are not related otherwise.
    The initial strength of $\arga$ is $0.5$, whereas each argument of $\argb_1, \dots, \argb_n$ has an initial strength of $1$.
    We apply QE semantics, which satisfies contribution existence with respect to $\ctrbrempty$, $\ctrbriempty$, and $\ctrbgempty$.
    Obviously, in a given QBAG, the contribution of every contributor $\argb_i$, $1 \leq i \leq n$ is the same, given a specific contribution function.
    We plot this contribution for $\ctrbrempty$ (which equals the one of $\ctrbriempty$ in this case), $\ctrbsempty$, and $\ctrbgempty$; i.e., we plot $\ctrbr{\arga}{\argb_i}$, $\ctrbs{\arga}{\argb_i}$, and $\ctrbg{\arga}{\argb_i}$ for our set of QBAGs.
    What we can see is that the contributions of $\ctrbrempty$ and $\ctrbgempty$ converge to zero faster with increasing $n$ than the contributions of $\ctrbsempty$.
    For $n = 10$, the Shapley value-based contribution is still at around $0.05$, whereas the removal- and gradient-based contributions are, visually, already indistinguishable from $0$.
    The example highlights that contribution existence alone may not always be a strong enough principle: for larger $n$ we see in the plot (approx. for $6 \leq n \leq 20$), $\ctrbr{\arga}{\argb_i}$ and $\ctrbg{\arga}{\argb_i}$ are negligibly small, although the proportional effect of a single supporter is still substantial, which can be considered misleading.
\begin{figure}[ht]
\centering
\subfloat[Argument $\arga$ with $n$ supporters.]{
\label{fig:graph-contribution}
\begin{tikzpicture}[scale=0.8]
     \node[unanode]    (a)    at(4,0)  {\argnode{\arga}{0.5}{}};
     \node[unanode]    (b1)    at(2,4)  {\argnode{\argbb} {1}{1}};
    \node[invnode]    (bx)    at(4,4)  {\argdots};
     \node[unanode]    (bn)    at(6,4)  {\argnode{\argbn} {1}{1}};
     \path [->, line width=0.2mm]  (b1) edge node[left] {+} (a);
     \path [->, line width=0.2mm]  (bx) edge node[left] {+} (a);
     \path [->, line width=0.2mm]  (bn) edge node[left] {+} (a);
\end{tikzpicture}
}
\subfloat[Contribution of $\argb_i$, $1 \leq i \leq n$ to $\arga$, given $n$.]{
\label{fig:plot-contribution}
\includegraphics[width=0.6\columnwidth]{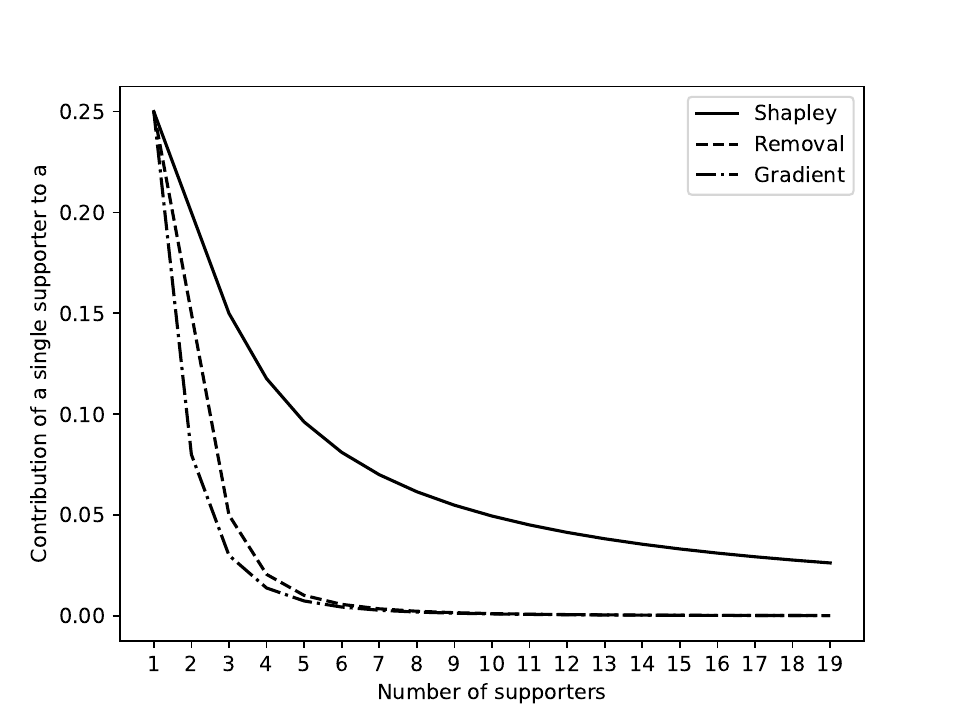}}
\caption{The contributions of $\ctrbrempty$ and $\ctrbgempty$ converge to zero faster than the contributions of $\ctrbsempty$. Here, the contribution of $\argb_i$, $1 \leq i \leq n$ to $\arga$ depends on the number of $\arga$'s supporters and we apply QE semantics.}
\label{fig:contribution-intuition}
\end{figure}
\end{example}

\subsection{Directionality}
\label{subsec:analysis-directionality}

Since modular semantics generally satisfy the 
 \emph{directionality} principle for gradual semantics~\cite{amgoud2018evaluation},
 we should expect that a faithful contribution function 
 satisfies the corresponding principle for contribution functions. As we show next, this is indeed the case
 for all contribution functions considered here. 
\begin{proposition}
\label{proposition_directionality}
$\ctrbrempty$, $\ctrbriempty$, $\ctrbsempty$, and $\ctrbgempty$ satisfy directionality w.r.t. all modular argumentation semantics $\fs$.
\end{proposition}
\begin{proof}
    Let $\graph = \QBAG$ be a QBAG.
\begin{description}
    \item[The proof for $\ctrbrempty$.] If there is no directed path from $\argx$ to $\arga$, then $\fs_{\graph}(\arga) = \fs_{\graph \downarrow_{\Args \setminus \{\argx\}}}(\arga)$ by definition of a modular semantics. Consequently, it holds that $\ctrbr{\argx}{\arga} = 0$, as required.
    \item[The proof for $\ctrbriempty$.] If there is no directed path from $\argx$ to $\arga$, then for $\graph' = (\Args, \tau, \Att \setminus \{(\argy, \argx) | (\argy, \argx) \in \Att \}), \Supp \setminus \{(\argy, \argx) | (\argy, \argx) \in \Supp \})$ it holds that $\fs_{\graph'}(\arga) = \fs_{\graph \downarrow_{\Args \setminus \{\argx\}}}(\arga)$ by definition of a modular semantics. Consequently, $\ctrbri{\argx}{\arga} = 0$, as required.
    \item[The proof for $\ctrbsempty$.] If there is no directed path from $\argx$ to $\arga$, then for $X \subseteq \Args$ it holds that $\fs_{\graph\downarrow_{\Args \setminus (X \cup \{ \argx \})}}(\arga) = \fs_{\graph\downarrow_{\Args \setminus X}}(\arga)$ by definition of a modular semantics. Consequently, $\ctrbs{\argx}{\arga} = 0$, as required.
    \item[The proof for $\ctrbgempty$.] If there is no directed path from $\argx$ to $\arga$, then $\fs(\arga) = f_{\arga}(\ldots, \is(\arga), \is(\argx), \ldots)$ is independent of $\is(\argx)$, by definition of a modular semantics. In that case, trivially, the partial derivative of $f_{\arga}$ w.r.t.\ $\is(\argx)$ is null, so that $\ctrbg{\argx}{\arga} = 0$, as required. 
\end{description}

\end{proof}

\subsection{(Quantitative) Local Faithfulness}

Intuitively, (quantitative) local faithfulness is best aligned with the gradient-based contribution function, which indeed satisfies the principle with respect to every differentiable modular argumentation semantics. Of course, we require differentiability of the semantics for the
gradient to be defined. For acyclic graphs, 
the condition is satisfied 
for all commonly considered semantics. In general,
a modular semantics is differentiable whenever both 
the aggregation and influence functions are differentiable.
Note that all aggregation functions and the linear and Euler-based influence functions in Table \ref{table:semantics} are differentiable. The $p$-Max aggregation function is differentiable for $p=2$ \cite[Proposition 1]{potyka2018continuous} but
can be pointwise non-differentiable for some other choices of $p$.
As we explain next, $\ctrbgempty$ satisfies quantitative 
local faithfulness for every differentiable modular argumentation semantics and, thus, also the weaker principle of 
local faithfulness.
\begin{proposition}
\label{proposition_quan_local_faith}
$\ctrbgempty$ satisfies quantitative local faithfulness w.r.t. every differentiable modular argumentation semantics.
\end{proposition}
\begin{proof}
Note that
\eqref{eq:quant_loc_faithfulness}
is trivially satisfied when letting
\begin{equation}
 e(\epsilon) =  \fs_{\graph\downarrow_{\is(\argx) \leftarrow \epsilon}}(a) - (\fs_{\graph}(a) + \epsilon \cdot \ctrbg{\argx}{\arga}. 
\end{equation}
It remains to show that 
$\lim_{\epsilon \rightarrow 0} \frac{e(\epsilon)}{\epsilon} = 0$. We have
\begin{align*}
    \lim_{\epsilon \rightarrow 0} \frac{e(\epsilon)}{\epsilon}
&= \lim_{\epsilon \rightarrow 0} \frac{\fs_{\graph\downarrow_{\is(\argx) \leftarrow \epsilon}}(a) - (\fs_{\graph}(a) + \epsilon \cdot \ctrbg{\argx}{\arga}}{\epsilon} \\
&= \lim_{\epsilon \rightarrow 0} \frac{\fs_{\graph\downarrow_{\is(\argx) \leftarrow \epsilon}}(a) - \fs_{\graph}(a)}{\epsilon} - \ctrbg{\argx}{\arga} \\
&=\ctrbg{\argx}{\arga} - \ctrbg{\argx}{\arga} = 0.
\end{align*}
For the third equality, note that the first term
in the second equality is just the partial derivative
of the strength of $\arga$
with respect to the initial strength of $\argx$,
which is $\ctrbg{\argx}{\arga}$ by definition.
\end{proof}
\begin{proposition}
\label{proposition_local_faith}
$\ctrbgempty$ satisfies local faithfulness w.r.t. every differentiable modular argumentation semantics.
\end{proposition}
\begin{proof}
The proof follows immediately from Proposition \ref{prop_faithfulness_relationships} (quantitative local faithfulness implies local faithfulness).    
\end{proof}

As we show next by a sequence of counterexamples, the other introduced contribution functions
do not satisfy local faithfulness and, thus, by Proposition~\ref{prop_faithfulness_relationships}, they also do not 
satisfy quantitative local faithfulness\footnote{According to Proposition~\ref{prop_faithfulness_relationships}, quantitative
local faithfulness implies local faithfulness. Hence, by contraposition, not local faithfulness implies not quantitative local faithfulness.}.
For each counter-example, we provide a plot of the topic argument's final strength dependent on the contributor's initial strength in order to facilitate a more intuitive grasp of the underlying issues (saddle points or plateaus). Since $\ctrbgempty$ satisfies local faithfulness, we can demonstrate
violation of the principle by showing that the sign under a
given contribution function differs from the sign under $\ctrbgempty$.
\begin{proposition}
$\ctrbrempty$ and $\ctrbriempty$ violate local faithfulness w.r.t.\ QE semantics.
\end{proposition}
\begin{proof}
Consider the QBAG $\graph = \QBAG$ depicted in Figure~\ref{fig:graph-faith-qe}, with the topic argument $\arga$ and the contributor $\argd$. Given QE semantics $\sigma$, we have $\ctrbr{\argd}{\arga} = \ctrbri{\argd}{\arga}  \approx -0.01122$ and $\ctrbg{\argd}{\arga} \approx 0.02987$; local faithfulness must be violated because $\ctrbr{\argd}{\arga} = \ctrbri{\argd}{\arga} < 0$ but $\ctrbg{\argd}{\arga} > 0$.
\end{proof}

\begin{figure}[ht]
\centering
\subfloat[$\graph$.]{
\label{fig:graph-faith-qe}
\begin{tikzpicture}[scale=0.8]
    \node[unanode]    (a)    at(2,0)  {\argnode{\arga}{1}{0.9812}};
    \node[unanode]    (b)    at(0,2)  {\argnode{\argb}{0.7}{0.3801}};
    \node[unanode]    (c)    at(4,2)  {\argnode{\argc}{0.6}{0.5172}};
    \node[unanode]    (d)    at(2,4)  {\argnode{\argd}{0.4}{0.4}};
    
    \path [->, line width=0.2mm]  (b) edge node[left] {+} (a);
    \path [->, line width=0.2mm]  (c) edge node[left] {-} (a);
    \path [->, line width=0.2mm]  (c) edge node[left, above] {-} (b);
    \path [->, line width=0.2mm]  (d) edge node[left] {-} (b);
    \path [->, line width=0.2mm]  (d) edge node[left] {-} (c);
\end{tikzpicture}
}
\subfloat[Final strength of $\arga$, given initial strength of $\argd$.]{
\label{fig:plot-faith-qe}
\includegraphics[width=0.6\columnwidth]{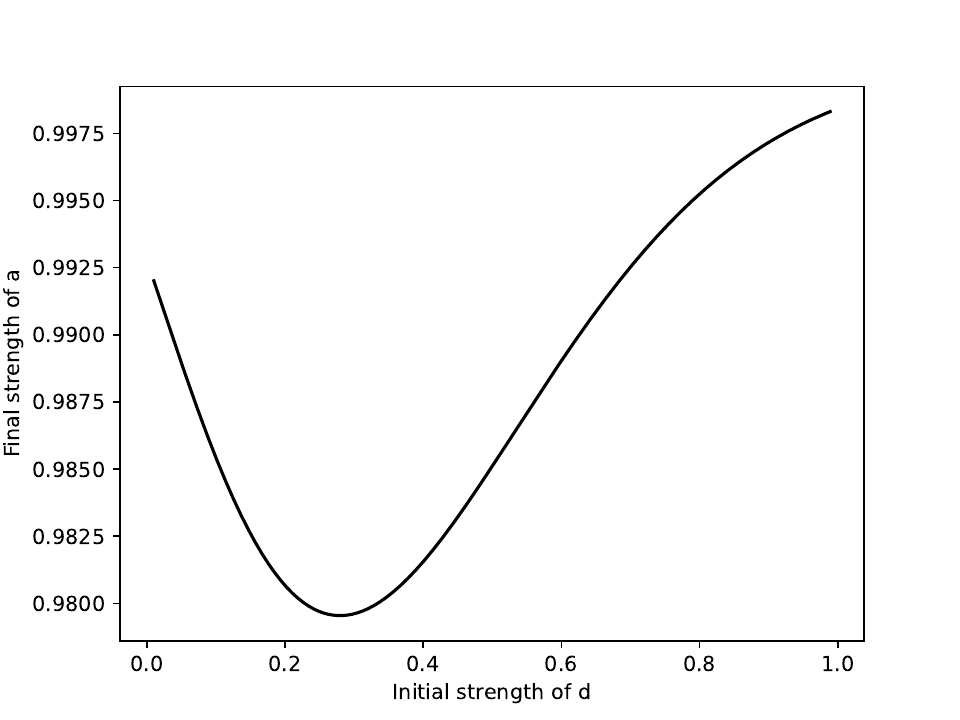}}
\caption{$\ctrbrempty$ and $\ctrbriempty$ violate local faithfulness w.r.t.\ QE semantics.}
\label{fig:faith-qe}
\end{figure}

\begin{proposition}
$\ctrbrempty$ and $\ctrbriempty$ violate local faithfulness w.r.t.\  DFQuAD semantics.
\end{proposition}
\begin{proof}
    Consider the QBAG $\graph = \QBAG$ depicted in Figure~\ref{fig:graph-faith-df}, with the topic argument $\arga$ and the contributor $\argd$. Given DFQuAD semantics $\sigma$, we have $\ctrbr{\argd}{\arga} = \ctrbri{\argd}{\arga}  \approx 0.32 >0$, but as can be seen in
    Figure~\ref{fig:plot-faith-df}, increasing
    $\argd$'s initial strength does not increase the topic
    argument's strength.
\end{proof}

\begin{figure}[ht]
\centering
\subfloat[$\graph$.]{
\label{fig:graph-faith-df}
\begin{tikzpicture}[scale=0.8]
    \node[unanode]    (a)    at(2,0)  {\argnode{\arga}{1}{1}};
    \node[unanode]    (b)    at(0,2)  {\argnode{\argb}{0.7}{0.1232}};
    \node[unanode]    (c)    at(4,2)  {\argnode{\argc}{0.6}{0.12}};
    \node[unanode]    (d)    at(2,4)  {\argnode{\argd}{0.8}{0.8}};
    
    \path [->, line width=0.2mm]  (b) edge node[left] {+} (a);
    \path [->, line width=0.2mm]  (c) edge node[left] {-} (a);
    \path [->, line width=0.2mm]  (c) edge node[left, above] {-} (b);
    \path [->, line width=0.2mm]  (d) edge node[left] {-} (b);
    \path [->, line width=0.2mm]  (d) edge node[left] {-} (c);
\end{tikzpicture}
}
\subfloat[Final strength of $\arga$, given initial strength of $\argd$.]{
\label{fig:plot-faith-df}
\includegraphics[width=0.6\columnwidth]{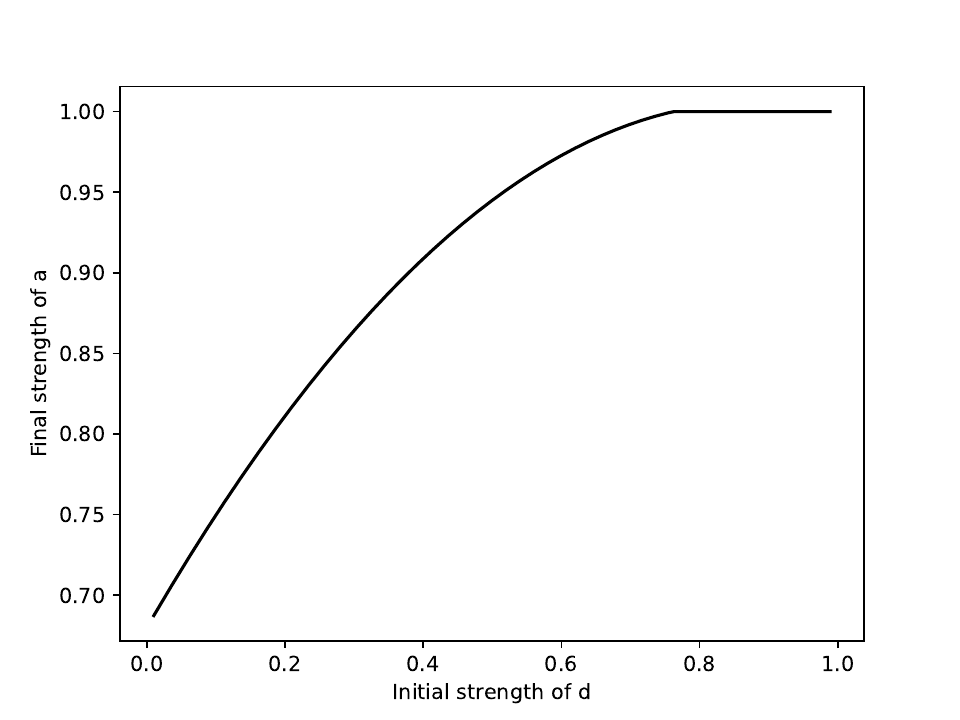}}
\caption{$\ctrbrempty$ and $\ctrbriempty$ violate local faithfulness w.r.t.\ DFQuAD semantics.}
\label{fig:faith-df}
\end{figure}

\begin{proposition}
$\ctrbrempty$ and $\ctrbriempty$ violate local faithfulness w.r.t.\ SD-DFQuAD semantics.
\end{proposition}
\begin{proof}
    Consider the QBAG $\graph = \QBAG$ depicted in Figure~\ref{fig:graph-faith-sd}, with the topic argument $\arga$ and the contributor $\argd$. Given SD-DFQuAD semantics $\sigma$, we have $\ctrbr{\argd}{\arga} = \ctrbri{\argd}{\arga}  \approx 0.1398 > 0$. However,
    as can be seen in Figure~\ref{fig:plot-faith-sd}, marginally increasing $\argd$'s initial strength ($0.6$)
    does not increase the topic argument's strength.
\end{proof}

\begin{figure}[ht]
\centering
\subfloat[$\graph$.]{
\label{fig:graph-faith-sd}
\begin{tikzpicture}[scale=0.8]
    \node[unanode]    (a)    at(2,0)  {\argnode{\arga}{1}{1}};
    \node[unanode]    (b)    at(0,2)  {\argnode{\argb}{0.7}{0.4}};
    \node[unanode]    (c)    at(4,2)  {\argnode{\argc}{0.6}{0.375}};
    \node[unanode]    (d)    at(2,4)  {\argnode{\argd}{0.6}{0.6}};
    
    \path [->, line width=0.2mm]  (b) edge node[left] {+} (a);
    \path [->, line width=0.2mm]  (c) edge node[left] {-} (a);
    \path [->, line width=0.2mm]  (c) edge node[left, above] {-} (b);
    \path [->, line width=0.2mm]  (d) edge node[left] {-} (b);
    \path [->, line width=0.2mm]  (d) edge node[left] {-} (c);
\end{tikzpicture}
}
\subfloat[Final strength of $\arga$, given initial strength of $\argd$.]{
\label{fig:plot-faith-sd}
\includegraphics[width=0.6\columnwidth]{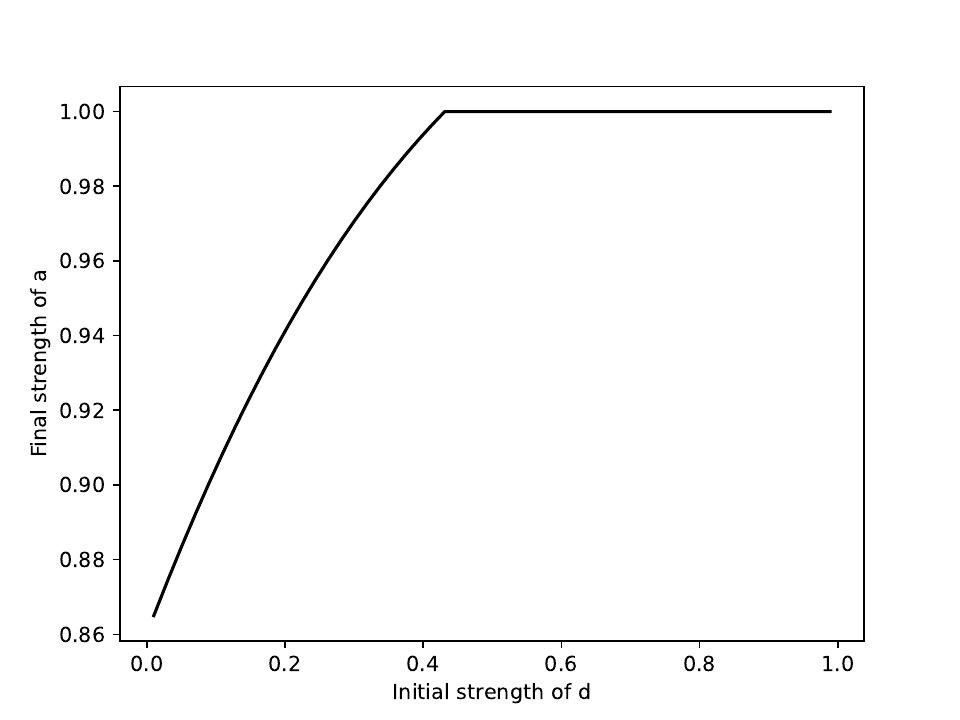}}
\caption{$\ctrbrempty$ and $\ctrbriempty$ violate local faithfulness w.r.t.\ SD-DFQuAD semantics.}
\label{fig:faith-sd}
\end{figure}

\begin{proposition}
$\ctrbrempty$ and $\ctrbriempty$ violate local faithfulness w.r.t.\  EB semantics.
\end{proposition}
\begin{proof}
Consider the QBAG $\graph = \QBAG$ depicted in Figure~\ref{fig:graph-faith-eb}, with the topic argument $\arga$ and the contributor $\argd$. Given EB semantics $\sigma$, we have $\ctrbr{\argd}{\arga} = \ctrbri{\argd}{\arga}  \approx 0.0016$ and $\ctrbg{\argd}{\arga} \approx -0.002$; local faithfulness must be violated because $\ctrbr{\argd}{\arga} = \ctrbri{\argd}{\arga} > 0$ but $\ctrbg{\argd}{\arga} < 0$.
\end{proof}

\begin{figure}[ht]
\centering
\subfloat[$\graph$.]{
\label{fig:graph-faith-eb}
\begin{tikzpicture}[scale=0.7]
    \node[unanode]    (a)    at(2,0)  {\argnode{\arga}{0.3}{0.4379}};
    \node[unanode]    (b)    at(0,2)  {\argnode{\argb}{0.8}{0.8721}};
    \node[unanode]    (c)    at(2,3)  {\argnode{\argc}{0.1}{0.1692}};
    \node[unanode]    (d)    at(2,5)  {\argnode{\argd}{0.65}{0.65}};
    \node[unanode]    (e)    at(4,2)  {\argnode{\arge}{0.1}{0.1478}};
    
    \path [->, line width=0.2mm]  (b) edge node[left] {+} (a);
    \path [->, line width=0.2mm]  (c) edge node[left, above] {+} (b);
    \path [->, line width=0.2mm]  (c) edge node[left, above] {-} (e);
    \path [->, line width=0.2mm]  (d) edge node[left] {+} (b);
    \path [->, line width=0.2mm]  (d) edge node[left] {+} (c);
    \path [->, line width=0.2mm]  (d) edge node[left, above] {+} (e);
    \path [->, line width=0.2mm]  (e) edge node[left] {-} (a);
\end{tikzpicture}
}
\subfloat[Final strength of $\arga$, given initial strength of $\argd$.]{
\label{fig:plot-faith-eb}
\includegraphics[width=0.6\columnwidth]{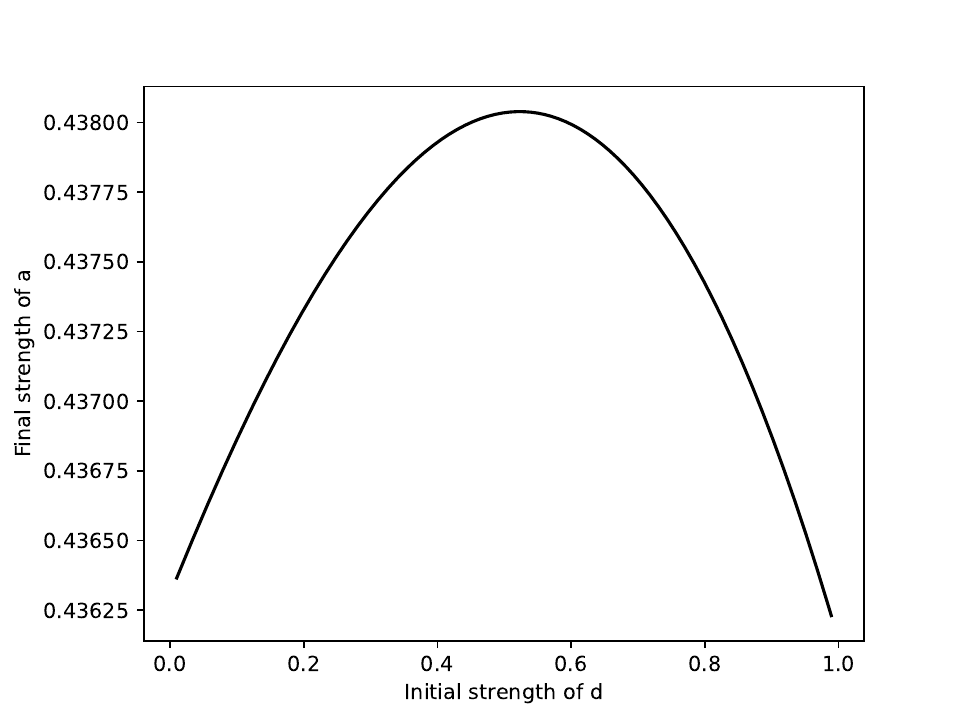}}
\caption{$\ctrbrempty$ and $\ctrbriempty$ violate local faithfulness w.r.t.\ EB semantics.}
\label{fig:faith-eb}
\end{figure}

\begin{proposition}
$\ctrbrempty$ and $\ctrbriempty$ violate local faithfulness w.r.t.\  EBT semantics.
\end{proposition}
\begin{proof}
    Consider the QBAG $\graph = \QBAG$ depicted in Figure~\ref{fig:graph-faith-ebt}, with the topic argument $\arga$ and the contributor $\argd$. Given EBT semantics $\sigma$, we have $\ctrbr{\argd}{\arga} = \ctrbri{\argd}{\arga}  \approx 0.013 > 0$, but as can be seen in
    Figure~\ref{fig:plot-faith-ebt}, increasing
    $\argd$'s initial strength does not increase the topic
    argument's strength. 
\end{proof}

\begin{figure}[ht]
\centering
\subfloat[$\graph$.]{
\label{fig:graph-faith-ebt}
\begin{tikzpicture}[scale=0.8]
    \node[unanode]    (a)    at(2,0)  {\argnode{\arga}{0.5}{0.4245}};
    \node[unanode]    (b)    at(0,2)  {\argnode{\argb}{0.5}{0.5}};
    \node[unanode]    (c)    at(4,2)  {\argnode{\argc}{0.6}{0.4959}};
    \node[unanode]    (d)    at(4,4)  {\argnode{\argd}{0.8}{0.8}};
    
    \path [->, line width=0.2mm]  (b) edge node[left] {-} (a);
    \path [->, line width=0.2mm]  (c) edge node[left] {-} (a);
    \path [->, line width=0.2mm]  (d) edge node[left] {-} (c);
\end{tikzpicture}
}
\subfloat[Final strength of $\arga$, given initial strength of $\argd$.]{
\label{fig:plot-faith-ebt}
\includegraphics[width=0.6\columnwidth]{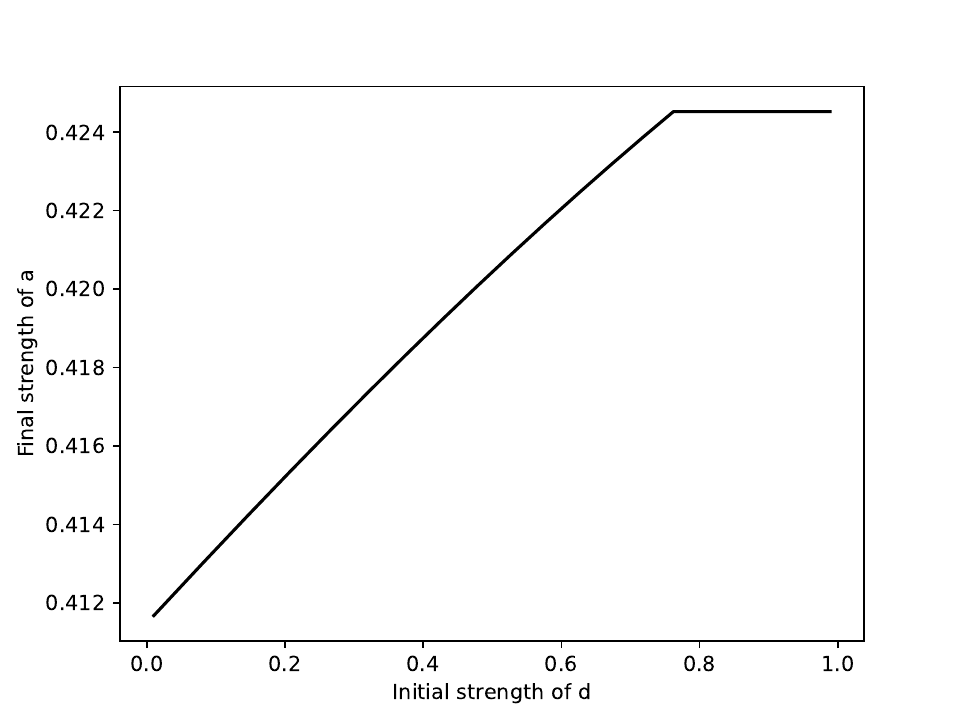}}
\caption{$\ctrbrempty$ and $\ctrbriempty$ violate local faithfulness w.r.t.\ EBT semantics.}
\label{fig:faith-ebt}
\end{figure}

\begin{proposition}
$\ctrbsempty$ violates local faithfulness w.r.t.\ QE semantics.
\end{proposition}
\begin{proof}
    Consider the QBAG $\graph = \QBAG$ depicted in Figure~\ref{fig:graph-faith-sqe}, with the topic argument $\arga$ and the contributor $\argd$. Given QE semantics $\sigma$, we have $\ctrbs{\argd}{\arga} = \approx -0.0016$ and $\ctrbg{\argd}{\arga} \approx 0.0302$; local faithfulness must be violated because $\ctrbs{\argd}{\arga} < 0$ but $\ctrbg{\argd}{\arga} > 0$. 
\end{proof}

\begin{figure}[ht]
\centering
\subfloat[$\graph$.]{
\label{fig:graph-faith-sqe}
\begin{tikzpicture}[scale=0.8]
    \node[unanode]    (a)    at(2,0)  {\argnode{\arga}{1}{0.9289}};
    \node[unanode]    (b)    at(0,2)  {\argnode{\argb}{0.65}{0.3602}};
    \node[unanode]    (c)    at(4,2)  {\argnode{\argc}{0.68}{0.6394}};
    \node[unanode]    (d)    at(2,4)  {\argnode{\argd}{0.26}{0.26}};
    
    \path [->, line width=0.2mm]  (b) edge node[left] {+} (a);
    \path [->, line width=0.2mm]  (c) edge node[left] {-} (a);
    \path [->, line width=0.2mm]  (c) edge node[left, above] {-} (b);
    \path [->, line width=0.2mm]  (d) edge node[left] {-} (b);
    \path [->, line width=0.2mm]  (d) edge node[left] {-} (c);
\end{tikzpicture}
}
\subfloat[Final strength of $\arga$, given initial strength of $\argd$.]{
\label{fig:plot-faith-sqe}
\includegraphics[width=0.6\columnwidth]{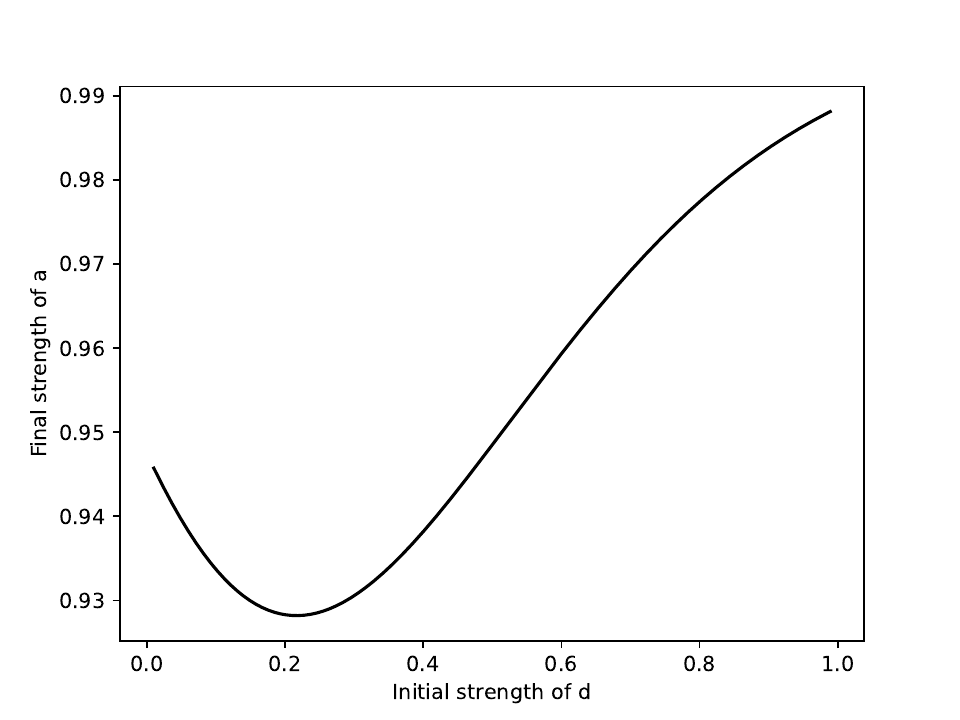}}
\caption{$\ctrbsempty$ violates local faithfulness w.r.t.\ QE semantics.}
\label{fig:faith-sqe}
\end{figure}

\begin{proposition}
$\ctrbsempty$ violates local faithfulness w.r.t.\ DFQuAD semantics.
\end{proposition}
\begin{proof}
    Consider the QBAG $\graph = \QBAG$ depicted in Figure~\ref{fig:graph-intro}, with the topic argument $\arga$ and the contributor $\arge$. Given DFQuAD semantics $\sigma$, we have $\ctrbs{\arge}{\arga} \approx -0.0833 <0$,
    but as can be seen in
    Figure~\ref{fig:plot-intro}, increasing
    $\arge$'s initial strength does actually increase the topic
    argument's strength. 
\end{proof}

\begin{proposition}
$\ctrbsempty$ violates local faithfulness w.r.t.\ SD-DFQuAD semantics.
\end{proposition}
\begin{proof}
    Consider the QBAG $\graph = \QBAG$ depicted in Figure~\ref{fig:graph-faith-sd}, with the topic argument $\arga$ and the contributor $\argd$. Given SD-DFQuAD semantics $\sigma$, we have $\ctrbs{\argd}{\arga} \approx 0.0636 > 0$, but as can be seen in
    Figure~\ref{fig:plot-faith-sd}, increasing
    $\argd$'s initial strength does not increase the topic
    argument's strength. 
\end{proof}

\begin{proposition}
$\ctrbsempty$ violates local faithfulness w.r.t.\ EB semantics.
\end{proposition}
\begin{proof}
Consider the QBAG $\graph = \QBAG$ depicted in Figure~\ref{fig:graph-faith-seb}, with the topic argument $\arga$ and the contributor $\argd$. Given EB semantics $\sigma$, we have $\ctrbs{\argd}{\arga}  \approx 0.0007$ and $\ctrbg{\argd}{\arga} \approx -0.0037$; local faithfulness must be violated because $\ctrbs{\argd}{\arga} > 0$ but $\ctrbg{\argd}{\arga} < 0$.
\end{proof}

\begin{figure}[ht]
\centering
\subfloat[$\graph$.]{
\label{fig:graph-faith-seb}
\begin{tikzpicture}[scale=0.7]
    \node[unanode]    (a)    at(2,0)  {\argnode{\arga}{0.3}{0.4376}};
    \node[unanode]    (b)    at(0,2)  {\argnode{\argb}{0.8}{0.8813}};
    \node[unanode]    (c)    at(2,3)  {\argnode{\argc}{0.1}{0.183}};
    \node[unanode]    (d)    at(2,5)  {\argnode{\argd}{0.75}{0.75}};
    \node[unanode]    (e)    at(4,2)  {\argnode{\arge}{0.1}{0.1583}};
    
    \path [->, line width=0.2mm]  (b) edge node[left] {+} (a);
    \path [->, line width=0.2mm]  (c) edge node[left, above] {+} (b);
    \path [->, line width=0.2mm]  (c) edge node[left, above] {-} (e);
    \path [->, line width=0.2mm]  (d) edge node[left] {+} (b);
    \path [->, line width=0.2mm]  (d) edge node[left] {+} (c);
    \path [->, line width=0.2mm]  (d) edge node[left, above] {+} (e);
    \path [->, line width=0.2mm]  (e) edge node[left] {-} (a);
\end{tikzpicture}
}
\subfloat[Final strength of $\arga$, given initial strength of $\argd$.]{
\label{fig:plot-faith-seb}
\includegraphics[width=0.6\columnwidth]{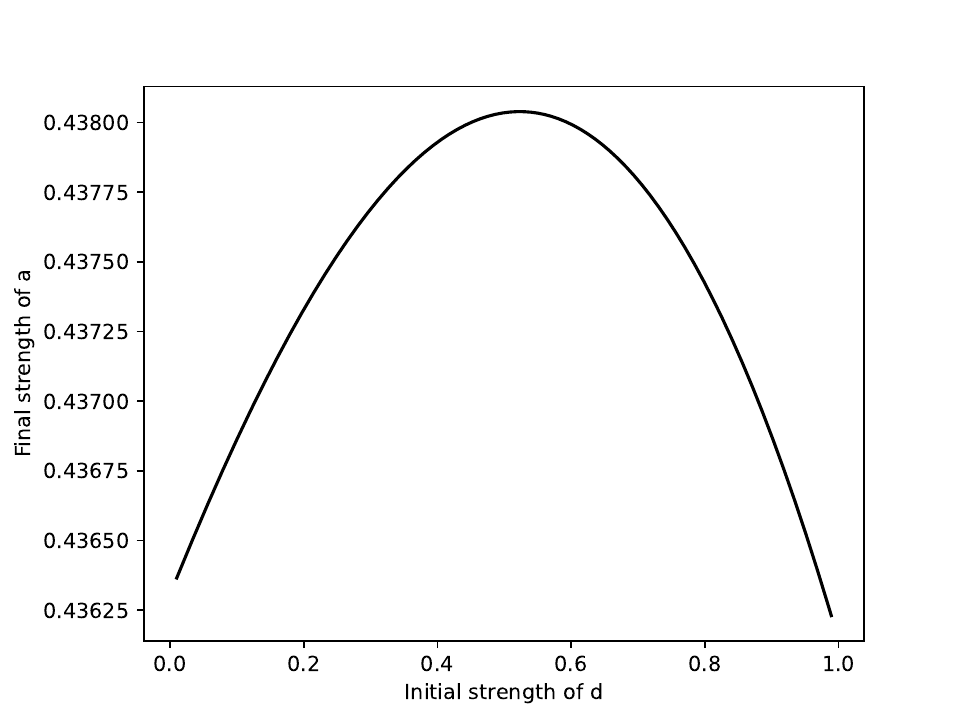}}
\caption{$\ctrbsempty$ violates local faithfulness w.r.t.\ EB semantics.}
\label{fig:faith-seb}
\end{figure}

\begin{proposition}
$\ctrbsempty$ violates local faithfulness w.r.t.\ EBT semantics.
\end{proposition}
\begin{proof}
Consider the QBAG $\graph = \QBAG$ depicted in Figure~\ref{fig:graph-faith-sebt}, with the topic argument $\arga$ and the contributor $\argb$. Given EBT semantics $\sigma$, we have $\ctrbs{\argb}{\arga}  \approx -0.0377 < 0$,
but changing $\argb$'s initial strength actually does not
affect the strength of the topic argument.
\end{proof}

\begin{figure}[ht]
\centering
\subfloat[$\graph$.]{
\label{fig:graph-faith-sebt}
\begin{tikzpicture}[scale=0.8]
    \node[unanode]    (a)    at(2,0)  {\argnode{\arga}{0.5}{0.3665}};
    \node[unanode]    (b)    at(0,2)  {\argnode{\argb}{0.5}{0.5}};
    \node[unanode]    (c)    at(4,2)  {\argnode{\argc}{1}{1}};
    \path [->, line width=0.2mm]  (b) edge node[left] {-} (a);
    \path [->, line width=0.2mm]  (c) edge node[left, below] {-} (a);
\end{tikzpicture}
}
\subfloat[Final strength of $\arga$, given initial strength of $\argb$.]{
\label{fig:plot-faith-sebt}
\includegraphics[width=0.6\columnwidth]{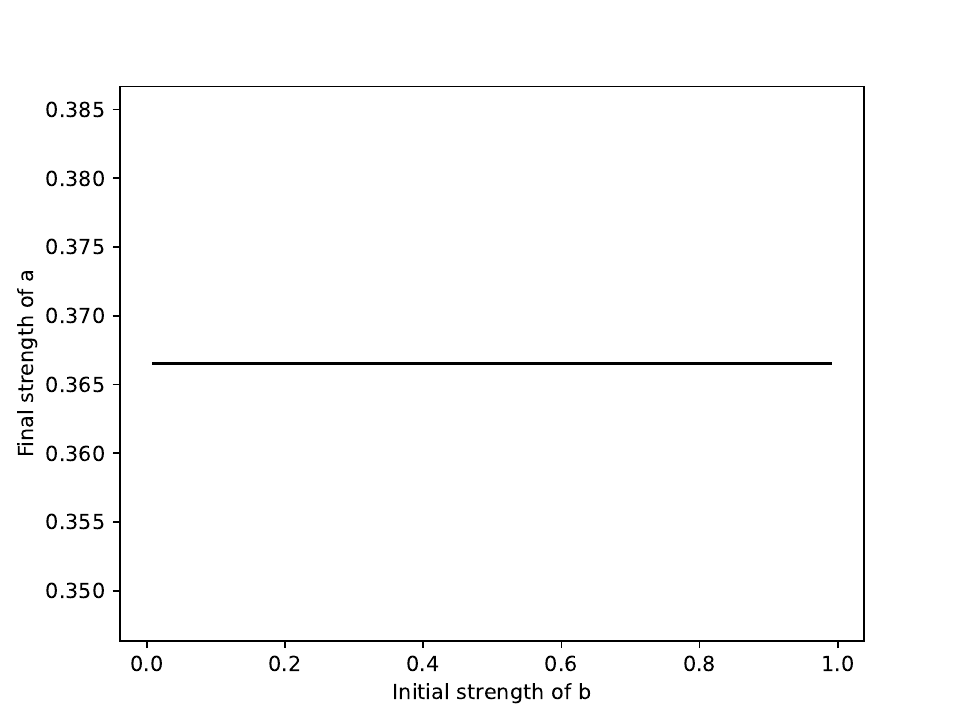}}
\caption{$\ctrbsempty$ violates local faithfulness w.r.t.\ EBT semantics.}
\label{fig:faith-sebt}
\end{figure}

\subsection{(Quantitative) Counterfactuality}
Let us first show that the contribution function $\ctrbrempty$ satisfies quantitative counterfactuality w.r.t.\ all argumentation semantics.
\begin{proposition}
\label{prop:removal-counterfactual}
$\ctrbrempty$ satisfies quantitative counterfactuality w.r.t.\ all argumentation semantics $\fs$.
\end{proposition}
\begin{proof}
    By definition of $\ctrbrempty$, it holds for every $\graph = \QBAG$, for every arguments $\arga, \argx \in \Args$ that $\ctrbr{\argx}{\arga} =  \fs_\graph(\arga) - \fs_{\graph\downarrow_{\Args \setminus \{\argx\}}}(\arga)$. Hence, quantitative counterfactuality (Definition~\ref{principle:q-counterfactual}) is trivially satisfied.
\end{proof}
%
%
%
\begin{corollary}
$\ctrbrempty$ satisfies counterfactuality w.r.t.\ all argumentation semantics $\fs$.
\end{corollary}
\begin{proof}
    Counterfactuality is satisfied because quantitative counterfactuality is satisfied (Proposition~\ref{prop:removal-counterfactual}) and the latter implies the former (Proposition~\ref{prop:qcounterfactual}).
\end{proof}
All other contribution functions can violate counterfactuality and quantitative counterfactuality w.r.t.\ all of the surveyed semantics.
\begin{proposition}
$\ctrbriempty$ violates counterfactuality and quantitative counterfactuality w.r.t.\  QE, DFQuAD, SD-DFQuAD and semantics.
\end{proposition}
\begin{proof}
Consider the QBAG $\graph = \QBAG$ depicted in Figure~\ref{fig:counterfactuality-negative}, with the topic argument $\arga$ and the contributor $\argb$.
Given QE, DFQuAD, or SD-DFQuAD semantics $\sigma$, it trivially holds that $\fs_{\graph}({\argb}) > 0$ and hence, $\fs_{\graph}({\arga}) < \is_{\graph}({\arga}) = 0.8$. Intuitively reflecting the influence of $\argb$ on $\arga$, $\fs_{\graph \downarrow_{\Args \setminus \{\argb\}}}(\arga) = 0.8$, i.e., we have $\fs_{\graph}({\arga}) < \fs_{\graph \downarrow_{\Args \setminus \{\argb\}}}(\arga)$. However, we have $\ctrbri{\argb}{\arga} = 0$ and hence, according to counterfactuality and quantitative counterfactuality, we must have $\fs_{\graph}({\arga}) = \fs_{\graph \downarrow_{\Args \setminus \{\argb\}}}(\arga)$, which proves the violation of the principles.
\end{proof}

\begin{figure}[ht]
\centering
\begin{tikzpicture}[scale=0.8]
     \node[unanode]    (a)    at(4,0)  {\argnode{\arga}{0.8}{<0.8}};
     \node[unanode]  (b)    at(2,0)  {\argnode{\argb}{0}{>0}};
     \node[unanode]    (c)    at(0,0)  {\argnode{\argc}{1}{1}};
     \path [->, line width=0.2mm]  (b) edge node[left, above] {-} (a);
     \path [->, line width=0.2mm]  (c) edge node[left, above] {+} (b);
\end{tikzpicture}
\caption{$\ctrbriempty$ violates the counterfactuality and quantitative counterfactuality principles w.r.t.\ QE, SD-DFQuAD, and DFQuAD semantics.}
\label{fig:counterfactuality-negative}
\end{figure}

\begin{proposition}
$\ctrbriempty$ violates counterfactuality and quantitative counterfactuality w.r.t.\  EB semantics.
\end{proposition}
\begin{proof}
    Consider the QBAG $\graph = \QBAG$ depicted in Figure~\ref{fig:counterfactuality-negative-ri-eb}, with the topic argument $\arga$ and the contributor $\arge$. Given EB semantics $\sigma$, it holds that $\fs_{\graph}({\arga}) < \fs_{\graph \downarrow_{\Args \setminus \{\arge\}}}(\arga)$ (observe that $\fs_{\graph}({\arga}) - \fs_{\graph \downarrow_{\Args \setminus \{\arge\}}}(\arga) \approx -2.5 \times 10^{-6}$). However, we have $\ctrbri{\arge}{\arga} \approx 3.5431 \times 10^{-6}$ and hence, according to counterfactuality and quantitative counterfactuality, we must have $\fs_{\graph}({\arga}) > \fs_{\graph \downarrow_{\Args \setminus \{\arge\}}}(\arga)$, which proves the violation of the principles.
\end{proof}
\begin{figure}[ht]
\centering
\begin{tikzpicture}[scale=0.8]
     \node[unanode]    (a)    at(4,0)  {\argnode{\arga}{0.5}{0.5067}};
     \node[unanode]  (b)    at(0,2)  {\argnode{\argb}{0.1}{0.1043}};
     \node[unanode]    (c)    at(2,2)  {\argnode{\argc}{0.1}{0.1043}};
     \node[unanode]    (d)    at(4,2)  {\argnode{\argd}{0.51}{0.5187}};
     \node[unanode]    (g)    at(6,2)  {\argnode{\argg}{0.27}{0.27}};
     \node[unanode]    (e)    at(2,4)  {\argnode{\arge}{0.02}{0.0519}};
     \node[unanode]    (f)    at(0,4)  {\argnode{\argf}{1}{1}};
     \path [->, line width=0.2mm]  (b) edge node[left, below] {-} (a);
     \path [->, line width=0.2mm]  (c) edge node[left] {-} (a);
     \path [->, line width=0.2mm]  (d) edge node[left] {+} (a);
     \path [->, line width=0.2mm]  (e) edge node[left, above] {+} (b);
     \path [->, line width=0.2mm]  (e) edge node[left] {+} (c);
     \path [->, line width=0.2mm]  (e) edge node[left, above] {+} (d);
     \path [->, line width=0.2mm]  (f) edge node[left, above] {+} (e);
     \path [->, line width=0.2mm]  (g) edge node[left, above] {-} (a);
\end{tikzpicture}
\caption{$\ctrbriempty$ violates the counterfactuality and quantitative counterfactuality principles w.r.t.\ EB semantics.}
\label{fig:counterfactuality-negative-ri-eb}
\end{figure}

\begin{proposition}
$\ctrbriempty$ violates counterfactuality and quantitative counterfactuality w.r.t.\  EBT semantics.
\end{proposition}
\begin{proof}
    Consider the QBAG $\graph = \QBAG$ depicted in Figure~\ref{fig:counterfactuality-negative-ri-ebt}, with the topic argument $\arga$ and the contributor $\argb$. Given EBT semantics $\sigma$, it holds that $\fs_{\graph}({\arga}) < \fs_{\graph \downarrow_{\Args \setminus \{\argb\}}}(\arga)$ (observe that $\fs_{\graph}({\arga}) - \fs_{\graph \downarrow_{\Args \setminus \{\argb\}}}(\arga) \approx -0.0145$). However, we have $\ctrbri{\argb}{\arga} = 0$ and hence, according to counterfactuality and quantitative counterfactuality, we must have $\fs_{\graph}({\arga}) = \fs_{\graph \downarrow_{\Args \setminus \{\argb\}}}(\arga)$, which proves the violation of the principles.
\end{proof}
\begin{figure}[ht]
\centering
\begin{tikzpicture}[scale=0.8]
     \node[unanode]    (a)    at(2,0)  {\argnode{\arga}{0.7}{0.6733}};
     \node[unanode]  (b)    at(0,1)  {\argnode{\argb}{0.1}{0.2216}};
     \node[unanode]    (c)    at(0,3)  {\argnode{\argc}{1}{1}};
      \node[unanode]    (d)    at(4,1)  {\argnode{\argd}{0.1}{0.1}};
     \path [->, line width=0.2mm]  (b) edge node[left, above] {-} (a);
     \path [->, line width=0.2mm]  (c) edge node[left] {+} (b);
     \path [->, line width=0.2mm]  (d) edge node[left, above] {-} (a);
\end{tikzpicture}
\caption{$\ctrbriempty$ violates the counterfactuality and quantitative counterfactuality principles w.r.t.\ EBT semantics.}
\label{fig:counterfactuality-negative-ri-ebt}
\end{figure}

\begin{proposition}
$\ctrbsempty$ violates counterfactuality and quantitative counterfactuality w.r.t.\ QE semantics.
\end{proposition}
\begin{proof}
    Consider the QBAG $\graph = \QBAG$ depicted in Figure~\ref{fig:counterfactuality-negative-shapley-qe}, with the topic argument $\arga$ and the contributor $\arge$. Given QE semantics $\sigma$, it holds that $\fs_{\graph}({\arga}) < \fs_{\graph \downarrow_{\Args \setminus \{\arge\}}}(\arga)$ (observe that $\fs_{\graph}({\arga}) - \fs_{\graph \downarrow_{\Args \setminus \{\arge\}}}(\arga) \approx -0.0149$). However, we have $\ctrbs{\arge}{\arga} \approx 4.9326 \times 10^{-5}$ and hence, according to counterfactuality and quantitative counterfactuality, we must have $\fs_{\graph}({\arga}) > \fs_{\graph \downarrow_{\Args \setminus \{\arge\}}}(\arga)$, which proves the violation of the principles.
\end{proof}

\begin{figure}[ht]
\centering
\begin{tikzpicture}[scale=0.8]
     \node[unanode]    (a)    at(6,2)  {\argnode{\arga}{0.1}{0.0829}};
     \node[unanode]  (b)    at(4,4)  {\argnode{\argb}{0.15}{0.4547}};
     \node[unanode]    (c)    at(4,2)  {\argnode{\argc}{0.15}{0.4547}};
     \node[unanode]    (d)    at(4,0)  {\argnode{\argd}{0.15}{0.4547}};
     \node[unanode]    (e)    at(2,2)  {\argnode{\arge}{0.495}{0.7475}};
     \node[unanode]    (f)    at(0,2)  {\argnode{\argf}{1}{1}};
     \path [->, line width=0.2mm]  (b) edge node[left] {-} (a);
     \path [->, line width=0.2mm]  (c) edge node[left, above] {-} (a);
     \path [->, line width=0.2mm]  (d) edge node[left] {+} (a);
     \path [->, line width=0.2mm]  (e) edge node[left, above] {+} (b);
     \path [->, line width=0.2mm]  (e) edge node[left, above] {+} (c);
     \path [->, line width=0.2mm]  (e) edge node[left, above] {+} (d);
     \path [->, line width=0.2mm]  (f) edge node[left, above] {+} (e);
\end{tikzpicture}
\caption{$\ctrbsempty$ violates the counterfactuality and quantitative counterfactuality principles w.r.t.\ QE semantics.}
\label{fig:counterfactuality-negative-shapley-qe}
\end{figure}

\begin{proposition}
$\ctrbsempty$ violates counterfactuality and quantitative counterfactuality w.r.t.\ DFQuAD semantics.
\end{proposition}
\begin{proof}
    Consider the QBAG $\graph = \QBAG$ depicted in Figure~\ref{fig:counterfactuality-negative-shapley-df}, with the topic argument $\arga$ and the contributor $\arge$. Given DFQuAD semantics $\sigma$, it holds that $\fs_{\graph}({\arga}) < \fs_{\graph \downarrow_{\Args \setminus \{\arge\}}}(\arga)$ (observe that $\fs_{\graph}({\arga}) - \fs_{\graph \downarrow_{\Args \setminus \{\arge\}}}(\arga) \approx -0.0109$). However, we have $\ctrbs{\arge}{\arga} \approx 0.0021$ and hence, according to counterfactuality and quantitative counterfactuality, we must have $\fs_{\graph}({\arga}) > \fs_{\graph \downarrow_{\Args \setminus \{\arge\}}}(\arga)$, which proves the violation of the principles.
\end{proof}

\begin{figure}[ht]
\centering
\begin{tikzpicture}[scale=0.8]
     \node[unanode]    (a)    at(6,2)  {\argnode{\arga}{0.1}{0.1}};
     \node[unanode]  (b)    at(4,4)  {\argnode{\argb}{0.15}{1}};
     \node[unanode]    (c)    at(4,2)  {\argnode{\argc}{0.17}{1}};
     \node[unanode]    (d)    at(4,0)  {\argnode{\argd}{0.3}{1}};
     \node[unanode]    (e)    at(2,2)  {\argnode{\arge}{0.495}{1}};
     \node[unanode]    (f)    at(0,2)  {\argnode{\argf}{1}{1}};
     \path [->, line width=0.2mm]  (b) edge node[left] {-} (a);
     \path [->, line width=0.2mm]  (c) edge node[left, above] {-} (a);
     \path [->, line width=0.2mm]  (d) edge node[left] {+} (a);
     \path [->, line width=0.2mm]  (e) edge node[left, above] {+} (b);
     \path [->, line width=0.2mm]  (e) edge node[left, above] {+} (c);
     \path [->, line width=0.2mm]  (e) edge node[left, above] {+} (d);
     \path [->, line width=0.2mm]  (f) edge node[left, above] {+} (e);
\end{tikzpicture}
\caption{$\ctrbsempty$ violates the counterfactuality and quantitative counterfactuality principles w.r.t.\ DFQuAD semantics.}
\label{fig:counterfactuality-negative-shapley-df}
\end{figure}

\begin{proposition}
$\ctrbsempty$ violates counterfactuality and quantitative counterfactuality w.r.t.\ SD-DFQuAD semantics.
\end{proposition}
\begin{proof}
    Consider the QBAG $\graph = \QBAG$ depicted in Figure~\ref{fig:counterfactuality-negative-shapley-sd}, with the topic argument $\arga$ and the contributor $\arge$. Given SD-DFQuAD semantics $\sigma$, it holds that $\fs_{\graph}({\arga}) < \fs_{\graph \downarrow_{\Args \setminus \{\arge\}}}(\arga)$ (observe that $\fs_{\graph}({\arga}) - \fs_{\graph \downarrow_{\Args \setminus \{\arge\}}}(\arga) \approx -0.0049$). However, we have $\ctrbs{\arge}{\arga} \approx 0.0027$ and hence, according to counterfactuality and quantitative counterfactuality, we must have $\fs_{\graph}({\arga}) > \fs_{\graph \downarrow_{\Args \setminus \{\arge\}}}(\arga)$, which proves the violation of the principles.
\end{proof}

\begin{figure}[ht]
\centering
\begin{tikzpicture}[scale=0.8]
     \node[unanode]    (a)    at(6,2)  {\argnode{\arga}{0.1}{0.0819}};
     \node[unanode]  (b)    at(4,4)  {\argnode{\argb}{0.15}{0.5136}};
     \node[unanode]    (c)    at(4,2)  {\argnode{\argc}{0.15}{0.5136}};
     \node[unanode]    (d)    at(4,0)  {\argnode{\argd}{0.2}{0.5422}};
     \node[unanode]    (e)    at(2,2)  {\argnode{\arge}{0.495}{0.7475}};
     \node[unanode]    (f)    at(0,2)  {\argnode{\argf}{1}{1}};
     \path [->, line width=0.2mm]  (b) edge node[left] {-} (a);
     \path [->, line width=0.2mm]  (c) edge node[left, above] {-} (a);
     \path [->, line width=0.2mm]  (d) edge node[left] {+} (a);
     \path [->, line width=0.2mm]  (e) edge node[left, above] {+} (b);
     \path [->, line width=0.2mm]  (e) edge node[left, above] {+} (c);
     \path [->, line width=0.2mm]  (e) edge node[left, above] {+} (d);
     \path [->, line width=0.2mm]  (f) edge node[left, above] {+} (e);
\end{tikzpicture}
\caption{$\ctrbsempty$ violates the counterfactuality and quantitative counterfactuality principles w.r.t.\ SD-DFQuAD semantics.}
\label{fig:counterfactuality-negative-shapley-sd}
\end{figure}

\begin{proposition}
$\ctrbsempty$ violates counterfactuality and quantitative counterfactuality w.r.t.\ EB semantics.
\end{proposition}
\begin{proof}
Consider the QBAG $\graph = \QBAG$ depicted in Figure~\ref{fig:counterfactuality-negative-shapley-eb}, with the topic argument $\arga$ and the contributor $\argf$. Given EB semantics $\sigma$, it holds that $\fs_{\graph}({\arga}) < \fs_{\graph \downarrow_{\Args \setminus \{\argf\}}}(\arga)$ (observe that $\fs_{\graph}({\arga}) - \fs_{\graph \downarrow_{\Args \setminus \{\argf\}}}(\arga) \approx -7.8369 \times 10^{-5}$). However, we have $\ctrbs{\argf}{\arga} \approx 3.438 \times 10^{-6}$ and hence, according to counterfactuality and quantitative counterfactuality, we must have $\fs_{\graph}({\arga}) > \fs_{\graph \downarrow_{\Args \setminus \{\argf\}}}(\arga)$, which proves the violation of the principles.
\end{proof}

\begin{figure}[ht]
\centering
\begin{tikzpicture}[scale=0.8]
     \node[unanode]    (a)    at(4,0)  {\argnode{\arga}{0.3}{0.2888}};
     \node[unanode]  (b)    at(0,2)  {\argnode{\argb}{0.11}{0.1158}};
     \node[unanode]    (c)    at(2,2)  {\argnode{\argc}{0.1}{0.1054}};
     \node[unanode]    (d)    at(4,2)  {\argnode{\argd}{0.54}{0.5505}};
     \node[unanode]    (g)    at(6,2)  {\argnode{\argg}{0.4}{0.4}};
     \node[unanode]    (e)    at(2,4)  {\argnode{\arge}{0.025}{0.0642}};
     \node[unanode]    (f)    at(0,4)  {\argnode{\argf}{1}{1}};
     \path [->, line width=0.2mm]  (b) edge node[left, below] {-} (a);
     \path [->, line width=0.2mm]  (c) edge node[left] {-} (a);
     \path [->, line width=0.2mm]  (d) edge node[left] {+} (a);
     \path [->, line width=0.2mm]  (e) edge node[left, above] {+} (b);
     \path [->, line width=0.2mm]  (e) edge node[left] {+} (c);
     \path [->, line width=0.2mm]  (e) edge node[left, above] {+} (d);
     \path [->, line width=0.2mm]  (f) edge node[left, above] {+} (e);
     \path [->, line width=0.2mm]  (g) edge node[left, above] {-} (a);
\end{tikzpicture}
\caption{$\ctrbsempty$ violates the counterfactuality and quantitative counterfactuality principles w.r.t.\ EB semantics.}
\label{fig:counterfactuality-negative-shapley-eb}
\end{figure}

\begin{proposition}
$\ctrbsempty$ violates counterfactuality and quantitative counterfactuality w.r.t.\ EBT semantics.
\end{proposition}
\begin{proof}
    Consider the QBAG $\graph = \QBAG$ depicted in Figure~\ref{fig:counterfactuality-negative-shapley-ebt}, with the topic argument $\arga$ and the contributor $\argf$. Given EBT semantics $\sigma$, it holds that $\fs_{\graph}({\arga}) > \fs_{\graph \downarrow_{\Args \setminus \{\argf\}}}(\arga)$ (observe that $\fs_{\graph}({\arga}) - \fs_{\graph \downarrow_{\Args \setminus \{\argf\}}}(\arga) \approx 7.3331 \times 10^{-5}$). However, we have $\ctrbs{\argf}{\arga} \approx -2.7043 \times 10^{-5}$ and hence, according to counterfactuality and quantitative counterfactuality, we must have $\fs_{\graph}({\arga}) < \fs_{\graph \downarrow_{\Args \setminus \{\argf\}}}(\arga)$, which proves the violation of the principles.
\end{proof}

\begin{figure}[ht]
\centering
\begin{tikzpicture}[scale=0.8]
     \node[unanode]    (a)    at(4,0)  {\argnode{\arga}{0.3}{0.3030}};
     \node[unanode]  (b)    at(0,2)  {\argnode{\argb}{0.4}{0.4250}};
     \node[unanode]    (c)    at(4,4)  {\argnode{\argc}{0.55}{0.55}};
     \node[unanode]    (d)    at(4,2)  {\argnode{\argd}{0.51}{0.4474}};
     \node[unanode]    (g)    at(7,2)  {\argnode{\argg}{0.429}{0.429}};
     \node[unanode]    (e)    at(2,4)  {\argnode{\arge}{0.25}{0.0141}};
     \node[unanode]    (f)    at(0,4)  {\argnode{\argf}{1}{1}};
     \path [->, line width=0.2mm]  (b) edge node[left, below] {-} (a);
     \path [->, line width=0.2mm]  (c) edge node[left] {-} (d);
     \path [->, line width=0.2mm]  (d) edge node[left] {+} (a);
     \path [->, line width=0.2mm]  (e) edge node[left, above] {+} (b);
     \path [->, line width=0.2mm]  (e) edge node[left] {+} (d);
     \path [->, line width=0.2mm]  (f) edge node[left, above] {-} (e);
     \path [->, line width=0.2mm]  (g) edge node[left, above] {-} (a);
     \path [->, line width=0.2mm]  (g) edge node[left, above] {-} (d);
\end{tikzpicture}
\caption{$\ctrbsempty$ violates the counterfactuality and quantitative counterfactuality principles w.r.t.\ EBT semantics.}
\label{fig:counterfactuality-negative-shapley-ebt}
\end{figure}

\begin{proposition}
$\ctrbgempty$ violates counterfactuality and quantitative counterfactuality w.r.t.\ QE semantics.
\end{proposition}
\begin{proof}
    Consider the QBAG $\graph = \QBAG$ depicted in Figure~\ref{fig:counterfactuality-gradient-qe}, with the topic argument $\arga$ and the contributor $\argd$. Given QE semantics $\sigma$, it holds that $\fs_{\graph}({\arga}) < \fs_{\graph \downarrow_{\Args \setminus \{\argd\}}}(\arga)$ (observe that $\fs_{\graph}({\arga}) - \fs_{\graph \downarrow_{\Args \setminus \{\argd\}}}(\arga) \approx -0.0038$). However, we have $\ctrbg{\argf}{\arga} = 0$ and hence, according to counterfactuality and quantitative counterfactuality, we must have $\fs_{\graph}({\arga}) = \fs_{\graph \downarrow_{\Args \setminus \{\argd\}}}(\arga)$, which proves the violation of the principles.
\end{proof}

\begin{figure}[ht]
\centering
\begin{tikzpicture}[scale=0.8]
     \node[unanode]    (a)    at(4,0)  {\argnode{\arga}{0.5}{0.5}};
     \node[unanode]  (b)    at(4,2)  {\argnode{\argb}{0.2}{0.2}};
     \node[unanode]    (c0)    at(-1.5,4)  {\argnode{\argca} {0.35}{0.4024}};
     \node[unanode]    (c1)    at(1,4)  {\argnode{\argcb} {0.35}{0.4024}};
     \node[unanode]    (c2)    at(3,4)  {\argnode{\argcc} {0.35}{0.4024}};
     \node[unanode]    (c3)    at(5,4)  {\argnode{\argcd} {0.35}{0.4024}};
     \node[unanode]    (c9)    at(9.5,4)  {\argnode{\argcj} {0.35}{0.4024}};
     \node[invnode]    (cx)    at(7,4)  {\argdots};
     \node[unanode]    (d)    at(4,6)  {\argnode{\argd}{0.296}{0.296}};
     \node[unanode]    (e)    at(0,1)  {\argnode{\arge}{0.2}{0.2}};
     \path [->, line width=0.2mm]  (b) edge node[left] {+} (a);
     \path [->, line width=0.2mm]  (c1) edge node[left, below] {+} (b);
     \path [->, line width=0.2mm]  (d) edge node[left, above] {+} (c1);
     \path [->, line width=0.2mm]  (c2) edge node[left] {+} (b);
     \path [->, line width=0.2mm]  (d) edge node[left, above] {+} (c2);
     \path [->, line width=0.2mm]  (c0) edge node[left, below] {+} (b);
     \path [->, line width=0.2mm]  (d) edge node[left, above] {+} (c0);
     \path [->, line width=0.2mm]  (c3) edge node[left] {-} (b);
     \path [->, line width=0.2mm]  (d) edge node[left, above] {+} (c3);
     \path [->, line width=0.2mm]  (c9) edge node[left] {-} (b);
     \path [->, line width=0.2mm]  (d) edge node[left, above] {+} (c9);
     \path [->, line width=0.2mm]  (cx) edge node[left] {-} (b);
     \path [->, line width=0.2mm]  (d) edge node[left, above] {+} (cx);
     \path [->, line width=0.2mm]  (e) edge node[left, above] {-} (a);
\end{tikzpicture}
\caption{$\ctrbgempty$ violates the counterfactuality and quantitative counterfactuality principles w.r.t.\ QE semantics. Note that ``$\dots$'' (with incoming support and outgoing attack) represents arguments $\mathsf{c}_4$ - $\mathsf{c}_8$ with same initial strength and incoming support and outgoing attack as $\argcd$ and $\argcj$.}
\label{fig:counterfactuality-gradient-qe}
\end{figure}

\begin{proposition}
$\ctrbgempty$ violates counterfactuality and quantitative counterfactuality w.r.t.\ DFQuAD semantics.
\end{proposition}
\begin{proof}
    Consider the QBAG $\graph = \QBAG$ depicted in Figure~\ref{fig:intro}, with the topic argument $\arga$ and the contributor $\arge$. Given DFQuAD semantics $\sigma$, it holds that $\fs_{\graph}({\arga}) < \fs_{\graph \downarrow_{\Args \setminus \{\arge\}}}(\arga)$ (observe that $\fs_{\graph}({\arga}) - \fs_{\graph \downarrow_{\Args \setminus \{\arge\}}}(\arga) = -0.125$). However, we have $\ctrbg{\arge}{\arga} = 0$ and hence, according to counterfactuality and quantitative counterfactuality, we must have $\fs_{\graph}({\arga}) = \fs_{\graph \downarrow_{\Args \setminus \{\arge\}}}(\arga)$, which proves the violation of the principles.
\end{proof}

\begin{proposition}
$\ctrbgempty$ violates counterfactuality and quantitative counterfactuality w.r.t.\ SD-DFQuAD semantics.
\end{proposition}
\begin{proof}
    Consider the QBAG $\graph = \QBAG$ depicted in Figure~\ref{fig:counterfactuality-gradient-sd}, with the topic argument $\arga$ and the contributor $\argb$.
    Given SD-DFQuAD semantics $\sigma$, it trivially holds that $\fs_{\graph}({\argb}) = 0$; and hence, $\fs_{\graph}({\arga}) = \is_{\graph}({\arga}) = 0.5$. Intuitively reflecting the influence of $\argb$ on $\arga$, $\fs_{\graph \downarrow_{\Args \setminus \{\argb\}}}(\arga) = 0.5$, i.e., we have $\fs_{\graph}({\arga}) = \fs_{\graph \downarrow_{\Args \setminus \{\argb\}}}(\arga)$. However, we have $\ctrbg{\argb}{\arga} = -0.25$ and hence, according to counterfactuality and quantitative counterfactuality, we must have $\fs_{\graph}({\arga}) < \fs_{\graph \downarrow_{\Args \setminus \{\argb\}}}(\arga)$, which proves the violation of the principles.
\end{proof}

\begin{figure}[ht]
\centering
\begin{tikzpicture}[scale=0.8]
     \node[unanode]    (a)    at(4,0)  {\argnode{\arga}{0.5}{0.5}};
     \node[unanode]  (b)    at(2,0)  {\argnode{\argb}{0}{0}};
     \node[unanode]    (c)    at(0,0)  {\argnode{\argc}{1}{1}};
     \path [->, line width=0.2mm]  (b) edge node[left, above] {-} (a);
     \path [->, line width=0.2mm]  (c) edge node[left, above] {-} (b);
\end{tikzpicture}
\caption{$\ctrbgempty$ violates the counterfactuality and quantitative counterfactuality principles w.r.t.\ SD-DFQuAD semantics.}
\label{fig:counterfactuality-gradient-sd}
\end{figure}

\begin{proposition}
$\ctrbgempty$ violates counterfactuality and quantitative counterfactuality w.r.t.\ EB and EBT semantics.
\end{proposition}
\begin{proof}
    Consider the QBAG $\graph = \QBAG$ depicted in Figure~\ref{fig:counterfactuality-gradient-eb}, with the topic argument $\arga$ and the contributor $\argb$.
    Given EB or EBT semantics $\sigma$, it trivially holds that $\fs_{\graph}({\argb}) = 0$; and hence, $\fs_{\graph}({\arga}) = \is_{\graph}({\arga}) = 0.5$. Intuitively reflecting the influence of $\argb$ on $\arga$, $\fs_{\graph \downarrow_{\Args \setminus \{\argb\}}}(\arga) = 0.5$, i.e., we have $\fs_{\graph}({\arga}) = \fs_{\graph \downarrow_{\Args \setminus \{\argb\}}}(\arga)$. However, we have $\ctrbri{\argb}{\arga} \approx -0.4530$ and hence, according to counterfactuality and quantitative counterfactuality, we must have $\fs_{\graph}({\arga}) < \fs_{\graph \downarrow_{\Args \setminus \{\argb\}}}(\arga)$, which proves the violation of the principles.
\end{proof}

\begin{figure}[ht]
\centering
\begin{tikzpicture}[scale=0.8]
     \node[unanode]    (a)    at(4,0)  {\argnode{\arga}{0.5}{0.5}};
     \node[unanode]  (b)    at(2,0)  {\argnode{\argb}{0}{0}};
     \node[unanode]    (c)    at(0,0)  {\argnode{\argc}{1}{1}};
     \path [->, line width=0.2mm]  (b) edge node[left, above] {-} (a);
     \path [->, line width=0.2mm]  (c) edge node[left, above] {+} (b);
\end{tikzpicture}
\caption{$\ctrbgempty$ violates the counterfactuality and quantitative counterfactuality principles w.r.t.\ EB and EBT semantics.}
\label{fig:counterfactuality-gradient-eb}
\end{figure}

\section{Minor Results and Conjectures}
\label{sec:rest}
In this section, we describe minor results obtained as a by-product of our search for and analysis of contribution function principles.
We also speculate about additional results that future research may obtain.

\subsection{Proximity}
Intuitively, one may expect that arguments that are \emph{strictly closer} to a topic argument, considering the direction of the attack/support relations, contribute more.
In order to turn this intuition into a principle, let us first define what we mean by \emph{strictly closer}.
\begin{definition}[Strictly Closer]
    Given $\graph = \QBAG$ and $\arga, \argx, \argy \in \Args$, we say that $\argy$ is strictly closer to $\arga$ than $\argx$ is iff $\argy$ is on every directed path (in $\graph$) from $\argx \in \Args$ to $\arga$.
\end{definition}
Consider the QBAG in Figure~\ref{fig:counterfactuality-negative-shapley-qe} and the topic argument $\arga$.
Here, $\arge$ is strictly closer to $\arga$ than $\argf$ is: we can reach $\arga$ from $\argf$ only through $\arge$.
In contrast, $\argb$ is not strictly closer to $\arga$ than $\arge$ is, because we can reach $\arga$ from $\arge$ through $\argc$ (and $\argb$ is not on the corresponding path). 
Now, we can formalise the intuition above by defining the \emph{proximity principle}.
\begin{principle}[Proximity]
\label{principle:proximity}
$\ctrbempty$ satisfies the \emph{proximity principle} w.r.t.\ a gradual semantics $\fs$ iff whenever $\argy \in \Args$ is strictly closer to $\arga$ than $\argx$ is then $|\ctrb{\argy}{\arga}| \geq |\ctrb{\argx}{\arga}|$.
\end{principle}
Interestingly, proximity is violated for very simple cases.
For example, consider $\ctrbrempty$, QE semantics, and the QBAG in Figure~\ref{fig:proximity-removal-qe}, with topic argument $\arga$ and contributors $\argb$ and $\argc$.
Although $\argb$ is obviously strictly closer to $\arga$ than $\argc$ is,
we have $|\ctrbr{\argb}{\arga}| \approx 0.0012 < |\ctrbr{\argc}{\arga}| \approx 0.0037$.
Intuitively, $\argb$ is so substantially weakened by $\argc$ that the removal of $\argb$ only marginally affects the final strength of $\arga$, whereas the removal of $\argc$ substantially strengthens $\argb$ and thus substantially weakens $\arga$. 
\begin{figure}[ht]
\centering
\begin{tikzpicture}[scale=0.8]
     \node[unanode]    (a)    at(4,0)  {\argnode{\arga}{0.5}{0.4988}};
     \node[unanode]  (b)    at(2,0)  {\argnode{\argb}{0.1}{0.05}};
     \node[unanode]    (c)    at(0,0)  {\argnode{\argc}{1}{1}};
     \path [->, line width=0.2mm]  (b) edge node[left, above] {-} (a);
     \path [->, line width=0.2mm]  (c) edge node[left, above] {-} (b);
\end{tikzpicture}
\caption{$\ctrbrempty$ violates the proximity principle w.r.t.\ QE semantics.}
\label{fig:proximity-removal-qe}
\end{figure}
Indeed, we could not find any contribution function that satisfies proximity with respect to any semantics (leaving the case for $\ctrbgempty$ and EBT open).
Hence, we opted to exclude proximity from our main list of principles.
\begin{proposition}
    $\ctrbrempty$, $\ctrbriempty$, and $\ctrbsempty$ violate proximity w.r.t.\ QE, DFQuAD, SD-DFQuAD, EB and EBT semantics;
    $\ctrbgempty$ violates proximity w.r.t.\ QE, DFQuAD, SD-DFQuAD, and EB semantics.
\end{proposition}
The proofs are provided in the appendix.
We leave questions whether $\ctrbgempty$ satisfies or violates proximity with respect to EBT semantics for future work. 

We speculate that proximity may be satisfied by the Shapley values-based contribution function $\ctrbsempty$ with respect to some semantics in special cases. By the assumptions of proximity, an argument $\argc$ that is closer to the topic argument $\arga$ is
on every path from $\argb$ to $\argc$. This means that if we have a ``coalition'' of arguments that does not
contain $\argc$, then adding $\argb$ will not result in a marginal contribution.
For the special case of an argumentation framework in which $\arga$, $\argb$, and $\argc$ are connected only via the support relation of the QBAG one would expect that if the coalition of arguments does contain $\argc$ and adding $\argb$ results in a marginal contribution, then adding $\argc$ to the coalition where
we replace $\argc$ with $\argb$ will result in an absolute marginal contribution that is at least
as large.
As a prerequisite for our conjecture, let us define the notion of a \emph{pure support path}.
\begin{definition}[Pure Support Paths]
    Given a QBAG $\QBAG$ and $\arga, \argb \in \Args$, a path $\arga$ to $\argb$ is a pure support path iff it is a path in the directed graph $(\Args, \Supp)$.
\end{definition}
Note that by definition of a QBAG support and attack relations are disjoint, so no edge occurring in a pure support path $\arga$ to $\argb$ is in $\Att$.

Let us now provide the conjecture, assuming that we can define principle satisfaction with respect to arguments with specific graph-topological properties.
\begin{conjecture}
  $\ctrbsempty$ satisfies proximity w.r.t. QE, DFQuAD, SD-DFQuAD, EB, and EBT semantics, for contributors $\argx \in \Args$ and topic arguments $\arga \in \Args$ for which it holds that all paths from $\argx$ to $\arga$ are pure support paths.
\end{conjecture}

\subsection{Violation of Strong Faithfulness}
\label{sec:violation-strong-faithfulness}
In Subsection~\ref{sec:principle-faithfulness}, we
argued that \emph{strong faithfulness} is too strong
as a property because it is incompatible with the fact
that the effect of arguments can be non-monotonic
(they may be positive in one region and negative in another)
or locally but not globally neutral.
Here, we elaborate on this idea and give concrete 
counterexamples for various semantics.
\begin{proposition}
$\ctrbrempty$, $\ctrbriempty$, $\ctrbsempty$, and $\ctrbgempty$ violate strong faithfulness w.r.t. QE, DFQuAD, SD-DFQuAD, EB, and EBT semantics.
\end{proposition}
\begin{proof}
Observe that if strong faithfulness is satisfied then one of the following statements must be true for all QBAGs $\QBAG$, all $\arga, \argx \in \Args$ and all $\epsilon, \epsilon'  \in \interval$ s.t. $\epsilon  > \is(\argx) > \epsilon'$, and $\graph_\epsilon = \graph\downarrow_{\is(\argx) \leftarrow \epsilon}$ and $\graph_\epsilon' = \graph\downarrow_{\is(\argx) \leftarrow \epsilon'}$:
\begin{enumerate}
    \item $\sigma_{\graph_\epsilon}(\arga) < \sigma_{\graph}(\arga) < \sigma_{\graph_\epsilon'}(\arga)$;
    \item $\sigma_{\graph_\epsilon}(\arga) > \sigma_{\graph}(\arga) > \sigma_{\graph_\epsilon'}(\arga)$;
    \item $\sigma_{\graph_\epsilon}(\arga) = \sigma_{\graph}(\arga) = \sigma_{\graph_\epsilon'}(\arga)$.
\end{enumerate}
This is not the case for any of QE, DFQuAD, SD-DFQuAD, EB, and EBT semantics:
\begin{description}
    \item[QE semantics.] Consider the QBAG in Figure~\ref{fig:graph-faith-qe}, contributor $\argd$ and topic argument  $\arga$, and the plot in Figure~\ref{fig:plot-faith-qe} for values of $\epsilon$ and $\epsilon'$ that provide a counter-example.
    \item[DFQuAD semantics.] Consider the QBAG in Figure~\ref{fig:graph-intro}, contributor $\arge$ and topic argument $\arga$, and the plot in Figure~\ref{fig:plot-intro} for values of $\epsilon$ and $\epsilon'$ that provide a counter-example.
    \item[SD-DFQuAD semantics.] Consider the QBAG in Figure~\ref{fig:graph-faith-sd}, contributor $\argd$ and topic argument $\arga$, and the plot in Figure~\ref{fig:plot-faith-sd} for values of $\epsilon$ and $\epsilon'$ that provide a counter-example.
    \item[EB semantics.] Consider the QBAG in Figure~\ref{fig:graph-faith-eb}, contributor $\argd$ and topic argument $\arga$, and the plot in Figure~\ref{fig:plot-faith-eb} for values of $\epsilon$ and $\epsilon'$ that provide a counter-example.
    \item[EBT semantics.] Consider the QBAG in Figure~\ref{fig:graph-faith-ebt}, contributor $\argd$ and topic argument $\arga$, and the plot in Figure~\ref{fig:plot-faith-ebt} for values of $\epsilon$ and $\epsilon'$ that provide a counter-example.
\end{description}
Let us provide a step-by-step walk-through for the case of DFQuAD semantics (and the QBAG/plot in Figure~\ref{fig:intro}); the other cases follow analogously from the examples referenced above. In the QBAG on the left (Figure~\ref{fig:graph-intro}), argument $\arge$ has a non-monotonic influence on the topic argument $\arga$. Figure~\ref{fig:plot-intro} plots the final strength of $\arga$ (y-axis) as a function of the initial strength of $\arge$ under DFQuAD semantics. $\arge$'s initial strength
has a negative influence up to $0.5$. Then the influence 
becomes positive.
The plot illustrates how the initial strength
of $\arge$ influences the final strength of $\arga$. As we increase
$\is(\arge)$ from $0$ to $0.5$, $\arga$ becomes weaker. However,
at this point, the effect reverses, and increasing $\is(\arge)$
further will make $\arge$ stronger.
In particular, if we let $\is(\arge) = 0.2$ in the QBAG on the left in Figure~\ref{fig:graph-intro}, then the effect is negative
for $\epsilon \in [0,0.8)$, neutral for $\epsilon = 0.8$ and
positive for $\epsilon \in (0.8,1]$.
We provide analogous examples for QE semantics (Figure~\ref{fig:faith-qe}), as well as for EB semantics (Figure~\ref{fig:faith-eb}); for SD-DFQuAD and EBT semantics, we provide counterexamples for the case $\ctrb{\argx}{\arga} = 0$ ($\fs_\graph(\arga) \neq \fs_{\graph_\epsilon}(\arga)$ for at least some $\epsilon \in \interval$) in Figures~\ref{fig:faith-sd} and~\ref{fig:faith-ebt}, respectively.
\end{proof}

\subsection{Strong Faithfulness given Monotonic Effects}
In Subsection~\ref{sec:violation-strong-faithfulness}, we show that strong faithfulness is violated by the gradient-based contribution function with respect to all of the surveyed argumentation semantics.
In the counter-examples we provide, we can see that intuitively, the contributor we are interested in does not have a \emph{monotonic effect} on the topic argument: the effect that marginally changing the initial strength of the contributor has on the topic argument may be positive, zero, or negative, depending on the contributor's initial strength (see Figure~\ref{fig:plot-intro}).
This means that strong faithfulness must be violated; indeed, from a common sense perspective, the principle \emph{should not be satisfied} in such situations.
However, we may consider the satisfaction of strong faithfulness desirable in situations in which the effect of the contributor on the topic argument is monotonic, an intuition which we formalise below. 
\begin{definition}[Monotonic Effect]\label{def:monotonic-effect}
Given a semantics $\fs$, a QBAG $\graph = \QBAG$, and arguments $\argx, \arga \in \Args$, we say that ``$\argx$ has a monotonic effect on $\arga$'' w.r.t. $\fs$ iff for every QBAG $\graph\downarrow_{\is(\argx) \leftarrow \epsilon}$, for every $\epsilon, \epsilon' \in \mathbb{I}$ the following holds:
\begin{itemize}
    \item either, for every QBAG $\graph\downarrow_{\is(\argx) \leftarrow \epsilon'}$, if $\epsilon < \epsilon'$ then $\fs_{\graph\downarrow_{\is(\argx) \leftarrow \epsilon}}(\arga) \leq \fs_{\graph\downarrow_{\is(\argx) \leftarrow \epsilon'}}(\arga)$;
    \item or, for every QBAG $\graph\downarrow_{\is(\argx) \leftarrow \epsilon'}$, if $\epsilon < \epsilon'$ then $\fs_{\graph\downarrow_{\is(\argx) \leftarrow \epsilon}}(\arga) \geq \fs_{\graph\downarrow_{\is(\argx) \leftarrow \epsilon'}}(\arga)$.
\end{itemize}
\end{definition}
Now, we can conjecture that $\ctrbgempty$ satisfies strong faithfulness given contributors that have a monotonic effect on our topic argument, i.e., given the class of QBAGs with respect to a semantics $\sigma$ in which the effect of any argument on another one in any QBAG is monotonic.
As a prerequisite for this conjecture, let us formalise the notion of monotonic effect QBAGs.
\begin{definition}[Monotonic Effect QBAGs]
Given a semantics $\fs$, a QBAG $\graph = \QBAG$ is a monotonic effect QBAG w.r.t. $\fs$ iff for all  $\argx, \arga \in \Args$ it holds that $\argx$ has a monotonic effect on $\arga$ w.r.t. $\fs$.
We denote the class of monotonic effect QBAGs w.r.t. $\fs$ by ${\cal Q}_{ME}^\fs$.
\end{definition}
Now, we can speculate that the gradient-based contribution function satisfies strong faithfulness with respect to all of the surveyed argumentation semantics when assuming monotonic effect QBAGs.
\begin{conjecture}
Let $\fs$ be QE, DFQuAD, SD-DFQuAD, EB, or EBT semantics and assume the class of QBAGs is ${Q}_{ME}^\fs$.
$\ctrbgempty$ satisfies strong faithfulness w.r.t. $\sigma$.
\end{conjecture}
We can further speculate about topological properties that can guarantee a monotonic effect.
For example, in some cases, the role that a contributor has in terms of (potentially indirect) attack or support of a topic argument may be clear:
\begin{itemize}
    \item a contributor may either directly attack or support the topic argument (and hence is a direct attacker or supporter);
    \item a contributor may attack direct or indirect supporters of the topic argument or support its direct or indirect attackers and hence is an indirect attacker;
    \item a contributor may attack direct or indirect attackers of the topic argument or support its direct or indirect supporters and hence is an indirect supporter;
    \item finally, a contributor is a clear attacker of a topic argument if it is a direct or indirect attacker but not a direct or indirect supporter and it is a  clear supporter if it is a direct or indirect supporter but not a direct or indirect attacker.
\end{itemize}
Now, we may speculate that contributors that are clear supporters of a topic argument must have an effect on the topic argument that is always $\geq 0$, no matter the contributor's initial strength; in the case of clear attackers the effect should be always $\leq 0$. 

\section{Case Study}
\label{sec:case-study}
To give a basic intuition of practical applications, we apply our 
four contribution functions to an application of QBAGs to explaining the
aggregated score of movie ratings on the website Rotten Tomatoes~\cite{cocarascu2019extracting}.
The QBAGs in this application are based on a hierarchical structure with the overall evaluation
at the top. The evaluation can be attacked and supported by evaluation
criteria like the quality of acting, directing and writing. The criteria can be
further attacked and supported by subcriteria like the performance of 
individual actors or the writing for particular parts of the movie.
The initial strength of the criteria is determined by a natural
language processing pipeline, which extracts argumentative phrases from
reviews, evaluates their sentiment using sentiment analysis tools,
assigns them to criteria using topic classification tools and 
merges the sentiment of the individual phrases such that the final
strength of the evaluation aligns with the Rotten Tomatoe score~\cite{cocarascu2019extracting}.

Figure~\ref{fig_case_study} shows a small example QBAG for this scenario.
Here, $\mathsf{m}$ is the topic argument representing a movie to be rated. The attacker $\mathsf{f_W}$ of $\mathsf{m}$ stands for the script quality of the movie (writing), while the supporters $\mathsf{f_D}$ and $\mathsf{f_A}$ represent the quality of directors and actors of the movie, respectively. $\mathsf{f_A}$ is further supported by two specific actors $\mathsf{f'_{A1}}$ and $\mathsf{f'_{A2}}$. The QBAG is evaluated by DFQuAD semantics and the movie $\mathsf{m}$ obtains a final strength of $0.85$.

Our four contribution functions can explain
in different ways how the criteria affect the final 
strength of the movie ($\mathsf{m}$'s final strength).
In Table~\ref{case_study_contribution}, we provide the contributions of all non-topic arguments determined by $\ctrbrempty$, $\ctrbriempty$, $\ctrbsempty$ and $\ctrbgempty$.
Note that all contribution functions agree qualitatively: $\mathsf{f_W}$ always has a negative contribution, whereas all other arguments have positive contributions.
We can also see that $\mathsf{f_A}$ always has the largest positive contribution. However, the contribution ranking for the remaining arguments varies across different contribution functions because they measure different things. 
The removal-based
contribution functions $\ctrbrempty$ and $\ctrbriempty$ 
measure the difference between the movie's strength when
arguments are removed in a particular way. The Shapley 
values determined by $\ctrbsempty$ have a similar meaning
but average over all possible subgraphs and are normalized
such that the contribution values sum up to the difference
between the movie's final and initial strengths.
Finally, $\ctrbgempty$ measures the local sensitivity 
of the movie's strength with respect to changes in the
other arguments' initial strength. More precisely, 
a small change in the arguments' initial strength
will change the strength of the movie by a value that is
proportional to the argument's contribution value.

Let us use the example to discuss the intuition of some principles. Since 
$\ctrbsempty$ satisfies (quantitative) contribution existence,
Proposition~\ref{prop:shapley-qcontribution-existence}
guarantees that at least one argument must have 
a non-zero Shapley-based contribution because
the initial strength of $\mathsf{m}$ is different from its final strength. Let us note that, in our example, this is also true for all
other functions even though they do not satisfy this principle in general
(it is possible that all contribution values are $0$ even though there was a change).
Another interesting property to note for the Shapley values
is that their sum equals the difference between the initial and final strengths of $\mathsf{m}$. In this way, it nicely 
distributes the contributions among the arguments.
$\ctrbgempty$ satisfies quantitative local faithfulness, as shown in Propositions~\ref{proposition_quan_local_faith}, which guarantees that the
scores accurately reflect the sensitivity of the final evaluation with respect to 
their initial strength. Since $\mathsf{f_{A1}}$, $\mathsf{f_{A2}}$ and $\mathsf{f_{D}}$ all have very similar scores under $\ctrbgempty$, we learn
that the final evaluation is approximately equally sensitive with respect to their
initial strengths. More specifically, changing their initial
strengths by some small value $\delta$ will change the
final strength of the movie by approximately
$0.15 \cdot \delta$. Similarly, changing $\mathsf{f_{A}}$
and $\mathsf{f_{W}}$ would result in a change of 
approximately $0.17 \cdot \delta$ and $-0.21 \cdot \delta$,
respectively.
Since $\ctrbrempty$ satisfies (quantitative) counterfactuality, Proposition~\ref{prop:removal-counterfactual} allows us
to infer that removing the criterion $\mathsf{f_A}$
would decrease the movie's final strength of $\mathsf{m}$  by approximately $0.05$.
All contribution functions satisfy directionality
according to Proposition~\ref{proposition_directionality}. 
This allows us to derive some contributions without making
any computations. 
For example, if we consider $\mathsf{f_A}$ as the topic argument, the contribution from $\mathsf{f_D}$ to $\mathsf{f_A}$ must be $0$ under all four contribution functions due to the absence of a path from $\mathsf{f_D}$ to $\mathsf{f_A}$.


\begin{figure}[t]
	\centering
	\begin{tikzpicture}[scale=0.8]
    \node[unanode]    (fa2)    at(0,2)  {\argnode{\mathsf{f'_{A2}}}{0.07}{0.07}};
    \node[unanode]    (fa)    at(2,2)  {\argnode{\mathsf{f_A}}{0.16}{0.26}};
    \node[unanode]  (m)    at(4,2)  {\argnode{\mathsf{m}}{0.79}{0.85}};
    \node[unanode]    (fa1)    at(0,0)  {\argnode{\mathsf{f'_{A1}}}{0.05}{0.05}};
    \node[unanode]    (fd)    at(2,0)  {\argnode{\mathsf{f_D}}{0.05}{0.05}};
    \node[unanode]    (fw)    at(4,0)  {\argnode{\mathsf{f_W}}{0.02}{0.02}};
    
    \path [->, line width=0.2mm]  (fa2) edge node[above] {+} (fa);
    \path [->, line width=0.2mm]  (fa) edge node[above] {+} (m);
    \path [->, line width=0.2mm]  (fa1) edge node[left] {+} (fa);
    \path [->, line width=0.2mm]  (fd) edge node[left] {+} (m);
    \path [->, line width=0.2mm]  (fw) edge node[left] {-} (m);
\end{tikzpicture}
	\caption{QBAG for movie rating with initial strengths in normal font and DFQuAD final strength in bold font (taken from \cite{cocarascu2019extracting}).}
	\label{fig_case_study}
\end{figure}

\begin{table}
\caption{The results of four contribution functions for the QBAG in Figure \ref{fig_case_study}.}
\label{case_study_contribution}
\begin{center}
\begin{tabular}{ccccc}
\toprule
\textbf{Argument}  & $\ctrbrempty$        & $\ctrbriempty$       & $\ctrbsempty$  & $\ctrbgempty$ \\
\midrule
$f_A$               & 0.051443 & 0.03192  & 0.044726 & 0.176258 \\
$f'_{A1}$           & 0.007792 & 0.007792 & 0.004065 & 0.15585  \\
$f'_{A2}$           & 0.011144 & 0.011144 & 0.00577  & 0.1592   \\
$f_D$               & 0.007792 & 0.007792 & 0.011248 & 0.15584  \\
$f_W$               & -0.0042  & -0.0042  & -0.00807 & -0.21   \\
\bottomrule
\end{tabular}
\end{center}
\end{table}

\section{Discussion and Related Work}
\label{sec:discussion}
The study presented in this paper can be considered a contribution to argumentative explainability, with a focus on quantitative bipolar argumentation.
The former has emerged as a research trend in recent years~\cite{Cyras.et.al:2021-IJCAI,vassiliades_bassiliades_patkos_2021,DBLP:conf/ecai/0007PT23,AmgoudBV17,YIN_RAE_IJCAI}; in applied research where argumentation is utilized, quantitative (bipolar) argumentation is often employed as the formal argumentation approach of choice~\cite{10.5555/3304889.3304929,RAGO2021103506,rago2016discontinuity,kotonya2019gradual,cocarascu2019extracting,lidecision,chi2021optimized}.
However, only a few works focus on the formal study of explaining inference in quantitative bipolar argumentation.
Notable exceptions are~\cite{KAMPIK2023109066}, which introduces several explanation notions for sets of argument explaining changes in inference in QBAGs after updates have been applied to the argumentation graph, and~\cite{Cyras:Kampik:Weng:2022}, which is the first paper introducing contribution functions to quantitative bipolar argumentation.
Let us note that some related contribution functions 
have been introduced for gradual semantics over non-bipolar abstract argumentation before~\cite{AmgoudBV17,Delobelle:Villata:2019}). 
We refer to \cite{AmgoudDV22} for an overview of semantics
in this area.
\cite{DBLP:conf/ecai/0007PT23} studies gradient-based contribution functions under Df-QuAD semantics in more detail.
\cite{DBLP:conf/sgai/HimeurYBC21} introduces a function that intuitively corresponds to the counterfactual contribution function; however, the study presented in~\cite{DBLP:conf/sgai/HimeurYBC21} does not focus on single-argument contributions but on the aggregation of contributions of arguments uttered by a specific agent.

Our work can be seen as a consolidation and extension of some of the theoretical contributions provided in \cite{Cyras:Kampik:Weng:2022} and \cite{DBLP:conf/ecai/0007PT23}, providing a rigorous analysis of previously introduced and new principles that highlights the differences between contribution functions, thus facilitating the well-informed selection of contribution functions in application scenarios.
From our analysis, we can see that for each of the three contribution functions $\ctrbrempty$, $\ctrbsempty$, and $\ctrbgempty$, there is a contribution function principle that is satisfied by only this function and with respect to all of the surveyed semantics: (quantitative) counterfactuality for $\ctrbrempty$, quantitative contribution existence for $\ctrbsempty$, and (quantitative) local faithfulness for $\ctrbgempty$. In the case of quantitative contribution existence, a relaxation of the principle---\emph{contribution existence}---is satisfied by other contribution functions as well, albeit merely with respect to some, and not all, of the surveyed semantics.
For $\ctrbriempty$, we did not define a corresponding contribution function principle: let us claim that adjusting the counterfactuality principle to fit the behavior of $\ctrbriempty$ is trivially possible but arguably pointless, as the principle would be deliberately designed to be satisfied by $\ctrbriempty$ (and only by $\ctrbriempty$).
What remains to be studied are the conjectures posed in Section~\ref{sec:rest}, as well as the satisfaction of contribution function principles in the case of cyclic QBAGs. 

\section*{Acknowledgements}
We thank the anonymous reviewers for their helpful feedback.
This work was 
 partially funded by the European Research Council (ERC) under the European Union’s Horizon 2020 research and innovation programme (grant agreement No. 101020934) and by 
J.P. Morgan and by the
Royal Academy of Engineering under the Research Chairs
and Senior Research Fellowships scheme. 

\appendix

\section{Appendix: Counter-Examples Proving the Violation of Proximity}
The appendix contains counterexamples showing that many contribution functions violate the proximity principle with respect to many of the surveyed contribution functions. The only exception is $\ctrbgempty$ and EBT semantics, for which we do not provide a proof of violation (or satisfaction).
\begin{proposition}
$\ctrbrempty$ violates proximity w.r.t. QE, DFQuAD, SD-DFQuAD, EB, and EBT semantics.
\end{proposition}
\begin{proof}
For the counterexample for QE semantics, consider Figure~\ref{fig:proximity-removal-qe}, topic argument $\arga$ and contributors $\argb$ and $\argc$. Although $\argb$ is strictly closer to $\arga$ than $\argc$ is, we have $|\ctrbr{\argb}{\arga}| \approx 0.0012 < |\ctrbr{\argc}{\arga}| \approx 0.0037$.
For DFQuAD semantics, consider Figure~\ref{fig:proximity-removal-df}, topic argument $\arga$ and contributors $\argb$ and $\argc$. Although $\argb$ is strictly closer to $\arga$ than $\argc$ is, we have $|\ctrbr{\argb}{\arga}| = 0 < |\ctrbr{\argc}{\arga}| \approx 0.05$.
For SD-DFQuAD semantics, consider Figure~\ref{fig:proximity-removal-df}, topic argument $\arga$ and contributors $\argb$ and $\argc$. Although $\argb$ is strictly closer to $\arga$ than $\argc$ is, we have $|\ctrbr{\argb}{\arga}| \approx 0.0238 < |\ctrbr{\argc}{\arga}| \approx 0.1483$.
For EB and EBT semantics, consider Figure~\ref{fig:proximity-removal-eb}, topic argument $\arga$ and contributors $\argb$ and $\argc$. Although $\argb$ is strictly closer to $\arga$ than $\argc$ is, we have $|\ctrbr{\argb}{\arga}| \approx 0.0075 < |\ctrbr{\argc}{\arga}| \approx 0.0089$.
\end{proof}

\begin{figure}[ht]
\centering
\begin{tikzpicture}[scale=0.8]
     \node[unanode]    (a)    at(4,0)  {\argnode{\arga}{0.5}{0.5}};
     \node[unanode]  (b)    at(2,0)  {\argnode{\argb}{0.1}{0.0}};
     \node[unanode]    (c)    at(0,0)  {\argnode{\argc}{1}{1}};
     \path [->, line width=0.2mm]  (b) edge node[left, above] {-} (a);
     \path [->, line width=0.2mm]  (c) edge node[left, above] {-} (b);
\end{tikzpicture}
\caption{$\ctrbrempty$ violates the proximity principle w.r.t.\ DFQuAD semantics.}
\label{fig:proximity-removal-df}
\end{figure}
\begin{figure}[ht]
\centering
\begin{tikzpicture}[scale=0.8]
     \node[unanode]    (a)    at(2,0)  {\argnode{\arga}{0.5}{0.5238}};
     \node[unanode]  (b)    at(2,2)  {\argnode{\argb}{0.05}{0.05}};
     \node[unanode]    (c)    at(0,4)  {\argnode{\argc}{1}{1}};
     \node[unanode]    (d)    at(4,4)  {\argnode{\argd}{1}{1}};
     \path [->, line width=0.2mm]  (b) edge node[left] {+} (a);
     \path [->, line width=0.2mm]  (c) edge node[left, above] {-} (b);
     \path [->, line width=0.2mm]  (d) edge node[left, above] {+} (b);
\end{tikzpicture}
\caption{$\ctrbrempty$ violates the proximity principle w.r.t.\ SD-DFQuAD semantics.}
\label{fig:proximity-removal-sd}
\end{figure}
\begin{figure}[ht]
\centering
\begin{tikzpicture}[scale=0.8]
     \node[unanode]    (a)    at(4,0)  {\argnode{\arga}{0.5}{0.4925}};
     \node[unanode]  (b)    at(2,0)  {\argnode{\argb}{0.1}{0.0451}};
     \node[unanode]    (c)    at(0,0)  {\argnode{\argc}{1}{1}};
     \path [->, line width=0.2mm]  (b) edge node[left, above] {-} (a);
     \path [->, line width=0.2mm]  (c) edge node[left, above] {-} (b);
\end{tikzpicture}
\caption{$\ctrbrempty$ violates the proximity principle w.r.t.\ EB and EBT semantics.}
\label{fig:proximity-removal-eb}
\end{figure}

\begin{proposition}
$\ctrbriempty$ violates proximity w.r.t. QE, DFQuAD, SD-DFQuAD, EB, and EBT semantics.
\end{proposition}
\begin{proof}
    For the counterexample for QE semantics, consider Figure~\ref{fig:proximity-iremoval-qe}, topic argument $\arga$ and contributors $\argb$ and $\argc$. Although $\argb$ is strictly closer to $\arga$ than $\argc$ is, we have $|\ctrbri{\argb}{\arga}| \approx 0.005 < |\ctrbri{\argc}{\arga}| \approx 0.0959$.
    For DFQuAD semantics, consider Figure~\ref{fig:proximity-iremoval-df}, topic argument $\arga$ and contributors $\argb$ and $\argc$. Although $\argb$ is strictly closer to $\arga$ than $\argc$ is, we have $|\ctrbri{\argb}{\arga}| \approx 0.05 < |\ctrbri{\argc}{\arga}| = 0.405$.
    For SD-DFQuAD semantics, consider Figure~\ref{fig:proximity-iremoval-sd}, topic argument $\arga$ and contributors $\argb$ and $\argc$. Although $\argb$ is strictly closer to $\arga$ than $\argc$ is, we have $|\ctrbri{\argb}{\arga}| \approx 0.0454 < |\ctrbri{\argc}{\arga}| \approx 0.127$.
    For EB and EBT semantics, consider Figure~\ref{fig:proximity-iremoval-eb}, topic argument $\arga$ and contributors $\argb$ and $\argc$. Although $\argb$ is strictly closer to $\arga$ than $\argc$ is, we have $|\ctrbri{\argb}{\arga}| \approx 0.0169 < |\ctrbri{\argc}{\arga}| \approx 0.0184$.
\end{proof}

\begin{figure}[ht]
\centering
\begin{tikzpicture}[scale=0.8]
     \node[unanode]    (a)    at(4,0)  {\argnode{\arga}{0.5}{0.6009}};
     \node[unanode]  (b)    at(2,0)  {\argnode{\argb}{0.1}{0.5028}};
     \node[unanode]    (c)    at(0,0)  {\argnode{\argc}{0.9}{0.9}};
     \path [->, line width=0.2mm]  (b) edge node[left, above] {+} (a);
     \path [->, line width=0.2mm]  (c) edge node[left, above] {+} (b);
\end{tikzpicture}
\caption{$\ctrbriempty$ violates the proximity principle w.r.t.\ QE semantics.}
\label{fig:proximity-iremoval-qe}
\end{figure}
\begin{figure}[ht]
\centering
\begin{tikzpicture}[scale=0.8]
     \node[unanode]    (a)    at(4,0)  {\argnode{\arga}{0.5}{0.955}};
     \node[unanode]  (b)    at(2,0)  {\argnode{\argb}{0.1}{0.91}};
     \node[unanode]    (c)    at(0,0)  {\argnode{\argc}{0.9}{0.9}};
     \path [->, line width=0.2mm]  (b) edge node[left, above] {+} (a);
     \path [->, line width=0.2mm]  (c) edge node[left, above] {+} (b);
\end{tikzpicture}
\caption{$\ctrbriempty$ and $\ctrbgempty$ violate the proximity principle w.r.t.\ DFQuAD semantics.}
\label{fig:proximity-iremoval-df}
\end{figure}
\begin{figure}[ht]
\centering
\begin{tikzpicture}[scale=0.8]
     \node[unanode]    (a)    at(4,0)  {\argnode{\arga}{0.5}{0.6724}};
     \node[unanode]  (b)    at(2,0)  {\argnode{\argb}{0.1}{0.5263}};
     \node[unanode]    (c)    at(0,0)  {\argnode{\argc}{0.9}{0.9}};
     \path [->, line width=0.2mm]  (b) edge node[left, above] {+} (a);
     \path [->, line width=0.2mm]  (c) edge node[left, above] {+} (b);
\end{tikzpicture}
\caption{$\ctrbriempty$ violates the proximity principle w.r.t.\ SD-DFQuAD semantics.}
\label{fig:proximity-iremoval-sd}
\end{figure}
\begin{figure}[ht]
\centering
\begin{tikzpicture}[scale=0.8]
     \node[unanode]    (a)    at(4,0)  {\argnode{\arga}{0.5}{0.5353}};
     \node[unanode]  (b)    at(2,0)  {\argnode{\argb}{0.1}{0.2054}};
     \node[unanode]    (c)    at(0,0)  {\argnode{\argc}{0.9}{0.9}};
     \path [->, line width=0.2mm]  (b) edge node[left, above] {+} (a);
     \path [->, line width=0.2mm]  (c) edge node[left, above] {+} (b);
\end{tikzpicture}
\caption{$\ctrbriempty$ violates the proximity principle w.r.t.\ EB and EBT semantics.}
\label{fig:proximity-iremoval-eb}
\end{figure}

\begin{proposition}
$\ctrbsempty$ violates proximity w.r.t. QE, DFQuAD, SD-DFQuAD, EB, and EBT semantics.
\end{proposition}
\begin{proof}
    For the counterexample for QE semantics, consider Figure~\ref{fig:proximity-shapley-qe}, topic argument $\arga$ and contributors $\arge$ and $\argf$. Although $\arge$ is strictly closer to $\arga$ than $\argf$ is, we have $|\ctrbs{\arge}{\arga}| \approx 0.00005 < |\ctrbs{\argf}{\arga}| \approx 0.00056$.
    For DFQuAD semantics, consider Figure~\ref{fig:proximity-shapley-df}, topic argument $\arga$ and contributors $\arge$ and $\argf$. Although $\arge$ is strictly closer to $\arga$ than $\argf$ is, we have $|\ctrbs{\arge}{\arga}| \approx 0.0037 < |\ctrbs{\argf}{\arga}| \approx 0.0057$.
    For SD-DFQuAD semantics, consider Figure~\ref{fig:proximity-shapley-sdf}, topic argument $\arga$ and contributors $\arge$ and $\argf$. Although $\arge$ is strictly closer to $\arga$ than $\argf$ is, we have $|\ctrbs{\arge}{\arga}| \approx 0.00065 < |\ctrbs{\argf}{\arga}| \approx 0.00075$.
    For EB semantics, consider Figure~\ref{fig:proximity-shapley-eb}, topic argument $\arga$ and contributors $\arge$ and $\argf$. Although $\arge$ is strictly closer to $\arga$ than $\argf$ is, we have $|\ctrbs{\arge}{\arga}| \approx 0.00022 < |\ctrbs{\argf}{\arga}| \approx 0.00026$.
     For EB semantics, consider Figure~\ref{fig:proximity-shapley-ebt}, topic argument $\arga$ and contributors $\arge$ and $\argf$. Although $\arge$ is strictly closer to $\arga$ than $\argf$ is, we have $|\ctrbs{\arge}{\arga}| \approx 0.000098 < |\ctrbs{\argf}{\arga}| \approx 0.000108$.
\end{proof}

\begin{figure}[ht]
\centering
\begin{tikzpicture}[scale=0.8]
     \node[unanode]    (a)    at(6,2)  {\argnode{\arga}{0.1}{0.0829}};
     \node[unanode]  (b)    at(4,4)  {\argnode{\argb}{0.15}{4547}};
     \node[unanode]    (c)    at(4,2)  {\argnode{\argc}{0.15}{4547}};
     \node[unanode]    (d)    at(4,0)  {\argnode{\argd}{0.15}{0.4547}};
     \node[unanode]    (e)    at(2,2)  {\argnode{\arge}{0.495}{0.7475}};
     \node[unanode]    (f)    at(0,2)  {\argnode{\argf}{1}{1}};
     \path [->, line width=0.2mm]  (b) edge node[left] {-} (a);
     \path [->, line width=0.2mm]  (c) edge node[left, above] {-} (a);
     \path [->, line width=0.2mm]  (d) edge node[left] {+} (a);
     \path [->, line width=0.2mm]  (e) edge node[left, above] {+} (b);
     \path [->, line width=0.2mm]  (e) edge node[left, above] {+} (c);
     \path [->, line width=0.2mm]  (e) edge node[left, above] {+} (d);
     \path [->, line width=0.2mm]  (f) edge node[left, above] {+} (e);
\end{tikzpicture}
\caption{$\ctrbsempty$ violates the proximity principle w.r.t.\ QE semantics.}
\label{fig:proximity-shapley-qe}
\end{figure}

\begin{figure}[ht]
\centering
\begin{tikzpicture}[scale=0.8]
     \node[unanode]    (a)    at(6,2)  {\argnode{\arga}{0.125}{0.125}};
     \node[unanode]  (b)    at(4,4)  {\argnode{\argb}{0}{1}};
     \node[unanode]    (c)    at(4,2)  {\argnode{\argc}{0.2}{1}};
     \node[unanode]    (d)    at(4,0)  {\argnode{\argd}{0.2}{0.1}};
     \node[unanode]    (e)    at(2,2)  {\argnode{\arge}{0.2}{1}};
     \node[unanode]    (f)    at(0,2)  {\argnode{\argf}{1}{1}};
     \path [->, line width=0.2mm]  (b) edge node[left] {-} (a);
     \path [->, line width=0.2mm]  (c) edge node[left, above] {-} (a);
     \path [->, line width=0.2mm]  (d) edge node[left] {+} (a);
     \path [->, line width=0.2mm]  (e) edge node[left, above] {+} (b);
     \path [->, line width=0.2mm]  (e) edge node[left, above] {+} (c);
     \path [->, line width=0.2mm]  (e) edge node[left, above] {+} (d);
     \path [->, line width=0.2mm]  (f) edge node[left, above] {+} (e);
\end{tikzpicture}
\caption{$\ctrbsempty$ violates the proximity principle w.r.t.\ DFQuAD semantics.}
\label{fig:proximity-shapley-df}
\end{figure}

\begin{figure}[ht]
\centering
\begin{tikzpicture}[scale=0.8]
     \node[unanode]    (a)    at(6,2)  {\argnode{\arga}{0.3}{0.1732}};
     \node[unanode]  (b)    at(4,4)  {\argnode{\argb}{0.5}{0.75}};
     \node[unanode]    (c)    at(4,2)  {\argnode{\argc}{0.01}{0.505}};
     \node[unanode]    (d)    at(4,0)  {\argnode{\argd}{0.4}{0.2676}};
     \node[unanode]    (e)    at(2,2)  {\argnode{\arge}{0.01}{0.505}};
     \node[unanode]    (f)    at(0,2)  {\argnode{\argf}{1}{1}};
     \node[unanode]    (g)    at(6,4)  {\argnode{\argg}{1}{1}};
     \node[unanode]    (h)    at(2,4)  {\argnode{\argh}{1}{1}};
     \node[unanode]    (i)    at(2,0)  {\argnode{\argi}{1}{1}};
     \path [->, line width=0.2mm]  (b) edge node[left] {-} (a);
     \path [->, line width=0.2mm]  (c) edge node[left, above] {-} (a);
     \path [->, line width=0.2mm]  (d) edge node[left] {+} (a);
     \path [->, line width=0.2mm]  (e) edge node[left, below] {+} (b);
     \path [->, line width=0.2mm]  (e) edge node[left, above] {+} (c);
     \path [->, line width=0.2mm]  (e) edge node[left, above] {+} (d);
     \path [->, line width=0.2mm]  (f) edge node[left, above] {+} (e);
     \path [->, line width=0.2mm]  (g) edge node[left] {-} (a);
     \path [->, line width=0.2mm]  (i) edge node[left, below] {-} (d);
     \path [->, line width=0.2mm]  (h) edge node[left, above] {+} (b);
     \path [->, line width=0.2mm]  (h) edge node[left, above] {+} (c);
     \path [->, line width=0.2mm]  (h) edge node[left, below] {+} (d);
\end{tikzpicture}
\caption{$\ctrbsempty$ violates the proximity principle w.r.t.\ SD-DFQuAD semantics.}
\label{fig:proximity-shapley-sdf}
\end{figure}

\begin{figure}[ht]
\centering
\begin{tikzpicture}[scale=0.8]
     \node[unanode]    (a)    at(6,2)  {\argnode{\arga}{0.5}{0.5052}};
     \node[unanode]  (b)    at(4,4)  {\argnode{\argb}{0.1}{0.1433}};
     \node[unanode]    (c)    at(4,2)  {\argnode{\argc}{0.1}{0.1433}};
     \node[unanode]    (d)    at(4,0)  {\argnode{\argd}{0.51}{0.5874}};
     \node[unanode]    (e)    at(2,2)  {\argnode{\arge}{0.25}{0.4418}};
     \node[unanode]    (f)    at(0,2)  {\argnode{\argf}{1}{1}};
     \node[unanode]    (g)    at(6,4)  {\argnode{\argg}{0.27}{0.27}};
     \path [->, line width=0.2mm]  (b) edge node[left] {-} (a);
     \path [->, line width=0.2mm]  (c) edge node[left, above] {-} (a);
     \path [->, line width=0.2mm]  (d) edge node[left] {+} (a);
     \path [->, line width=0.2mm]  (e) edge node[left, below] {+} (b);
     \path [->, line width=0.2mm]  (e) edge node[left, above] {+} (c);
     \path [->, line width=0.2mm]  (e) edge node[left, above] {+} (d);
     \path [->, line width=0.2mm]  (f) edge node[left, above] {+} (e);
     \path [->, line width=0.2mm]  (g) edge node[left] {-} (a);
\end{tikzpicture}
\caption{$\ctrbsempty$ violates the proximity principle w.r.t.\ EB semantics.}
\label{fig:proximity-shapley-eb}
\end{figure}

\begin{figure}[ht]
\centering
\begin{tikzpicture}[scale=0.8]
     \node[unanode]    (a)    at(6,2)  {\argnode{\arga}{0.3}{0.3036}};
     \node[unanode]  (b)    at(4,4)  {\argnode{\argb}{0.4}{0.425}};
     \node[unanode]    (c)    at(0,2)  {\argnode{\argc}{0.55}{0.55}};
     \node[unanode]    (d)    at(2,2)  {\argnode{\argd}{0.51}{0.4474}};
     \node[unanode]    (e)    at(2,4)  {\argnode{\arge}{0.25}{0.1415}};
     \node[unanode]    (f)    at(0,4)  {\argnode{\argf}{1}{1}};
     \node[unanode]    (g)    at(0,0)  {\argnode{\argg}{0.25}{0.25}};
     \path [->, line width=0.2mm]  (b) edge node[left] {-} (a);
     \path [->, line width=0.2mm]  (c) edge node[left, above] {-} (d);
     \path [->, line width=0.2mm]  (d) edge node[left, above] {+} (a);
     \path [->, line width=0.2mm]  (e) edge node[left, above] {+} (b);
     \path [->, line width=0.2mm]  (e) edge node[left] {+} (d);
     \path [->, line width=0.2mm]  (f) edge node[left, above] {-} (e);
     \path [->, line width=0.2mm]  (g) edge node[left] {-} (a);
     \path [->, line width=0.2mm]  (g) edge node[left] {-} (d);
\end{tikzpicture}
\caption{$\ctrbsempty$ violates the proximity principle w.r.t.\ EBT semantics.}
\label{fig:proximity-shapley-ebt}
\end{figure}
\begin{proposition}
$\ctrbgempty$ violates proximity w.r.t. QE, DFQuAD, SD-DFQuAD, and EB semantics.
\end{proposition}
\begin{proof}
For the counterexample for QE semantics, consider Figure~\ref{fig:proximity-gradient-qe}, topic argument $\arga$ and contributors $\argb$ and $\argd$. Although $\argb$ is strictly closer to $\arga$ than $\argd$ is, we have $|\ctrbg{\argb}{\arga}| \approx 0.1081 < |\ctrbg{\argd}{\arga}| \approx 0.2945$.
For DFQuAD semantics, consider Figure~\ref{fig:proximity-iremoval-df}, topic argument $\arga$ and contributors $\argb$ and $\argc$. Although $\argb$ is strictly closer to $\arga$ than $\argc$ is, we have $|\ctrbg{\argb}{\arga}| \approx 0.05 < |\ctrbg{\argc}{\arga}| = 0.45$.
For SD-DFQuAD semantics, consider Figure~\ref{fig:proximity-gradient-sd}, topic argument $\arga$ and contributors $\argb$ and $\argd$. Although $\argb$ is strictly closer to $\arga$ than $\argd$ is, we have $|\ctrbg{\argb}{\arga}| \approx 0.121 < |\ctrbg{\argd}{\arga}| \approx 0.234$.
For EB semantics, consider Figure~\ref{fig:proximity-gradient-eb}, topic argument $\arga$ and contributors $\argb$ and $\argd$. Although $\argb$ is strictly closer to $\arga$ than $\argd$ is, we have $|\ctrbg{\argb}{\arga}| \approx 0.0083 < |\ctrbg{\argd}{\arga}| \approx 0.0101$.
\end{proof}
\begin{figure}[ht]
\centering
\begin{tikzpicture}[scale=0.8]
     \node[unanode]    (a)    at(4,0)  {\argnode{\arga}{0.5}{0.6524}};
     \node[unanode]  (b)    at(4,2)  {\argnode{\argb}{0}{0.6622}};
     \node[unanode]    (c0)    at(0,4)  {\argnode{\argca} {0.1}{0.28}};
     \node[unanode]    (c1)    at(2,4)  {\argnode{\argcb} {0.1}{0.28}};
     \node[unanode]    (c2)    at(4,4)  {\argnode{\argcc} {0.1}{0.28}};
     \node[unanode]    (c3)    at(6,4)  {\argnode{\argcd} {0.1}{0.28}};
     \node[unanode]    (c4)    at(8,4)  {\argnode{\argce} {0.1}{0.28}};
     \node[unanode]    (d)    at(4,6)  {\argnode{\argd}{0.5}{0.5}};
     \path [->, line width=0.2mm]  (b) edge node[left] {+} (a);
     \path [->, line width=0.2mm]  (c1) edge node[left, below] {+} (b);
     \path [->, line width=0.2mm]  (d) edge node[left, above] {+} (c1);
     \path [->, line width=0.2mm]  (c2) edge node[left] {+} (b);
     \path [->, line width=0.2mm]  (d) edge node[left] {+} (c2);
     \path [->, line width=0.2mm]  (c0) edge node[left, below] {+} (b);
     \path [->, line width=0.2mm]  (d) edge node[left, above] {+} (c0);
     \path [->, line width=0.2mm]  (c3) edge node[left] {+} (b);
     \path [->, line width=0.2mm]  (d) edge node[left, above] {+} (c3);
     \path [->, line width=0.2mm]  (d) edge node[left, above] {+} (c4);
     \path [->, line width=0.2mm]  (c4) edge node[left] {+} (b);
\end{tikzpicture}
\caption{$\ctrbgempty$ violates the proximity principle w.r.t.\ QE semantics.}
\label{fig:proximity-gradient-qe}
\end{figure}

\begin{figure}[ht]
\centering
\begin{tikzpicture}[scale=0.8]
     \node[unanode]    (a)    at(4,0)  {\argnode{\arga}{0.5}{0.7051}};
     \node[unanode]  (b)    at(4,2)  {\argnode{\argb}{1}{0.6952}};
     \node[unanode]    (c0)    at(0,4)  {\argnode{\argca} {0.1}{0.1089}};
     \node[unanode]    (c1)    at(2,4)  {\argnode{\argcb} {0.1}{0.1089}};
     \node[unanode]    (c2)    at(4,4)  {\argnode{\argcc} {0.1}{0.1089}};
     \node[unanode]    (c3)    at(6,4)  {\argnode{\argcd} {0.1}{0.1089}};
     \node[unanode]    (c4)    at(8,4)  {\argnode{\argce} {0.1}{0.1089}};
     \node[unanode]    (d)    at(4,6)  {\argnode{\argd}{0.01}{0.01}};
     \path [->, line width=0.2mm]  (b) edge node[left] {+} (a);
     \path [->, line width=0.2mm]  (c1) edge node[left, below] {-} (b);
     \path [->, line width=0.2mm]  (d) edge node[left, above] {+} (c1);
     \path [->, line width=0.2mm]  (c2) edge node[left] {-} (b);
     \path [->, line width=0.2mm]  (d) edge node[left] {+} (c2);
     \path [->, line width=0.2mm]  (c0) edge node[left, below] {-} (b);
     \path [->, line width=0.2mm]  (d) edge node[left, above] {+} (c0);
     \path [->, line width=0.2mm]  (c3) edge node[left] {-} (b);
     \path [->, line width=0.2mm]  (d) edge node[left, above] {+} (c3);
     \path [->, line width=0.2mm]  (d) edge node[left, above] {+} (c4);
     \path [->, line width=0.2mm]  (c4) edge node[left] {-} (b);
\end{tikzpicture}
\caption{$\ctrbgempty$ violates the proximity principle w.r.t.\ SD-DFQuAD semantics.}
\label{fig:proximity-gradient-sd}
\end{figure}

\begin{figure}[ht]
\centering
\begin{tikzpicture}[scale=0.8]
     \node[unanode]    (a)    at(4,0)  {\argnode{\arga}{0.25}{0.4394}};
     \node[unanode]  (b)    at(4,2)  {\argnode{\argb}{0.4}{0.9892}};
     \node[unanode]    (c0)    at(0,4)  {\argnode{\argca} {0.1}{0.1501}};
     \node[unanode]    (c1)    at(2,4)  {\argnode{\argcb} {0.1}{0.1501}};
    \node[invnode]    (cx)    at(4,4)  {\argdots};
     \node[unanode]    (c3)    at(6,4)  {\argnode{\argccc} {0.1}{0.1501}};
     \node[unanode]    (c4)    at(8,4)  {\argnode{\argccd} {0.1}{0.1501}};
     \node[unanode]    (d)    at(4,6)  {\argnode{\argd}{0.5}{0.5}};
     \path [->, line width=0.2mm]  (b) edge node[left] {+} (a);
     \path [->, line width=0.2mm]  (c1) edge node[left, below] {+} (b);
     \path [->, line width=0.2mm]  (d) edge node[left, above] {+} (c1);
     \path [->, line width=0.2mm]  (cx) edge node[left] {+} (b);
     \path [->, line width=0.2mm]  (d) edge node[left] {+} (cx);
     \path [->, line width=0.2mm]  (c0) edge node[left, below] {+} (b);
     \path [->, line width=0.2mm]  (d) edge node[left, above] {+} (c0);
     \path [->, line width=0.2mm]  (c3) edge node[left] {+} (b);
     \path [->, line width=0.2mm]  (d) edge node[left, above] {+} (c3);
     \path [->, line width=0.2mm]  (d) edge node[left, above] {+} (c4);
     \path [->, line width=0.2mm]  (c4) edge node[left] {+} (b);
\end{tikzpicture}
\caption{$\ctrbgempty$ violates the proximity principle w.r.t.\ EB semantics.}
\label{fig:proximity-gradient-eb}
\end{figure}

\clearpage
\small
\bibliographystyle{elsarticle-num} 
\bibliography{references}

\end{document}